\documentclass[11pt,letterpaper]{article}
\usepackage[utf8]{inputenc}
\usepackage[T1]{fontenc}
\usepackage[margin=1in]{geometry}
\usepackage{lmodern}
\usepackage{microtype}
\usepackage[dvipsnames]{xcolor}
\usepackage{amsmath,mathtools,amssymb,mathrsfs}
\usepackage{nicefrac}
\usepackage{booktabs,multirow,array}
\usepackage{graphicx}
\graphicspath{{figures/}}
\usepackage{subcaption}
\usepackage[ruled,vlined]{algorithm2e}
\renewcommand{\KwIn}[1]{\textbf{Input:}\nobreakspace #1\;}
\renewcommand{\KwOut}[1]{\textbf{Output:}\nobreakspace #1\;}
\usepackage{enumitem}
\usepackage{listings}
\usepackage{tcolorbox}
\usepackage[titletoc]{appendix}
\usepackage[round,authoryear]{natbib}
\usepackage[hidelinks]{hyperref}
\usepackage[capitalize]{cleveref}

\newtheorem{theorem}{Theorem}[section]

\usepackage{aliascnt}

\newcommand{\newtheoremalias}[3]{%
  \newaliascnt{#1}{#2}%
  \newtheorem{#1}[#1]{#3}%
  \aliascntresetthe{#1}%
}

\newtheoremalias{corollary}{theorem}{Corollary}
\newtheoremalias{lemma}{theorem}{Lemma}
\newtheoremalias{proposition}{theorem}{Proposition}
\newtheoremalias{remark}{theorem}{Remark}
\newtheoremalias{assumption}{theorem}{Assumption}

\newenvironment{proof}[1][Proof]
  {\par\noindent{\em #1.}\space}
  {\hfill $\Box$ \vskip 0.4cm}

\crefname{assumption}{Assumption}{Assumptions}
\crefname{algorithm}{Algorithm}{Algorithms}
\crefname{lemma}{Lemma}{Lemmas}
\crefname{proposition}{Proposition}{Propositions}
\crefname{remark}{Remark}{Remarks}
\crefname{theorem}{Theorem}{Theorems}

\newcommand{\E}{\mathbb{E}}
\newcommand{\KL}{\mathrm{KL}}
\DeclareMathOperator*{\argmax}{arg\,max}
\DeclareMathOperator*{\argmin}{arg\,min}
\newcommand{\1}{\mathbf{1}}
\newcommand{\y}{\mathbf y}
\newcommand{\x}{\mathbf x}

\newcommand{\EE}{{\mathbb{E}}}
\newcommand{\br}{{\mathbb{R}}}

\newcommand{\VCal}{{\mathcal{V}}}

\newcommand{\LCal}{{\mathcal{L}}}

\newcommand{\tnref}{\textnormal{ref}}

\newcommand{\SEC}{\mathrm{SEC}_{\mathrm{NLHF}}}
\newcommand{\A}{\mathcal Y} 
\newcommand{\Reg}{\mathrm{Reg}}
\newcommand{\pref}{\pi_{\rm ref}}
\newcommand{\pistar}{\pi^\star}
\newcommand{\pitstar}{\pi_t^\star}
\newcommand{\Vcal}{\mathcal V}

\newcommand{\X}{\mathcal X}

\newcommand{\gap}{\mathrm{gap}}

\newcommand{\TV}{\mathrm{TV}}

\DeclarePairedDelimiterX{\inp}[2]{\langle}{\rangle}{#1 \, , \, #2 }

\newcommand{\rev}[1]{#1}

\title{Efficient Exploration for Iterative Nash Preference Optimization}

\author{%
  Tianlong Nan$^{\diamond}$ \quad
  Xiaopeng Li$^{\dagger}$ \quad
  Christian Kroer$^{\diamond}$ \quad
  Tianyi Lin$^{\diamond}$\\[0.75em]
  Columbia University$^{\diamond}$\\
  The Chinese University of Hong Kong, Shenzhen$^{\dagger}$%
}
\date{}

\begin{document}

\maketitle

\begin{abstract}
    Preference alignment is central to improving large language models (LLMs), but reward-based formulations can be restrictive when human preferences are non-transitive.
    Nash learning from human feedback (NLHF) addresses this limitation by modeling alignment as a preference game and seeking a Nash equilibrium.
    However, the learning-theoretic foundations of scalable NLHF remain limited: existing regret guarantees rely on explicit preference-model estimation and minimax oracles, whereas simpler iterative methods lack such guarantees.
    We study online iterative NLHF and identify exploration as a key obstacle.
    First, we show that standard iterative NLHF can incur an exponential dependence on the inverse KL-regularization parameter, demonstrating that implicit exploration through policy updates can be insufficient.
    We then propose Exploratory Nash Preference Optimization (ENPO), which combines a SFT-type regularization with adversarial policy exploration.
    ENPO eliminates this exponential dependence without requiring minimax oracles or explicit preference-model estimation.
    We further introduce \emph{Bonus-Explorer ENPO} (BENPO), which uses additional oracles to achieve an $O(\log T)$ regret bound.
    Finally, we develop \emph{Direct ENPO} (DENPO), a practical variant of ENPO for fine-tuning LLMs.
    Experiments with \texttt{Llama-3-8B-Instruct} demonstrate consistent improvements over the evaluated RLHF and NLHF baselines across multiple benchmarks.
\end{abstract}

\section{Introduction}\label{sec:intro}
Large language models (LLMs) are increasingly trained to act according to human preferences, 
typically through \textit{reinforcement learning from human feedback} (RLHF)~\citep{Christiano-2017-Deep, Stiennon-2020-Learning}. 
Standard RLHF first fits a scalar reward model from pairwise comparisons and then optimizes a policy against that reward. 
This reduction has been empirically successful~\citep{Ziegler-2019-Fine, Ouyang-2022-Training, Touvron-2023-Llama, Achiam-2023-GPT}, 
but it can be restrictive when preferences are heterogeneous, cyclic, 
or otherwise not well described by a single Bradley-Terry (BT) reward model~\citep{Bradley-1952-Rank}. 
Game-theoretic alignment methods, 
including \emph{Nash learning from human feedback} (NLHF)~\citep{Munos-2024-Nash} and its variants~\citep{Chakraborty-2024-Maxmin, Swamy-2024-Minimaximalist, Ye-2024-Online, Rosset-2024-Direct, Wu-2025-Self, Zhang-2025-Iterative, Wu-2025-Greedy, Wu-2026-Multiplayer, Lee-2026-Regularized}, 
instead treat pairwise preference as the payoff of a game and seek a Nash equilibrium (NE) policy.

Despite growing interest in game-theoretic alignment~\citep{Bai-2022-Training, Cui-2024-Ultrafeedback}, 
its online learning foundations remain limited. 
Most existing analyses assume the preference game is fixed and known, 
and study how quickly equilibrium dynamics compute a Nash equilibrium~\citep{Munos-2024-Nash, Wu-2025-Self, Zhang-2025-Improving, Zhang-2025-Iterative, Wang-2025-Magnetic, Zhou-2025-Extragradient, Tiapkin-2025-Accelerating}. 
This optimization view is useful, 
but it abstracts away a central difficulty in LLM alignment: 
the payoff is revealed only through costly, learner-chosen preference queries. 
Each update determines which responses are sampled, which opponents are compared, 
and therefore which parts of the game become statistically visible. 
Online NLHF is thus a coupled optimization and coverage problem. 
Existing learning results are comparatively scarce~\citep{Ye-2024-Online, Wu-2025-Greedy, Lee-2026-Regularized}. 
Among them,~\citet{Ye-2024-Online} obtain an $O(\sqrt{T})$ regret bound without exponential dependence on the KL-regularization parameter, 
but their method requires minimax oracles and \emph{explicit estimation} of a general preference model. 
In contrast, iterative NLHF methods~\citep{Zhang-2025-Iterative} are simpler and more scalable 
because they update the policy directly through preference losses, 
but they currently lack regret guarantees. 
This leaves a central question: 
\emph{can iterative NLHF retain its algorithmic simplicity while achieving provable regret guarantees?}

A key obstacle is exploration. 
Recent work on iterative direct preference optimization (DPO) under the BT model shows that purely on-policy sampling can suffer an \emph{exponential dependence} on the KL-regularization modulus 
when high-quality responses are rarely sampled by the reference policy~\citep{Xie-2025-Exploratory}. 
This suggests that explicit exploration is necessary in online alignment. 
However, extending this idea from DPO to NLHF is highly nontrivial. 
The analysis of exploratory DPO relies on a scalar reward representation and the interpretation of DPO as implicit $Q^\star$-approximation~\citep{Rafailov-2024-From}. 
In NLHF, preferences may be non-transitive and need not admit any scalar reward representation. 
Moreover, the goal is not to maximize a reward, but to minimize a duality gap against \emph{adaptive} opponent policies, which makes equilibrium learning much harder. 
Thus, exploration in NLHF must cover not only informative responses, 
but also the opponent distributions that reveal poorly learned regions of the preference game.

To achieve this goal, 
we augment DPO in iterative NLHF with two exploration mechanisms. 
First, we introduce a \emph{pessimistic} supervised fine-tuning-like (SFT) regularization term into the iterative preference loss. Unlike the optimism term used for exploratory DPO under the BT model~\citep{Xie-2025-Exploratory}, our term acts as a stabilizing regularizer: it strengthens the mirror-descent structure of the update and prevents exploration from destabilizing policy improvement. Second, we construct an adversarial policy maximization procedure that identifies opponent policies exposing poorly covered regions of the preference game, and we sample completions from these policies. This targeted exploration directly addresses the opponent distributions appearing in the duality gap, enabling regret control without requiring full preference-model training at every iteration.

\textbf{Contributions.} 
We propose an explicitly exploratory iterative NLHF method with provable regret guarantees. 
Our contributions can be summarized as follows:
\begin{enumerate}
    \item First, we show that standard iterative NLHF can fail to cover adversarially informative responses and may incur an \emph{exponential dependence} on the KL-regularization parameter. 
    In contrast, we show that by combining SFT-based regularization with adversarial policy exploration, our method \emph{Exploratory Nash Preference Optimization} (ENPO) can 
    eliminate such 
    an exponential dependence, while avoiding minimax oracles and preference model training. 
    Additionally, we present another theoretical algorithm, \emph{Bonus-Explorer ENPO} (BENPO), with an explicit exploration mechanism that requires additional oracles but guarantees an $O(\log T)$ regret bound.
    \item 
    Second, we propose a practical algorithm that uses a surrogate to implement the adversarial policy exploration in ENPO, called \emph{Direct ENPO} (DENPO).
    We use DENPO to post-train \texttt{Llama-3-8B-Instruct} and evaluate it across various benchmarks. 
    Experimental results show that DENPO consistently outperforms existing RLHF and NLHF baselines at matched computational cost, 
    demonstrating its potential to improve the performance of NLHF alignment methods. 
    This empirical evidence also highlights that our explicit policy exploration mechanism is useful not only theoretically but also in practical LLM alignment.
\end{enumerate}

\textbf{Related work.} Our work is connected to the literature on game-theoretic preference alignment and online exploration. Due to space limitations, we defer our comments on other relevant topics to Appendix~\ref{app:addition}. Recent works study game-theoretic alignment under general preference models. Early methods solve KL-regularized two-player minimax problems. In particular, NLHF learns a pairwise preference model and applies mirror descent against a geometric mixture of the current and reference policies~\citep{Munos-2024-Nash}, while exploratory NLHF provides offline pessimistic and online optimistic variants using conservative estimates or enhancer policies~\citep{Ye-2024-Online}. MaxMin-RLHF addresses diverse users by learning subpopulation-specific rewards via EM and updating toward the currently worst-aligned group~\citep{Chakraborty-2024-Maxmin}. However, these methods involve equilibrium solving, which may be hard to stabilize in practice. More recent approaches avoid such oracles through iterative self-play: SPO uses one no-regret learner with on-policy win-rate rewards~\citep{Swamy-2024-Minimaximalist}, SPPO updates the policy according to empirical win rates~\citep{Wu-2025-Self}, and INPO optimizes a preference-pair square loss without explicit per-response win-rate estimation~\citep{Zhang-2025-Iterative}. Greedy sampling analyses show that simple empirical updates can be statistically efficient under KL-regularized general preferences and generalized bilinear preferences~\citep{Wu-2025-Greedy, Lee-2026-Regularized}. MNPO extends previous works to multiplayer alignment with historical or heterogeneous opponents~\citep{Wu-2026-Multiplayer}.

Online exploration was developed in reinforcement learning theory to exploit interactive feedback by deliberately sampling informative and diverse behaviors. Modern exploration methods provide algorithmic principles that adapt to problem structure and achieve near-optimal regret bounds~\citep{Jiang-2017-Contextual, Foster-2023-Foundations}. However, these approaches are often computationally intractable in general RL settings~\citep{Jiang-2017-Contextual, Jin-2021-Bellman, Foster-2021-Statistical}. In preference alignment, practical approaches either collect feedback without updating the policy model~\citep{Dwaracherla-2024-Efficient} or rely on passive on-policy sampling~\citep{Guo-2024-Direct, Gao-2024-Rebel, Zhang-2025-Iterative}, which provides no explicit mechanism for novelty or diversity and can fail to discover rare high-quality responses. Exploratory approaches either impose restrictive assumptions~\citep{Novoseller-2020-Dueling, Xu-2020-Preference, Saha-2023-Dueling, Wu-2024-Making, Zhan-2024-Provable, Du-2024-Exploration, Das-2025-Active} or require computationally expensive procedures that are difficult to implement at LLM scale~\citep{Chen-2022-Human, Wang-2023-Standard, Ye-2024-Online}. Recently,~\citet{Xie-2025-Exploratory} have taken a first step toward simple and efficient exploration in RLHF under the BT reward model. Our work adapts the principle to online iterative NLHF under general preference models in LLM fine-tuning.

\section{Preliminaries and technical background}\label{sec:prelim}

\paragraph{Game-theoretic preference alignment.}
Modern LLMs are designed based on Transformer architectures~\citep{Vaswani-2017-Attention}: the input is a user prompt $\x \in \VCal^\star$ and a response $\y \in \VCal^\star$ is generated, where 
$\VCal^\star$ is the set of all finite-length sequences of tokens drawn from the vocabulary set $\VCal$. 
We cast an LLM as a policy $\pi_\theta(\y | \x)$ which assigns a probability to each $\y$ given $\x$. 
Given a base LLM, the generated responses might not be helpful, safe, or reliable. Thus, techniques for aligning the LLM with human preferences are applied.

We assume access to a general preference oracle $P^\star: \VCal^\star \times \VCal^\star \times \VCal^\star \to [0,1]$. Given a prompt $\x$ and two responses $\y,\y'\in\VCal^\star$, querying the oracle returns $I \sim \textnormal{Bern}(P^\star(\y\succ\y'\mid\x))$, where $I=1$ means that $\y$ is preferred to $\y'$, and $I=0$ means that $\y'$ is preferred to $\y$. We assume $P^\star(\y\succ\y'\mid \x)+P^\star(\y'\succ\y\mid\x)=1$ and $P^\star(\y\succ\y\mid\x)=\frac{1}{2}$. 
Given a pair of responses $(\y,\y')$, we define the preference distribution $\lambda_{P^\star}$ by
\begin{equation*}
\lambda_{P^\star}(\x,\y,\y') = \begin{cases}
(\y,\y'), & \textnormal{with probability } P^\star(\y\succ\y'\mid\x),\\
(\y',\y), & \text{with probability } 1-P^\star(\y\succ\y'\mid\x).
\end{cases}
\end{equation*}
Accordingly, we write $(\y^+,\y^-) \sim \lambda_{P^\star}(\x,\y,\y')$ for a pair of preferred and dispreferred responses. 

We consider \emph{game-theoretic} preference alignment, which formulates alignment as a KL-regularized game~\citep{Munos-2024-Nash, Ye-2024-Online}. Formally, we consider pairs of policies $\pi_1,\pi_2$ that play an alignment game against each other:
\begin{equation*}
    \max_{\pi_1} \min_{\pi_2} \ J_\beta(\pi_1,\pi_2) := P^\star(\pi_1\succ\pi_2) - \beta\textnormal{KL}(\pi_1\|\pi_\textnormal{ref})+\beta\textnormal{KL}(\pi_2\|\pi_\textnormal{ref}),
\end{equation*}
where $\pi_\textnormal{ref}$ is the model after SFT, 
$\beta>0$ is the KL-regularization parameter, 
and 
$P^\star(\pi_1\succ\pi_2) = \EE_{x\sim\rho, \y\sim\pi_1(\cdot\mid\x), \y'\sim\pi_2(\cdot\mid\x)} [P^\star(\y\succ\y'\mid\x)]$. 
The goal of both players is to maximize their win rate against the opponent while not deviating too far from $\pi_\textnormal{ref}$. 

Without loss of generality, we restrict our attention to the policy class containing the policies with the same support set as $\pi_\textnormal{ref}$. A Nash equilibrium (NE) is then defined as $(\pi_1^\star, \pi_2^\star)$ satisfying
\begin{equation*}
J_\beta(\pi_1, \pi_2^\star) \leq J_\beta(\pi_1^\star, \pi_2^\star) \leq J_\beta(\pi_1^\star, \pi_2), \quad \textnormal{for any } (\pi_1, \pi_2). 
\end{equation*}
Since the game is symmetric for the two players, as proven by~\citet{Ye-2024-Online}, the NE is unique and the two equilibrium policies coincide: $\pi_1^\star=\pi_2^\star=\pi^\star$. 
Since $J_\beta(\pi^\star, \pi^\star)=\frac{1}{2}$, we have $J_\beta(\pi^\star,\pi)\geq\frac{1}{2}$ for any policy $\pi$. For each policy $\pi$, we use the following duality gap to measure how well it approximates $\pi^\star$:
\begin{equation}
    \label{def:gap}
    \textnormal{gap}_\Pi(\pi) = \max_{\pi_1 \in \Pi} J_\beta(\pi_1, \pi) - \min_{\pi_2 \in \Pi} J_\beta(\pi, \pi_2). 
\end{equation}
In~\cref{app:subsec:setup}, we show that this duality gap is directly connected to the regret bound used to measure policy quality.

\paragraph{Online exploration.}
Under the BT model, the preference oracle can be represented as 
$P^\star(\y\succ\y'\mid\x) = \sigma(r^\star(\x,\y)-r^\star(\x,\y'))$,
where $\sigma: \br \mapsto [0, 1]$ is the sigmoid function. Direct alignment methods (e.g., DPO~\citep{Rafailov-2023-Direct}) optimize the policy $\pi_\theta$ over the dataset $D=\{(\x_i,\y_i^+,\y_i^-)\}_{i=1}^n$ by minimizing a loss of the following form:
\begin{equation*}
    \LCal_{\text{DPO}}(\theta) = -\EE_{(\x,\y^+,\y^-)\sim D}\left[\log \sigma\left(\beta\left(\log\tfrac{\pi_\theta(\y^+ | \x)}{\pi_{\text{ref}}(\y^+ | \x)}-\log\tfrac{\pi_\theta(\y^- | \x)}{\pi_{\text{ref}}(\y^- | \x)}\right)\right)\right].
\end{equation*}

In passive online DPO~\citep{Guo-2024-Direct}, 
the learner repeatedly samples responses from the current policy, queries the preference oracle, adds the resulting pair of preferred and dispreferred responses to the dataset, and updates the policy using the DPO loss. 
This procedure uses online feedback, but the learner only observes responses that the current policy already tends to generate. 
If the reference policy assigns very small probability to a high-reward response, then passive online DPO may fail to discover it for exponentially many rounds in the small-$\beta$ regime. 
If an online algorithm does not have such an exponential dependence, we say the algorithm has \emph{efficient exploration}. 
Thus, online feedback alone does not guarantee efficient exploration.

The key idea of XPO~\cite{Xie-2025-Exploratory} is to include an SFT-based optimism term when updating the policy using the DPO loss. At round $t$, sample $\x_t\sim \rho$, $\y_t\sim\pi_{\theta_t}(\cdot\mid\x_t)$, and $\y'_t\sim\pi_\textnormal{ref}(\cdot\mid \x_t)$. After querying the preference oracle, we obtain a new data sample $(\x_t, \y_t^+, \y_t^-)$ and update the dataset by $D_t=D_{t-1}\cup\{(\x_t, \y_t^+, \y_t^-)\}$. Then, we compute the next policy by minimizing
\begin{equation*}
\alpha \sum\nolimits_{i=1}^t \log(\pi(\y'_i\mid\x_i))-\EE_{(\x,\y^+,\y^-)\sim D_t}\left[\log \sigma\left(\beta\left(\log\tfrac{\pi_\theta(\y^+ | \x)}{\pi_{\text{ref}}(\y^+ | \x)}-\log\tfrac{\pi_\theta(\y^- | \x)}{\pi_{\text{ref}}(\y^- | \x)}\right)\right)\right]. 
\end{equation*}
\citet{Xie-2025-Exploratory} interpret this additional term 
through DPO's role as implicit $Q^\star$ approximation: 
the extra term implements optimism in the face of uncertainty. 
This yields efficient exploration and an improved regret bound.

Unfortunately, the existing analysis relies on assumptions that are absent in game-theoretic preference alignment. 
First, general preference oracles need not admit any scalar reward representation. 
Second, KL-regularized games do not admit a single objective to be maximized and lack an interpretation as implicit $Q^\star$ approximation. 
Finally, the goal of game-theoretic alignment is to minimize the duality gap against adversarial opponent policies. 
Thus, exploration must control comparisons against possible opponents, not only discover high-reward responses. 
This leaves open the question of how to achieve efficient exploration in iterative NLHF.

\section{Exploratory Nash preference optimization (ENPO) and sample efficiency}
\label{sec:theory}

\subsection{Deliberate exploration is necessary}
\label{subsec:necessity}

In this subsection, we show that naive exploration in iterative NLHF is insufficient to achieve efficient exploration, which requires eliminating exponential dependence on the KL-regularization parameter. 

Before presenting the theoretical analysis, we briefly review \emph{mirror descent}, a canonical no-regret method for computing Nash equilibria. 
In mirror descent dynamics, at each round, every player updates her strategy along the first-order improvement direction against the current opponent profile---which coincides with her own strategy in our symmetric setting---regularized by a KL divergence that anchors the new iterate to the previous one: 
\begin{equation}
    \label{eq:exact-update-argmax}
    \pi^{\textnormal{next}}_{\star, \eta}[\pi] = \underset{\tilde{\pi}}{\textnormal{argmax}} \underset{\x \sim \rho}{\E}\Big[\eta \underset{\y \sim \tilde{\pi}(\cdot|\x)}{\E}\Big[ P^*(\y \succ \pi | \x) - \beta \log\Big(\frac{\pi(\y | \x)}{\pi_\tnref(\y | \x)} \Big) \Big] - \textnormal{KL}(\tilde{\pi}(\cdot | \x) \Vert \pi(\cdot | \x)) \Big].
\end{equation}
Under an appropriate step-size schedule, the dynamics converges to the regularized NE and attains a sublinear regret bound.
In essence, the INPO-type algorithm aims to track the mirror descent update at every round.

To study sample complexity guarantees for iterative NLHF methods, we make the following statistical assumptions.
The second assumption guarantees that the policy class $\Pi$ is large enough to ensure that the mirror descent update of every $\pi$ lies in $\Pi$, unless $\pi$ is sufficiently close to the NE\footnote{
    The second condition in~\cref{assumption:main} is analogous to the “policy realizability for general preferences” assumption in~\citet{Huang-2025-Correcting}, but is slightly relaxed to accommodate finite policy classes.
}.
\begin{assumption}
    \label{assumption:main}
    We assume that $P^\star$ and $\Pi$ satisfy the following conditions: 
    \begin{itemize}[leftmargin=2em]
        \item \textbf{Finite policy class}: $\Pi$ is finite, and $\pref, \pi^\star \in \Pi$. 
        \item \textbf{Mirror descent realizability}: For each $\eta$ in a finite set of stepsizes, there exists a sufficiently small \(\varepsilon_\eta\ge0\) such that $\pi^{\textnormal{next}}_{\star,\eta}[\pi]\in\Pi$ for every \(\pi\in\Pi\) satisfying $\KL(\pi^\star\|\pi)>\varepsilon_\eta$.
        \item \textbf{Bounded log-ratio}:
        For any $\pi, \pi' \in \Pi$, and any $\x, \y, \y' \in \VCal^\star$, we have $\big| \log\frac{\pi(\y \mid \x)}{\pi'(\y \mid \x)} \big| \leq {V}/{\beta}$ where $V\geq 1$ is a constant.
        \item \textbf{Symmetric preferences}:
        For any $\x, \y, \y' \in \VCal^\star$, 
        $P^\star(\y \succ \y' | \x) + P^\star(\y' \succ \y | \x) = 1$.
    \end{itemize}
\end{assumption}

In contrast to offline algorithms such as DPO, Nash learning algorithms such as INPO intrinsically draw new samples from the policy that mimics the mirror descent trajectory, 
and thus naturally perform on-policy sampling with moving policies.
It is therefore natural to ask whether this iterative sampling process helps with exploration, and whether it is sufficient to achieve efficient exploration without further modifications.
The following proposition shows that, even with iterative sampling, INPO still requires $\Omega(e^{1/\beta})$ samples to achieve any sublinear regret bound.
We defer the formal algorithm statement and its proof to~\cref{app:subsec:necessity}.
\begin{proposition}
    \label{prop:lower-bound-1}
    There exists an instance satisfying~\cref{assumption:main}. 
    Let $\pi_t$ be a learned policy of INPO (described in~\cref{algo:inpo}) on this instance after round $t$ with constant batch size $n$.
    \rev{Then, for any \(\delta\in(0,1)\), if $4n(T+1)\le\delta e^{1/(4\beta)}$, with probability at least \(1-\delta\),}
    $\textnormal{gap}_\Pi(\pi^t)$ is at least a constant for all $t \leq T$
    and $\Reg_T \geq \Omega(T)$.
\end{proposition}

Recall that we refer to sample complexity bounds that do \emph{not} scale exponentially in $1/\beta$ as \emph{sample efficient}.
Without sample efficiency, iterative NLHF methods may require an extremely large number of iterations to discover useful information in the small-$\beta$ regime. We therefore make \emph{sample efficiency} the main objective of this paper.

\subsection{Exploratory Nash preference optimization (ENPO)}

In this subsection, we present our algorithmic solution to inefficient exploration in the iterative NLHF framework: Exploratory Nash Preference Optimization (ENPO).

We first introduce additional notation and the main loss function. We use $\tau := (\x,\y)$ to denote a prompt--response pair.
Assume that, at round $t$, we have empirical preference estimates for $\tau$ against a policy $\pi$, denoted by $\widehat P_t(\tau\succ\pi)$. We then define $
    \widehat\delta_t(\tau;\pi):=
    \log\frac{\pi(\tau)}{\pi_{t-1}(\tau)}
    -\eta\beta\log\frac{\pref(\tau)}{\pi_{t-1}(\tau)}
    -\eta\widehat P_t(\tau\succ\pi_{t-1})$,
and an empirical loss function: 
\begin{equation}
    L_{t-1}(\pi)
    :=\sum_{s=0}^{t-1}\frac{1}{n_{s+1}}
      \sum_{i=1}^{n_{s+1}}
        \left( \widehat\delta_t(\tau_{s,i};\pi)-\widehat\delta_t(\tau'_{s,i};\pi) \right)^2.
    \label{eq:learner-loss}
\end{equation}
The minimizer of the population version of~\cref{eq:learner-loss} corresponds to the exact mirror descent target at the current round.

Each iteration $t$ in ENPO has three steps.
\emph{(i) Data collection.} 
ENPO samples response pairs, 
each with one response from an earlier iterate and 
one from the explorer policy $\hat\pi_{t-1}$ chosen in the previous round, 
and adds them to the pairwise dataset $\mathcal{D}_2^t$. 
It also draws fresh responses from $\pi_{t-1}$ to form $\mathcal{D}_1^t$. 
For each response $\tau$ in $\mathcal{D}_2^t$, ENPO estimates its win rate against the current policy: $\widehat{P}_t(\tau \succ \pi_{t-1})$. 
These estimates define the loss $L_{t-1}$.
\emph{(ii) Pessimistic exploitation.} 
The new policy $\pi_t$ minimizes $L_{t-1}$ plus an SFT-style term that rewards likelihood on the learner's own recent samples in $\mathcal{D}_1^t$. 
Unlike the optimism bonus in exploratory DPO, this term keeps $\pi_t$ close to $\pi_{t-1}$, which strengthens the mirror-descent structure of the update.
\emph{(iii) Adversarial exploration.} ENPO then chooses the explorer $\hat\pi_t$ by solving a policy optimization problem. 
The explorer is a sampling policy under which $\pi_t$ can still disagree most with some alternative policy, relative to how tightly the pairs collected so far constrain that disagreement. 
It thus targets opponent distributions that expose poorly covered regions of the preference game.
After $T$ rounds, ENPO outputs the uniform mixture of the iterates or the iterate with the smallest duality gap. 
We summarize ENPO in~\cref{alg:adv-exp-NashMD}. 

\begin{algorithm}[t]
    \caption{Exploratory Nash Preference Optimization (ENPO) for Theoretical Analysis}
    \label{alg:adv-exp-NashMD}
    \KwIn{
        horizon $T$, 
        schedules \((m_t, n_t, k_t)_{t=1}^T\), 
        coefficients $\gamma, \lambda>0$, 
        stepsize \( \eta>0 \), 
        reference policy \(\pi_{\rm ref}\), 
        pessimistic parameter $\alpha_t$}
    \textbf{Initialize} $\pi_{-1} \gets \pi_\tnref$, $\pi_0 \gets \pi_\tnref$, $\hat{\pi}_0 \gets \pi_\tnref$\;
    \textbf{Initialize} $\mathcal{D}_{1}^0 \gets \varnothing, \mathcal{D}_{2}^0 \gets \varnothing$\;

    \For{$t = 1, \ldots, T$}{
        Generate $n_t$ response pairs $(\tau_{t - 1,i}, \hat{\tau}_{t - 1,i})$ via $\tau^{t - 1, i} \sim \pi_{t - 2}, \hat{\tau}^{t - 1, i} \sim \hat{\pi}_{t - 1}$\;
        Generate $m_t$ responses $\tilde{\tau}_{t - 1, i}$ via $\tilde{\tau}_{t - 1, i} \sim \pi_{t - 1}$\;
        Update preference datasets $D^t_{1} \gets \{\tilde{\tau}_{t - 1, i} \}_{i=1}^{m_t}$, $D^t_{2} \gets D^{t-1}_{2} \cup \{ (\tau_{t - 1, i}, \hat{\tau}_{t - 1, i}) \}_{i=1}^{n_t}$\; 
        Define $\A_t^{\rm act} := \{\tau,\tau':(\tau,\tau')\in D_2^t\}$. Generate \(c_{t,1},\dots,c_{t,k_t}\sim\pi_{t-1}\) independently. For every \(a\in \A_t^{\rm act}\), query $Y_{t,a,j}\sim \mathrm{Bern}(P^*(a\succ c_{t,j}))$ for $j=1,\dots,k_t$, and set $\widehat P_t(a\succ\pi_{t-1}) = \frac1{k_t}\sum_{j=1}^{k_t}Y_{t,a,j}$\; 
        Compute 
        \begin{equation}
            \label{eq:learner-update}
            \pi_{t} \gets \underset{\pi \in \Pi}{\textnormal{argmin}} \left\{ - \alpha_t \sum\nolimits_{\tau \in D^{t}_{1}} \log{\pi(\tau)} + L_{t-1}(\pi) \right\}, 
        \end{equation}
        where $L_{t-1}$ is defined in~\cref{eq:learner-loss}\;
        Compute 
        \begin{equation}
            \label{eq:explorer-update}
            \hat{\pi}_{t} \gets \underset{\tilde{\pi} \in \Pi}{\textnormal{argmax}} \max_{\pi \in \Pi} \frac{\big( \frac{1}{m_t}\sum_{i=1}^{m_t} \log{\frac{\pi_t(\tilde{\tau}_{t - 1, i})}{\pi(\tilde{\tau}_{t - 1, i})}} - \EE_{\tau' \sim \tilde{\pi}}\big[ \log{\frac{\pi_t(\tau')}{\pi(\tau')}} \big] \big)^2}{\lambda + 
            \sum_{s=0}^{t-1}\frac{1}{n_{s+1}}\sum_{i=1}^{n_{s+1}}
            \big(
            \log \frac{\pi_t(\tau_{s,i})}{\pi(\tau_{s,i})}
            - \log \frac{\pi_t(\tau'_{s,i})}{\pi(\tau'_{s,i})}
            \big)^2}. 
        \end{equation}
    }
    \KwOut{$\textnormal{Uniform}(\{ \pi^t \}_{t=0}^{T-1})$ or $\pi_{\textnormal{best}} \in \{ \pi^t \}_{t=0}^{T-1}$ with the best duality gap.}
\end{algorithm}

\subsection{Sample efficiency guarantee of ENPO}

The example in~\cref{subsec:necessity} suggests that relying solely on the exploration induced by the mirror descent trajectory is insufficient for achieving sample efficiency. More broadly, it has long been recognized in the RL literature that learning a Nash equilibrium is substantially harder than single-agent RL, due to strategic interaction and the need for coordinated exploration~\cite{zhang2020model,loftin2021strategically,huang2024model}. In particular, na\"ive optimistic exploration may waste samples on strategically irrelevant regions of the state space~\cite{loftin2021strategically}. 
The separation between general preferences and single-reward models has also been established in the offline RL setting~\citep{cui2022offline,Huang-2025-Correcting}.
More generally, computing an equilibrium is well known to be computationally hard~\cite{daskalakis2009complexity}.

The main difficulty of achieving sample efficiency under general preferences can be viewed from two perspectives. First, even in the symmetric-game setting, a player’s effective reward depends on her own policy, so the learning objective is nonstationary rather than fixed. Second, game-theoretic performance measures are defined relative to either a Nash equilibrium or an adversarial policy, which may be completely out-of-sample. 
This gives rise to terms of the form 
\begin{equation*}
    \mathbb{E}_{\x \sim \rho, \y \sim \pi_t(\cdot|\x), \y' \sim \pi(\cdot|\x)}[\psi_t(\x, \y, \y') - \psi'_t(\x, \y, \y')] 
\end{equation*}
in the regret decomposition, where $\pi$ may be \emph{completely unpredictable} at time $t$. As a result, even abundant in-sample data may provide little control over this term, leading to arbitrarily poor worst-case performance.
In particular, this suggests that simple modifications such as proximal exploration in~\cite{Xie-2020-Q} are unlikely to achieve sample efficiency for equilibrium learning.

Next, we show that the nontrivial ``pessimistic exploitation + optimistic exploration'' structure in ENPO aligns naturally with mirror descent dynamics. 
By strengthening the regularization in the mirror descent update and then performing deliberate exploration, ENPO achieves efficient exploration.
\begin{theorem}
    \label{thm:relaxed-realizability-best}
    Assume that $P^\star, \Pi$ satisfy~\cref{assumption:main} with $\varepsilon_\eta \leq c/T$ for a stepsize $\eta = 2\beta/(1 + 4\beta^2) < 1$.
    Let $M_T:=T(T+1)(T+2)/3$ and 
    \[
        L_T:=\max\{1,L_\Pi,L_P,\log(eT)\}, \text{ where } L_\Pi:=\log\left(16T|\Pi|^2\delta^{-1}\right), \; L_P:=\log\left(16M_T\delta^{-1}\right). 
    \]
    Set $m_t = n_t = k_t = t$, 
    $\lambda=16 (1 + (V^2 + 1)/\beta^2) L_T^2$, 
    $\gamma = \sqrt{(1\vee\SEC(\Psi,\lambda))/(\lambda T)}$, 
    $\alpha_t=2/(\gamma t)$, and stepsize $\eta = 2\beta/(1 + 4\beta^2) < 1$ in~\cref{alg:adv-exp-NashMD}. 
    Then, with probability at least \(1-\delta\), 
    \[
        \Reg_T^{\rm best} \le \max\left\{
            \frac{1+2e^{\eta(2V+1)}}{\eta}
            \left( B+20\sqrt{(1\vee\SEC(\Psi,\lambda))\lambda T} \right),
            (1+4V)\sqrt{\frac{cT}{2}} \right\}.
    \]
\end{theorem}
The proof of this theorem uses a regret decomposition analogous to that in~\cite{Xie-2020-Q} and its connection to online data coverage through the deliberate exploration step.
We provide a complete proof of~\cref{thm:relaxed-realizability-best} in~\cref{app:sec:efficient-exploration-guarantee}.

\paragraph{Discussion.} 
This is the first sample-efficiency result of this type available in the literature. Although iterative Nash preference optimization has dominated empirical practice and achieved significant progress, the only known sample-complexity result directly based on a mirror-descent oracle is that of~\cite{Huang-2025-Correcting}, which is offline and therefore does not guarantee no-regret performance when the reference policy is suboptimal. In the more natural online setting, sample-complexity guarantees are scarce, let alone sample-efficient guarantees. Our result also resolves an open problem posed by~\cite{Lee-2026-Regularized}, which asks for efficient algorithms based on mirror-descent oracles.

In the stated algorithm, 
we sample an $O(t^3)$-sized batch at each iteration, 
which could potentially be reduced through better sampling strategies.
In this paper, we focus on eliminating the exponential dependence in iterative NLHF methods and do not optimize the batch size.
Note that, even under the BT model, the general XPO algorithm of~\cite{Xie-2025-Exploratory} also requires $t$ samples per iteration. 
Their algorithm can avoid this only when samples can be fully reused, i.e., when a stationary policy is used as the sampler; this is incompatible with efficient learning in the game setting.

\section{Bonus-Explorer ENPO and logarithmic regret bound}
\label{sec:logT}

Our ENPO algorithm achieves sample efficiency without relying on either a minimax oracle or an explicit preference model, making it a natural basis for practical algorithms. 
At the same time, more complex oracle-based methods remain a central paradigm in RL, where access to best-response or minimax planning oracles often enables sharper regret guarantees. 
Recent results further suggest that, in KL-regularized zero-sum Markov games, oracle-style optimistic best-response sampling can improve the standard $O(\sqrt T)$ rate to logarithmic regret. Motivated by this gap, we next study efficient exploration under access to minimax oracles and ask whether one can design an exploration strategy that achieves $O(\log T)$ regret in our preference-based setting.

We further design a minimax-oracle variant of ENPO, called Bonus-Explorer ENPO (BENPO), which is equipped with minimax oracles and best-response computation. This variant achieves an improved $O(\log T)$ regret bound without any $e^{1/\beta}$ dependence. Its analysis builds on recent techniques from KL-regularized RL~\citep{zhao2025logarithmic,Wu-2025-Greedy,Lee-2026-Regularized}; in particular, we exploit the structure of KL regularization to convert first-order errors into \emph{second-order} errors, which ultimately enables the logarithmic rate. We conclude with the main theorem here, and defer the formal algorithm and analysis to~\cref{app:subsec:BENPO}.
\begin{theorem}[Logarithmic regret bound of~\cref{alg:BENPO}]
    \label{thm:benpo-regret-main}
    Assume that $\mathcal{P}$ is a finite class of antisymmetric preference functions with $P^\star \in \mathcal{P}$, and let $d_\mathcal P$ be its eluder coefficient (see~\cref{app:sec:BENPO-additional-assumptions}).
    Then, with probability at least $1 - \delta$, \cref{alg:BENPO} with $\lambda \in (0, 1]$ and $\gamma_t(\delta)=4\log(2T|\mathcal P|/\delta)+\lambda$ outputs iterates $\{\pi_t\}_{t=1}^T$ that satisfy 
    \begin{equation*}
        \sum_{t=1}^T \textnormal{gap}_{\Pi}(\pi_t) \leq O\left(\frac{d_\mathcal{P}\,\log(T |\mathcal{P}|/\delta)}{\beta}\right),
    \end{equation*}
    where the duality gap (1) is taken over all policies with the same support set as $\pi_{\mathrm{ref}}$.
\end{theorem}

\paragraph{Discussion.}
To the best of our knowledge, BENPO is the first algorithm in the NLHF literature to \emph{simultaneously} achieve $O(\log T)$ regret and avoid the $\exp({1/\beta})$ regularization scaling under standard reverse-KL regularization and general function approximation. In comparison, \cite{Wu-2025-Greedy,Lee-2026-Regularized} achieve $O(\log T)$ regret bounds but still incur an $\exp({1/\beta})$ dependence under KL regularization, while \cite{nayak2025achieving} avoid this dependence but require linear function approximation.
We note, however, that all existing algorithms achieving $O(\log T)$ regret rely on minimax oracles, leaving open whether logarithmic regret can be attained using only optimization oracles.

One may observe that BENPO is closely related to the Optimistic ELHF algorithm of~\citet{Ye-2024-Online}. Indeed, both algorithms learn a minimax policy together with an explorer that maximizes a UCB-type bonus. The key difference is that Optimistic ELHF trains a policy-level explorer, whereas \texttt{BENPO} searches over the \emph{response space} for each context $\x$. This finer-grained exploration structure is what enables our sharper analysis.

\section{Direct ENPO and Experiments}
\label{sec:exp}

\subsection{Direct ENPO (DENPO)}

For applications such as LLM fine-tuning, we use a lighter algorithm.
We therefore present \emph{Direct ENPO} (DENPO), a practical instantiation that keeps ENPO's two exploration mechanisms but replaces these oracles with simple surrogates. 
The policy update becomes a few gradient steps on the iterative preference loss plus the SFT-style term on self-generated responses, so each round costs about the same as standard iterative NLHF. 
For exploration, DENPO does not solve~\cref{eq:explorer-update}. Instead, it uses a surrogate to approximate the adversarial opponent without a separate optimization procedure. 
We defer a full introduction to DENPO to~\cref{app:sec:DENPO}.

As a result, DENPO uses the same ``pessimistic exploitation + optimistic exploration'' structure as ENPO, with an even lighter implementation. 
This design makes DENPO compatible with LLM fine-tuning while preserving ENPO's explicit exploration mechanism.
In the subsequent numerical experiments, 
we show that DENPO consistently outperforms RLHF and NLHF baselines on \texttt{Llama-3-8B-Instruct} at matched computational cost.

\subsection{Numerical experiments}

\paragraph{Experimental setup.}
We compare our DENPO algorithm with 
XPO~\citep{Xie-2025-Exploratory} and 
INPO~\citep{Zhang-2025-Iterative}, 
using \texttt{Llama-3-8B-Instruct} as the base model.
For INPO and DENPO, 
we sweep $1/\eta$ over the following inverse stepsize values: $\{7.5, 3.75, 1.875, 0.9375\} \times 10^{-3}$.
The corresponding KL regularization coefficients $\beta$ are $\{2.5, 2.5, 1.25, 0.625\} \times 10^{-3}$, respectively, for both methods.
For ease of presentation, we use the subscripts $1/2$, $1/4$, and $1/8$ to denote scaling $1/\eta$ by these factors relative to the default value $1/\eta = 7.5 \times 10^{-3}$.

All methods use the same $20{,}000$ UltraFeedback prompts per iteration for $T = 3$ iterations, 
with \texttt{Skywork-Reward-V2-Llama-3.1-8B-40M} providing training-time preference feedback, 
a learning rate of $10^{-6}$, 
and one epoch per optimization stage.
Training-time generation uses temperature $0.7$, top-$p$ $0.9$, and at most $512$ new tokens.
Let $K$ be the number of generated completions in a batch.
INPO samples $K = 8$ responses per prompt from the current policy; XPO samples $K = 8$ responses from each of the current and reference policies.
Both retain the top $70\%$ of candidate pairs by reward margin.
DENPO uses the single-pair configuration with $K = 2$ and $n_z = 4$ comparator responses from the current policy, trains on all prompts without a replay buffer, and updates both the policy and the adversary in each iteration.
The comparison matches prompt budgets and iteration counts, with method-dependent computational costs.
We report comparisons at matched computational cost in~\cref{sec:denpo-compute}.
Experiments are conducted on a node with $8 \times \textnormal{H100}$ GPUs.

\paragraph{Evaluation.}
For each iteration-$3$ policy, 
we report four chat-quality metrics.
We report performance across all three iterations in~\cref{sec:denpo-trajectories}.

The two AlpacaEval-2 metrics are length-controlled win rates under \texttt{gpt-4o-mini} and \texttt{gpt-4-turbo}.
We also report the MT-Bench mean score under \texttt{gpt-4o-mini} and 
the Arena-Hard-v0.1 win rate against \texttt{gpt-4-0314}, judged by \texttt{gpt-4o-mini} and estimated using Bradley-Terry scoring with bootstrap resampling over $500$ prompts.
AlpacaEval-2 uses its $805$-prompt benchmark, and MT-Bench contains $80$ two-turn questions.
For every trained configuration,~\cref{tab:leaderboard} reports the mean and sample standard deviation across three training runs. 
These standard deviations describe variation across runs.

\begin{table}[!t]
\centering
\small
\setlength{\tabcolsep}{9pt}
\caption{Iteration-$3$ performance, reported as mean $\pm$ sample standard deviation across three training runs.
AE2-mini and AE2-turbo are AlpacaEval-2 length-controlled win rates (\%) under the respective judges; MT-Bench is the mean score; Arena-Hard is the win rate (\%).
Higher is better for all metrics.
Bold marks the highest mean in each column, including ties at the displayed precision.
}
\label{tab:leaderboard}
\begin{tabular}{@{}lrrrrr@{}}
\toprule
Policy & $\eta^{-1} (\times 10^{-3})$
    & AE2-mini & AE2-turbo & MT-Bench & Arena-Hard \\
\midrule
Base model & ---
    & $34.83$ & $28.61$ & $6.64$ & $36.40$ \\
XPO & ---
    & $47.06 \pm 1.46$ & $40.75 \pm 0.59$
    & $6.91 \pm 0.05$ & $45.33 \pm 2.57$ \\
\midrule
INPO & $7.5$
    & $53.69 \pm 0.38$ & $45.81 \pm 0.52$
    & $6.99 \pm 0.05$ & $50.83 \pm 1.85$ \\
INPO$_{1/2}$ & $3.75$
    & $52.09 \pm 2.59$ & $44.15 \pm 1.82$
    & $6.92 \pm 0.07$ & $48.57 \pm 2.03$ \\
INPO$_{1/4}$ & $1.875$
    & $46.25 \pm 2.89$ & $39.61 \pm 1.82$
    & $6.47 \pm 0.11$ & $38.70 \pm 2.36$ \\
INPO$_{1/8}$ & $0.9375$
    & $30.24 \pm 2.40$
    & $26.77 \pm 1.62$
    & $5.38 \pm 0.35$
    & $17.67 \pm 5.39$ \\
\midrule
DENPO & $7.5$
    & $52.03 \pm 1.34$ & $44.29 \pm 0.84$
    & $6.93 \pm 0.07$ & $50.70 \pm 2.08$ \\
DENPO$_{1/2}$ & $3.75$
    & $52.56 \pm 3.14$ & $45.54 \pm 0.87$
    & $6.97 \pm 0.14$ & $53.07 \pm 2.22$ \\
DENPO$_{1/4}$ & $1.875$
    & $57.57 \pm 2.86$ & $47.48 \pm 2.11$
    & $\mathbf{7.04} \pm 0.08$ & $51.43 \pm 3.04$ \\
DENPO$_{1/8}$ & $0.9375$
    & $\mathbf{58.33} \pm 1.00$
    & $\mathbf{47.93} \pm 1.49$
    & $\mathbf{7.04} \pm 0.09$
    & $\mathbf{55.10} \pm 2.17$ \\
\bottomrule
\end{tabular}
\end{table}

\paragraph{Results and discussion.}
As shown in~\cref{tab:leaderboard}, 
DENPO$_{1/8}$ achieves the highest mean scores on both AlpacaEval-2 judges and Arena-Hard, 
and ties DENPO$_{1/4}$ on MT-Bench.
Compared with default INPO, it improves the mean AE2-mini, AE2-turbo, and Arena-Hard win rates by $4.64$, $2.12$, and $4.27$ percentage points, respectively.
At the same $1/\eta = 0.9375 \times 10^{-3}$, INPO falls below the initial-model baseline on all four metrics in every run, whereas DENPO remains above the baseline throughout these evaluations.
This robust contrast supports DENPO's greater empirical stability and superior performance across the tested configurations.

DENPO performs better as the stepsize parameter $\eta$ increases (as $\eta^{-1}$ decreases). 
These results indicate that standard INPO cannot work with a large stepsize, whereas stabilization and active exploration can help DENPO make progress toward the mirror-descent target even with a large stepsize.

DENPO also exhibits robust performance. Across three random seeds, the standard deviations of DENPO's scores are substantially smaller than its improvements over the baselines. This indicates that DENPO's gains are consistent across runs and unlikely to be explained by random variation alone.

\section{Conclusion}\label{sec:conclu}
We studied online iterative NLHF under general preference models and identified exploration as the key obstacle to provable regret guarantees. We showed that standard iterative NLHF, which relies on implicit exploration through its own policy updates, can incur an exponential dependence on the KL-regularization parameter. To address this, we proposed ENPO, which augments the iterative preference loss with two mechanisms: a pessimistic SFT-style regularizer that stabilizes the mirror-descent update, and an adversarial policy search that samples from opponents exposing poorly covered regions of the preference game. ENPO solves a major problem in online iterative NLHF by achieving \emph{sample efficiency}: a regret bound without an exponential dependence on the KL-regularization parameter. It retains the direct policy-optimization structure of iterative NLHF and requires neither minimax oracles nor preference-model training. 
Its variant BENPO improves the regret bound to $O(\log T)$ given additional oracles. 
Finally, our practical algorithm DENPO replaces these oracles with surrogates and consistently outperforms RLHF and NLHF baselines on \texttt{Llama-3-8B-Instruct} at matched computational cost. These results show that explicit exploration is useful not only in theory but also in practical LLM alignment.
An interesting future direction is to reduce the per-iteration batch sample sizes required by our analysis.

\section*{Acknowledgment}
We sincerely appreciate Buzz High Performance Computing (\href{https://www.buzzhpc.ai}{\texttt{https://www.buzzhpc.ai}}, \url{info@buzzhpc.ai}) for providing computational resources and support for this work.
The research of Xiaopeng Li was supported by a start-up fund at the Chinese University of Hong Kong, Shenzhen.
The research of Christian Kroer was supported by the Office of Naval Research awards N00014-22-1-2530 and N00014-23-1-2374, and the National Science Foundation awards IIS-2147361 and IIS-2238960.

\bibliographystyle{abbrvnat}
\bibliography{ICLR_ref}

\clearpage
\appendix

\begin{center}
    {\Large Appendices}
\end{center}

\tableofcontents \newpage

\section{Further Related Work}\label{app:addition}
We comment on other topics, including further discussion of game-theoretic alignment methods, theoretical and empirical preference alignment methods, and theoretical reinforcement learning (RL). For an overview of preference alignment methods, we refer to the recent survey~\citep{Casper-2023-Open}. 

\paragraph{Further discussion of game-theoretic alignment methods.} Beyond algorithmic progress, game-theoretic alignment is motivated by a more structural question: what alignment properties are induced by treating pairwise human preferences as a game rather than compressing them into a scalar reward? This perspective is especially relevant when preferences are cyclic, non-transitive, or incompatible with the BT reward model, where standard RLHF might fail to satisfy social choice and diversity desiderata~\citep{Bradley-1952-Rank, Luce-2005-Individual, Mishra-2023-AI, Conitzer-2024-Position, Dai-2024-Mapping, Xiao-2025-Algorithmic}. Recent results have shown that the appeal of NLHF is not merely computational: with raw pairwise preference as the payoff, NLHF is Condorcet consistent, Smith consistent under a no-tie condition, and can preserve preference diversity through mixed strategies when no single response dominates all alternatives~\citep{Maura-2025-Jackpot, Liu-2025-Statistical}. However,~\citet{Shi-2025-Fundamental} emphasize that raw preference is only one possible payoff choice and study the more general game with payoff $\Psi(P(\y\succ\y'\mid\x))$, where $\Psi(t)=t$ recovers standard NLHF and $\Psi(t)=\log(\frac{t}{1-t})$ connects to reward-based objectives under the BT reward model. Their analysis shows both robustness and limits: Condorcet consistency only requires preserving the majority threshold, while Smith consistency requires an anti-symmetry condition $\Psi(t)+\Psi(1-t)=2\Psi(\frac{1}{2})$, implying that practical preference models must preserve pairwise anti-symmetry to guarantee support on the Smith set. At the same time, mixed strategies provide only a partial notion of diversity. In the stronger sense of preference matching, which asks the learned policy to match a target distribution reflecting human preference diversity,~\citet{Shi-2025-Fundamental} prove an impossibility result: no smooth and learnable mapping of pairwise preferences can generally guarantee a unique Nash equilibrium equal to the target policy, even under the standard BT model~\citep{Xiao-2025-Algorithmic, Shi-2025-Fundamental}. These results suggest that game-theoretic alignment is well motivated by social choice consistency and robustness to non-transitive preferences, but its objective has axiomatic limitations that are different from the exploration mechanism studied in this paper.

\paragraph{Theoretical preference alignment methods.} Recent progress in theoretical preference alignment has developed along several complementary directions. Under the BT reward model, early work established finite-sample guarantees for RLHF and preference-based reinforcement learning in structured settings such as tabular MDPs, linear MDPs, linearly parameterized policies, and KL-constrained policy classes \citep{Novoseller-2020-Dueling, Xu-2020-Preference, Saha-2023-Dueling, Zhu-2023-Principled, Li-2023-Reinforcement, Zhan-2024-Provable, Xiong-2024-Iterative}. Another line connects alignment to offline reinforcement learning theory, where the central challenge is distribution shift between the reference policy and candidate aligned policies, leading to analyses based on coverage, uncertainty quantification, and regularization \citep{Liu-2020-Provably, Jin-2021-Pessimism, Rashidinejad-2021-Bridging, Xie-2021-Bellman, Uehara-2022-Pessimistic, Zhan-2022-Offline, Chen-2022-Offline}. Recent work studies direct preference objectives themselves: DPO provides a policy-only formulation of KL-regularized preference optimization~\citep{Rafailov-2023-Direct}, while subsequent theory considers broader divergence choices, generalized preference losses, and additional regularization mechanisms \citep{Wang-2024-Beyond, Tang-2024-Generalized, Liu-2024-Provably, Cen-2025-Value, Fisch-2025-Robust}. XPO and $\chi$PO are representative of this recent shift toward simple direct objectives with end-to-end guarantees under general function approximation or offline alignment assumptions~\citep{Xie-2025-Exploratory, Huang-2025-Correcting}. Other works conduct a rigorous analysis through the lens of dataset coverage to differentiate offline DPO and online RLHF~\citep{Song-2024-Importance} and report faster convergence rates than the information-theoretic lower bounds for online reward maximization in RL by exploiting the structure induced by KL regularization~\citep{Shi-2025-Crucial}. Beyond scalar-reward preference models, another emerging theoretical direction studies general, possibly non-transitive preferences through game-theoretic solution concepts, replacing reward maximization with duality-gap guarantees~\citep{Munos-2024-Nash, Swamy-2024-Minimaximalist, Rosset-2024-Direct, Ye-2024-Online}.

\paragraph{Empirical preference alignment methods.} Empirical preference alignment methods have mainly focused on improving the data construction, objective design, and robustness of direct fine-tuning pipelines. Since offline preference optimization is sensitive to the source and quality of comparison pairs, several works modify the preference data itself: SLiC-HF calibrates sequence likelihoods using preference rankings, while statistical rejection sampling constructs more informative preference pairs from refined model samples rather than relying only on static SFT-generated data~\citep{Zhao-2023-Slic,Liu-2024-Statistical}. A second line redesigns the alignment loss to address practical failure modes. KTO replaces pairwise comparisons with prospect-theoretic desirable/undesirable feedback~\citep{Ethayarajh-2024-Model}; length-aware variants study verbosity and length-quality confounding~\citep{Park-2024-Disentangling}; contrastive preference objectives and SimPO adjust reward parameterization, normalization, or margins to improve performance~\citep{Xu-2024-Contrastive, Meng-2024-SimPO}. More recent variants broaden empirical preference alignment beyond a fixed pairwise-loss formulation, including comparison-oracle optimization~\citep{Chen-2026-Compo} and reward-free multi-objective alignment under conflicting preferences~\citep{Chen-2026-Raco}. Another empirical direction studies reward hacking and overoptimization, proposing practical stabilizers such as reward-model ensembles, uncertainty-aware or adversarial policy optimization, demonstration-guided calibration, and positive-response correction \citep{Coste-2024-Reward, Eisenstein-2024-Helping, Rita-2024-Countering, Pal-2024-Smaug, Zhang-2024-Overcoming}. Finally, recent work on sampling dynamics shows that sampling can change ranking behavior and concentration, while iterative deployment choices can lead to entropy collapse~\citep{Chen-2026-Sampling}.

\paragraph{Theoretical reinforcement learning (RL).} The RL theory relevant to modern alignment has developed around structural complexity, coverage, regularization, and computational tractability. One line studies sample-efficient RL with rich function approximation by identifying complexity measures under which prediction error can be converted into policy improvement, including eluder-dimension, contextual-decision-process, general value-function-approximation, bilinear-class, and decision-estimation-coefficient frameworks~\citep{Russo-2013-Eluder, Sun-2019-Model, Wang-2020-Reinforcement, Du-2021-Bilinear, Xie-2023-Role, Foster-2023-Tight}. These results clarify when statistically efficient learning is possible, but they also reveal a persistent computational-statistical tension: many general optimistic or version-space methods require oracle access, constrained optimization, or other procedures that are difficult to implement at scale~\citep{Dann-2018-Oracle, Kane-2022-Computational, Golowich-2024-Exploration}. A second line analyzes offline and off-policy learning through coverage and concentrability, where guarantees depend on how errors under a behavior distribution transfer to comparator policies or value functions~\citep{Farahmand-2010-Error, Xie-2020-Q, Zanette-2021-Provable}. This viewpoint is closely related to policy learning and off-policy evaluation in statistics and econometrics, where doubly robust and semiparametric methods are used to control distribution shift under observational data~\citep{Chernozhukov-2019-Semi, Kallus-2020-Double, Athey-2021-Policy}. A third line studies regularized RL objectives: maximum-entropy and KL-regularized MDPs provide soft dynamic-programming identities and policy representations~\citep{Ziebart-2008-Maximum, Ziebart-2010-Modeling, Neu-2017-Unified}, while recent finite-sample analyses establish sharper guarantees for entropy- or KL-regularized learning in tabular settings~\citep{Kozuno-2022-KL, Tiapkin-2023-Fast, Tiapkin-2024-Demonstration}. Finally, $\chi^2$-based weighting and regularization have emerged as tools for uncertainty-sensitive offline policy evaluation and optimization~\citep{Duchi-2019-Variance, Duan-2020-Minimax, Wang-2024-Oracle, Gabbianelli-2024-Importance}.

\section{Theoretical analysis of efficient sample complexity guarantees}

\subsection{Problem setup}
\label{app:subsec:setup}

In~\cref{sec:prelim},
we describe policies as conditional LLM policies $\pi(\y \mid \x)$, 
where prompts $\x$ are drawn from a fixed distribution $\rho$ 
and responses $\y$ belong to $\Vcal^\star$. 
We use the same convention in this section. 
To keep formulas short, we often suppress the prompt and \emph{write a response atom as $\tau$}; 
formally, we use the convention
\[
    \E_{\tau\sim\pi,\tau'\sim\pi'}[f(\tau,\tau')] \equiv \E_{\x \sim \rho,\ \y \sim \pi(\cdot \mid \x),\ \y' \sim \pi'(\cdot \mid \x)}[f(\y,\y'; \x)],
\]
where $f$ can be a general function of two response atoms $\tau := (\x,\y)$ and $\tau' := (\x,\y')$. This notation implies that the two responses are always compared under the same prompt. 

Similarly, 
\[
    P^\star(\pi \succ \pi')
    := \E_{\x \sim \rho, \ \y \sim \pi(\cdot \mid \x),\ \y' \sim \pi'(\cdot \mid \x)}
    [P^\star(\y \succ \y' \mid \x)]. 
\]
The KL divergence is the prompt-averaged conditional KL,
\[
    \KL(\pi\|\pi') := \E_{\x \sim \rho}\!\left[\KL(\pi(\cdot \mid \x) \| \pi'(\cdot \mid \x))\right].
\]
When the prompt is suppressed, 
$\A$ denotes the finite response support at a fixed prompt; equivalently, all statements below hold pointwise in $\x$ 
and are then integrated over $\x \sim \rho$.

Let $\Pi$ be a finite class of conditional policies.
Recall that $J_\beta$ is defined in \cref{sec:prelim} as 
\begin{align*}
    J_\beta(\pi_1, \pi_2) 
    =& P^\star(\pi_1\succ\pi_2) - \beta\textnormal{KL}(\pi_1\|\pi_\textnormal{ref})+\beta\textnormal{KL}(\pi_2\|\pi_\textnormal{ref}) \\ 
    =& \E_{\x \sim \rho}\left[ \E_{\y_1 \sim \pi_1(\cdot | \x), \y_2 \sim \pi_2(\cdot | \x)}[P^\star(\y_1 \succ \y_2 | \x)] \right] \\ 
    & \hspace{60pt} - \E_{\x \sim \rho}\left[ \beta \KL(\pi_1(\cdot \mid \x) \| \pref(\cdot \mid \x)) - \beta \KL(\pi_2(\cdot \mid \x) \| \pref(\cdot \mid \x)) \right].
\end{align*}
The regularized regret is defined as 
\begin{equation}
    \label{eq:regret-def}
    \begin{aligned}
        \Reg_T :=& \max_{\pi \in \Pi} \sum_{t = 1}^T \left( J_\beta(\pi, \pi_{t - 1}) - \frac{1}{2} \right) \\ 
        =& \max_{\pi \in \Pi} \sum_{t=1}^T
        \left(P^\star(\pi\succ\pi_{t-1}) - \frac12
        - \beta(\KL(\pi\|\pref) - \KL(\pi_{t-1}\|\pref)) \right).
    \end{aligned}
\end{equation}

Notice that the regularized regret above is connected to the duality gap defined in~\cref{sec:prelim}.

\begin{proposition}
    \label{prop:regret-gap}
    For any $\hat{\pi} \in \Pi$, the duality gap satisfies
    \begin{equation}\label{eq:half-gap-identity}
    \frac12\gap(\hat{\pi}) =
    \max_{\pi\in\Pi}J_\beta(\pi,\hat{\pi})-\frac12
    = \frac12-\min_{\pi\in\Pi}J_\beta(\hat{\pi},\pi).
    \end{equation}
    Let $\pi_0, \ldots, \pi_{t - 1}$ be a sequence of policies in $\Pi$. 
    Then, $\Reg_T \leq \frac{1}{2}\sum_{t = 1}^T \textnormal{gap}(\pi_{t - 1})$, and $\frac{1}{2} \textnormal{gap}(\frac{1}{T}\sum_{t = 1}^T \pi_{t - 1}) \leq \frac{1}{T} \Reg_T$.
\end{proposition}
\begin{proof}[Proof of~\cref{prop:regret-gap}]
    For any $\hat{\pi} \in \Pi$, we have 
    \begin{align*}
        \textnormal{gap}(\hat{\pi}) =& \max_{\pi \in \Pi} J_\beta(\pi, \hat{\pi}) - \min_{\pi \in \Pi} J_\beta(\hat{\pi}, \pi) \\ 
        =& \left( \max_{\pi \in \Pi} J_\beta(\pi, \hat{\pi}) - \frac{1}{2} \right) - \left( \min_{\pi \in \Pi} J_\beta(\hat{\pi}, \pi) - \frac{1}{2} \right) \\ 
        =& \left( \max_{\pi \in \Pi} J_\beta(\pi, \hat{\pi}) - \frac{1}{2} \right) + \left( \max_{\pi \in \Pi} - J_\beta(\hat{\pi}, \pi) + \frac{1}{2} \right) \\ 
        =& \left( \max_{\pi \in \Pi} J_\beta(\pi, \hat{\pi}) - \frac{1}{2} \right) + \left( \max_{\pi \in \Pi} J_\beta(\pi, \hat{\pi}) - 1 + \frac{1}{2} \right) \\ 
        =& 2 \left( \max_{\pi \in \Pi} J_\beta(\pi, \hat{\pi}) - \frac{1}{2} \right).
    \end{align*}
    This proves \cref{eq:half-gap-identity}. Then, we have 
    \begin{equation*}
        \Reg_T = \max_{\pi \in \Pi} \sum_{t = 1}^T \left( J_\beta(\pi, \pi_{t - 1}) - \frac{1}{2} \right) \leq \sum_{t = 1}^T \left( \max_{\pi \in \Pi} J_\beta(\pi, \pi_{t - 1}) - \frac{1}{2} \right) = \frac{1}{2} \sum_{t = 1}^T \textnormal{gap}(\pi_{t - 1}),
    \end{equation*}
    and 
    \begin{align*}
        \frac{1}{T} \Reg_T = \max_{\pi \in \Pi} \frac{1}{T} \sum_{t = 1}^T \left( J_\beta(\pi, \pi_{t - 1}) - \frac{1}{2} \right) 
        \geq& \max_{\pi \in \Pi} J_\beta\left(\pi, \frac{1}{T} \sum_{t = 1}^T \pi_{t - 1}\right) - \frac{1}{2} \\ 
        =& \frac12\textnormal{gap}\left(\frac{1}{T}\sum_{t = 1}^T \pi_{t - 1}\right).
    \end{align*}
    This completes the proof.
\end{proof}

We restate the antisymmetric property of the ground-truth preferences as the following assumption.
\begin{assumption}[Antisymmetric preference]
    \label{ass:anti}
    For every prompt $\x \in \VCal^\star$ and every pair of responses $\y, \y' \in \Vcal^\star$,
    \[
        P^\star(\y \succ \y' \mid \x) + P^\star(\y' \succ \y \mid \x) = 1.
    \]
    Consequently, $P^\star(\y \succ \y \mid \x) = \frac{1}{2}$ and $P^\star(\pi \succ \pi) = \frac{1}{2}$ for every policy $\pi$.
\end{assumption}

\begin{assumption}[Finite policy class]
    \label{ass:finite}
    The prompt-response support considered in the proof is finite, and the policy class $\Pi$ is finite. 
\end{assumption}

Recall that the symmetric Nash equilibrium $\pi^\star$ is defined as 
\begin{equation}
    (\pistar,\pistar) \in \argmax_{\pi_1} \argmin_{\pi_2} \ J_\beta(\pi_1,\pi_2),
    \label{def:sNE-appendix}
\end{equation}
and define the unconstrained mirror descent update based on a policy $\pi \in \Pi$ as 
\begin{equation}
    \pi^{+}_{\star, \eta}[\pi] = \underset{\tilde{\pi}}{\textnormal{argmax}} \underset{\x \sim \rho}{\E}\Big[\eta \underset{\y \sim \tilde{\pi}(\cdot|\x)}{\E}\Big[ P^*(\y \succ \pi | \x) - \beta \log\Big(\frac{\pi(\y | \x)}{\pi_\tnref(\y | \x)} \Big) \Big] - \textnormal{KL}(\tilde{\pi}(\cdot | \x) \Vert \pi(\cdot | \x)) \Big], 
    \label{eq:exact-update-appendix}
\end{equation}
which is the solution of a KL-proximal problem obtained from the exact game gradient.
Note that in the subsequent appendix we use $\pi^{+}_{\star, \eta}[\pi]$ instead of $\pi^{\textnormal{next}}_{\star, \eta}[\pi]$ for simplicity.
This exact population mirror descent update is equivalent to the conditional policy
\begin{equation}
    \label{eq:exact-update}
    \pi^+_{\star,\eta}[\pi](\y \mid \x) := \frac{\pref(\y \mid \x)^{\eta\beta} \pi(\y \mid \x)^{1 - \eta\beta} \exp(\eta P^\star(\y \succ \pi \mid \x))}
    {\sum_{\mathbf{u} \in \A}\pref(\mathbf{u} \mid \x)^{\eta\beta}\pi(\mathbf{u}\mid \x)^{1 - \eta\beta}
    \exp(\eta P^\star(\mathbf{u} \succ \pi \mid \x))}, 
\end{equation}
for any $\x, \y \in \VCal^\star$.
In the algorithmic analysis, we often use the following notation: 
at round $t$, 
\[
    \pitstar := \pi^+_{\star,\eta}[\pi_{t-1}]. 
\]

\begin{assumption}[Policy realizability]
    \label{ass:real}
    The unique Nash equilibrium $\pistar \in \Pi$. 
    Moreover, 
    $\pi^+_{\star,\eta}[\pi] \in \Pi$ for all $\pi \in \Pi$, 
    where the definitions of $\pistar$ and $\pi^+_{\star,\eta}[\pi]$ can be found in~\cref{def:sNE-appendix} and~\cref{eq:exact-update-appendix}.
\end{assumption}

\begin{assumption}[Bounded logarithmic ratio]
    \label{ass:ratio}
    For every prompt $\x \in \VCal^\star$ and every response $\y \in \VCal^\star$, 
    and all $\pi, \pi' \in \Pi$,
    \[
        \left| \log\frac{\pi(\y \mid \x)}{\pi'(\y \mid \x)} \right| \leq \frac{V}{\beta}, 
    \]
    where $V$ is a constant.
\end{assumption}
We will write $$B := \frac{V}{\beta}. $$

For any policy $\pi \in \Pi$, define
\[
    P^\star(\y \succ \pi \mid \x) := \E_{\mathbf{u} \sim \pi(\cdot \mid \x)}[P^\star(\y \succ \mathbf{u} \mid \x)].
\]

Denote 
\begin{equation*}
    h_t(\pi,\x,\y,\y')
    :=
    \log\frac{\pi(\y\mid \x)}{\pi(\y'\mid \x)}
    - \eta\beta \log\frac{\pi_\tnref(\y\mid \x)}{\pi_\tnref(\y'\mid \x)}
    - (1-\eta\beta)\log\frac{\pi_{t-1}(\y\mid \x)}{\pi_{t-1}(\y'\mid \x)}.
\end{equation*}
\begin{proposition}
    For any $\x, \y, \y' \in \VCal^\star$, and a distribution $\mu$ over $\VCal^\star$ given $\x$, 
    let 
    \begin{equation*}
        L_1 = \E_{\hat{\y} \sim \mu, I_t(\y) \sim \textnormal{Bern}(P^\star(\y \succ \hat{\y} | \x)), I_t(\y') \sim \textnormal{Bern}(P^\star(\y' \succ \hat{\y} | \x))}
        \Bigl(
        h_t(\pi,\x,\y,\y')
        - \eta\bigl(I_t(\y)-I_t(\y')\bigr)
        \Bigr)^2,
    \end{equation*}
    and 
    \begin{equation*}
        L_2 = \Bigl(
        h_t(\pi,\x,\y,\y')
        - \eta\bigl( P^\star(\y \succ \mu \mid \x) - P^\star(\y' \succ \mu \mid \x) \bigr)
        \Bigr)^2.
    \end{equation*}
    Then, $L_1$ is equivalent to $L_2$ up to a constant.
\end{proposition}
\begin{proof}
    Fix $\x,\y,\y'\in\mathcal{V}^\star$ and write
    \[
    h \coloneqq h_t(\pi,\x,\y,\y'), \qquad Z \coloneqq I_t(\y) - I_t(\y').
    \]
    Conditional on $\hat{\y}\sim\mu$, the indicators $I_t(\y)$ and $I_t(\y')$ are independent Bernoulli random variables with means $P^*(\y\succ\hat\y\mid\x)$ and $P^*(\y'\succ\hat\y\mid\x)$; thus, the tower property gives
    \begin{align*}
        \EE[Z] =& \E_{\hat{\y} \sim \mu, I_t(\y) \sim \textnormal{Bern}(P^\star(\y \succ \hat{\y} | \x)), I_t(\y') \sim \textnormal{Bern}(P^\star(\y' \succ \hat{\y} | \x))}
        \bigl(I_t(\y)-I_t(\y')\bigr) \\ 
        =& \E_{\hat{\y} \sim \mu}\left[ \E_{I_t(\y) \sim \textnormal{Bern}(P^\star(\y \succ \hat{\y} | \x)), I_t(\y') \sim \textnormal{Bern}(P^\star(\y' \succ \hat{\y} | \x))}
        \left[ I_t(\y)-I_t(\y') \mid \hat{\y} \right] \right] \\ 
        =& \EE_{\hat \y \sim \mu} \bigl[P^\star(\y \succ \hat{\y} \mid \x) - P^\star(\y' \succ \hat{\y} \mid \x) \bigr]
        = P^\star(\y \succ \mu \mid \x) - P^\star(\y' \succ \mu \mid \x).
    \end{align*}
    Applying the bias-variance identity to the square loss in $\eta Z$,
    \[
        L_1 = \EE \bigl[ (h - \eta Z)^2 \bigr]
        = \bigl(h - \eta \EE[Z] \bigr)^2 + \eta^2 \mathrm{Var}(Z)
        = L_2 + \eta^2 \mathrm{Var}(Z).
    \]
    By independence of $I_t(\y),I_t(\y')$ given $\hat\y$, the law of total variance yields
    \begin{align*}
        \mathrm{Var}(Z) &= \E_{\hat{\y} \sim \mu} [\mathrm{Var}(Z \mid \hat{\y})] + \mathrm{Var}_{\hat{\y} \sim \mu}(\E[Z \mid \hat{\y}]) \\ 
        &= \EE_{\hat{\y} \sim \mu} \left[P^\star(\y \succ \hat{\y} | \x) \bigl( 1 - P^\star(\y \succ \hat{\y} | \x) \bigr) + P^\star(\y' \succ \hat{\y} | \x) \bigl( 1 - P^\star( \y' \succ \hat{\y} | \x) \bigr) \right] \\ 
        & \hspace{100pt} + \mathrm{Var}_{\hat{\y} \sim \mu} \bigl(P^\star(\y \succ \hat{\y} \mid \x) - P^\star(\y' \succ \hat{\y} \mid \x) \bigr),
    \end{align*}
    which depends only on $(\x, \y, \y', \mu, P^\star)$ and not on $\pi$. 
    Hence, $L_1 = L_2 + C$ for a constant $C$ independent of $\pi$.
\end{proof}

\subsection{Proofs of the necessity of efficient exploration}
\label{app:subsec:necessity}

\subsubsection{Iterative Nash preference optimization is \emph{not} sufficient for efficient exploration}

Define 
\begin{equation}
    h_t(\pi, \x, \y_+, \y_-) := \log\left( \frac{\pi(\y_+ | \x)(\pi_\tnref(\y_- | \x))^{\eta\beta} (\pi^t(\y_- | \x))^{1 - \eta\beta}}{\pi(\y_- | \x)(\pi_\tnref(\y_+ | \x))^{\eta\beta} (\pi^t(\y_+ | \x))^{1 - \eta\beta}} \right).
    \label{def:inpo-ht}
\end{equation}

\begin{algorithm}[t]
    \caption{Iterative Nash Preference Optimization (INPO~\citep{Zhang-2025-Iterative})}
    \KwIn{
        Number of rounds $T$,
        per-round sample sizes $\{ n_t \}_{t = 1}^T$, 
        reference policy $\pi_\tnref$,
        stepsize $\eta > 0$}
    Initialize $\pi^0 \gets \pi_\tnref, \;
    D_{\textnormal{pref}}^t \gets \varnothing$\;
    \For{$t = 0, 1, \ldots, T$}{
        Generate $n_t$ response pairs $\{ (\x, \y, \y') \}$ via $\x \sim \rho, \y, \y' \sim \pi_t$\;
        Query preference label $(\x, \y_+, \y_-) \sim \lambda_{P^\star}(\x, \y, \y')$\;
        Update datasets $D^t_{\textnormal{pref}} \gets D^t_{\textnormal{pref}} \cup \{ (\x, \y_+, \y_-) \}$ \;
        Compute $\pi_{t + 1}$ via 
        \begin{equation*}
            \pi_{t + 1} \gets \underset{\pi \in \Pi}{\textnormal{argmin}} \left\{\sum\nolimits_{(\x, \y_+, \y_-) \in D^t_{\textnormal{pref}}} \left( h_t(\pi, \x, \y_+, \y_-) - \frac{\eta}{2} \right)^2 \right\},
        \end{equation*}
        where $h_t$ is defined in~\cref{def:inpo-ht}.
    }
    \KwOut{$\textnormal{Uniform}(\{ \pi^t \}_{t=1}^T)$}
    \label{algo:inpo}
\end{algorithm}

The following counterexample is constructed similarly to that of~\citet{Xie-2025-Exploratory}. However, our counterexample must satisfy much stricter assumptions, especially mirror descent realizability, and is therefore harder to find.
Note that we define $\pi^{+}_{\star, \eta}[\pi]$ as in \cref{eq:exact-update}.

\begin{assumption}[Relaxed policy realizability]\label{ass:relaxed-realizability}
The unique symmetric Nash policy satisfies \(\pi^\star\in\Pi\). Moreover, for the stepsize \(\eta\) used in the algorithm, there exists \(\varepsilon_\eta\ge0\) such that $\pi^+_{\star,\eta}[\pi]\in\Pi$ for every \(\pi\in\Pi\) satisfying $\KL(\pi^\star\|\pi)>\varepsilon_\eta$. 
\end{assumption}

\cref{ass:relaxed-realizability} weakens \cref{ass:real} by requiring mirror-target realizability only away from the Nash policy. If the current policy is already inside the
\(\varepsilon_\eta\)-neighborhood of \(\pi^\star\), then the policy is already a small-gap candidate, and exact realizability of its next mirror target is no longer needed for a best-iterate guarantee. 

\begin{lemma}\label{lem:near-nash-small-gap}
Suppose \cref{ass:anti,ass:finite,ass:relaxed-realizability} hold. If \(q\in\Pi\) satisfies $\KL(\pi^\star\|q)\le \varepsilon$, then
\[
    \gap(q) \le \sqrt{2}(1+4V)\sqrt{\varepsilon}.
\]
\end{lemma}
\begin{proof}
By \cref{prop:regret-gap}, 
\[
    \frac12\gap(q) = \max_{\pi\in\Pi}J_\beta(\pi,q)-\frac12.
\]
Fix any \(\pi\in\Pi\). Since \(\pi^\star\) is the symmetric Nash policy, $J_\beta(\pi,\pi^\star)\le \frac12$. Hence
\[
    J_\beta(\pi,q)-\frac12 \le J_\beta(\pi,q)-J_\beta(\pi,\pi^\star). 
\]
Let \(p_x:=\pi^\star(\cdot\mid \x)\) and \(q_x:=q(\cdot\mid \x)\). The preference part satisfies
\[
\begin{aligned}
\left|P^\star(\pi\succ q)-P^\star(\pi\succ \pi^\star)\right| &\le \E_{\x\sim\rho}[\TV(q_x,p_x)] \le \sqrt{\E_{\x\sim\rho}\TV(q_x,p_x)^2} \\
&\le \sqrt{\frac12\KL(\pi^\star\|q)} \le
\sqrt{\frac{\varepsilon}{2}},
\end{aligned}
\]
where we used Cauchy--Schwarz and Pinsker's inequality.

It remains to control the KL-regularization difference. For each prompt \(\x\),
\[
\KL(q_x\|\pi_{\rm ref}(\cdot\mid \x)) - \KL(p_x\|\pi_{\rm ref}(\cdot\mid \x)) = \KL(q_x\|p_x) + \left\langle q_x-p_x,\log\frac{p_x}{\pi_{\rm ref}(\cdot\mid \x)}\right\rangle.
\]
By \cref{ass:ratio},
\[
\left\|\log\frac{p_x}{\pi_{\rm ref}(\cdot\mid \x)}\right\|_\infty\le \frac{V}{\beta},
\qquad
\left\|\log\frac{q_x}{p_x}\right\|_\infty\le \frac{V}{\beta}.
\]
Moreover,
\[
\KL(q_x\|p_x) \le \KL(q_x\|p_x)+\KL(p_x\|q_x) = \left\langle q_x-p_x,\log\frac{q_x}{p_x}\right\rangle \le \frac{V}{\beta}\|q_x-p_x\|_1.
\]
Therefore
\[
\KL(q_x\|\pi_{\rm ref}(\cdot\mid \x)) - \KL(p_x\|\pi_{\rm ref}(\cdot\mid \x)) \le \frac{2V}{\beta}\|q_x-p_x\|_1.
\]
Averaging over \(\x\sim\rho\) and applying Pinsker again,
\[
\beta(\KL(q\|\pi_{\rm ref})-\KL(\pi^\star\|\pi_{\rm ref})) \le 2V\,\E_x\|q_x-p_x\|_1 \le 2\sqrt{2}V\sqrt{\varepsilon}.
\]
Combining the preference and KL parts gives, uniformly over \(\pi\in\Pi\),
\[
J_\beta(\pi,q)-\frac12 \le \frac{1}{\sqrt2}\sqrt{\varepsilon} + 2\sqrt2 V\sqrt{\varepsilon} = \frac{1+4V}{\sqrt2}\sqrt{\varepsilon}.
\]
Taking the maximum over \(\pi\in\Pi\) and using \cref{prop:regret-gap} proves
\[
\gap(q) \le \sqrt2(1+4V)\sqrt{\varepsilon}.
\]
\end{proof}

\begin{proof}[Proof of~\cref{prop:lower-bound-1}]
    Consider the following bandit instance with general preferences, i.e., a single fixed prompt $\x$ and three responses 
    $\A = \{ a, b, c \}$ with 
    \begin{equation*}
        P^\star(a \succ b) = 1, P^\star(b \succ c) = 1, P^\star(a \succ c) = \frac{1}{2}.
    \end{equation*}
    
    Let $\beta \leq \frac{1}{16 \log(2)}$.
    Let the reference policy be $\pi_\tnref(a) = \epsilon, \pi_\tnref(b) = 1 - 2\epsilon$ and $\pi_\tnref(c) = \epsilon$, where $\epsilon := e^{-\frac{1}{4\beta}} \leq \frac{1}{16}$.
    
    Consider a policy class 
    \begin{equation*}
        \Pi = \left\{ \pi_\tnref, (\pi_\tnref)^1_{\star, \eta}, \ldots, (\pi_\tnref)_{\star, \eta}^{\lceil (1 / \beta^2) \log{(8/\epsilon)} \rceil}, 
        \pi^\star \right\}, 
    \end{equation*}
    where $(\pi_\tnref)^{t}_{\star, \eta}$ is defined inductively by $(\pi_\tnref)^{t}_{\star, \eta} = \pi^+_{\star, \eta}[(\pi_\tnref)^{t - 1}_{\star, \eta}]$ for all $t \geq 1$, 
    and $\pi^\star$ is a Nash equilibrium of the game $\max_{\pi_1 \in \Delta(\A)} \min_{\pi_2 \in \Delta(\A)} J_\beta(\pi_1, \pi_2)$.
    
    We first show that~\cref{ass:relaxed-realizability,ass:ratio} hold for $\Pi$.
    \cref{ass:relaxed-realizability} is satisfied by the construction of $\Pi$ and the contraction of KL divergence between the mirror descent iterates and $\pi^\star$.
    In particular, we use the contraction guarantees to show that $(\pi_\tnref)_{\star, \eta}^{\lceil (1 / \beta^2) \log{(8/\epsilon)} \rceil}$ has to satisfy $\KL(\pi^\star \| (\pi_\tnref)_{\star, \eta}^{\lceil (1 / \beta^2) \log{(8/\epsilon)} \rceil}) \leq \varepsilon$.
    
    To verify~\cref{ass:ratio}, 
    we prove the following fact: if $\pi \in \Delta(\A)$ satisfies $\min_{\tau \in \A} \pi(\tau) \geq \epsilon^5$, 
    then $\pi^+_\star[\pi](\tau) \geq \epsilon^5$ for all $\tau \in \A$. 
    For any $\tau \in \A$, $\pi^+_\star[\pi]$ has a closed-form solution as shown in~\cref{eq:exact-update}.
    Denote 
    \begin{equation*}
        Z := \sum_{\tau \in \A} (\pi_\tnref(\tau))^{\eta\beta} (\pi(\tau))^{1 - \eta\beta} e^{\eta P^\star(\tau \succ \pi)}.
    \end{equation*}
    Then, we can upper bound $Z$ by 
    \begin{equation*}
         Z \leq e^\eta \sum_{\tau \in \A} (\pi_\tnref(\tau))^{\eta\beta} (\pi(\tau))^{1 - \eta\beta} \leq e^\eta \Big( \sum_{\tau \in \A} \pi_\tnref(\tau) \Big)^{\eta\beta} \Big( \sum_{\tau \in \A} \pi(\tau) \Big)^{1 - \eta\beta} = e^\eta, 
    \end{equation*}
    where the first inequality is because $P^\star(\tau \succ \pi) \leq 1$ for all $\tau \in \A$ and $\pi \in \Delta(\A)$, and the second inequality is by Hölder's inequality with $p = \frac{1}{\eta\beta}$ and $q = \frac{1}{1 - \eta\beta}$.
    Thus,
    \begin{equation*}
        \pi^{+}_\star(\tau) 
        = \frac{1}{Z} (\pi_\tnref(\tau))^{\eta\beta} (\pi(\tau))^{1 - \eta\beta} e^{\eta P^\star(\tau \succ \pi)} 
        \geq e^{-\eta} \epsilon^{\eta\beta} \left( \epsilon^5 \right)^{1 - \eta\beta} 
        = (\epsilon^{-4\beta})^{-\eta} \epsilon^{\eta\beta} \epsilon^{5(1 - \eta\beta)} = \epsilon^5, 
    \end{equation*}
    where the first inequality is because $Z \leq e^\eta$, $\pi_\tnref(\tau) \geq \epsilon, \pi(\tau) \geq \epsilon^5$ for all $\tau \in \A$, 
    and $e^{\eta P^\star(\cdot \succ \cdot)} \geq 1$.
    Since $\min_{\tau \in \A} \pi_\tnref(\tau) \geq \epsilon \geq \epsilon^5$, 
    we have $\min_{\tau \in \A} \pi(\tau) \geq \epsilon^5$ for all $\pi \in \Pi \setminus \{ \pi^\star \}$.
    It remains to show $\min_{\tau \in \A^*} \pi^\star(\tau) \geq \epsilon^5$.
    Note that the Nash equilibrium can be characterized by the following equality~\citep[Eq. (12)]{calandriello2024human}:
    \begin{equation}
        \pi^\star(\tau)
        =
        \frac{\pi_{\rm ref}(\tau)
        \exp\left(P^\star(\tau \succ \pi^\star)/\beta\right)}
        {\sum_{\tau' \in \mathcal Y}
        \pi_{\rm ref}(\tau')
        \exp\left(P^\star(\tau' \succ \pi^\star)/\beta\right)}.
        \label{eq:NE-characterization}
    \end{equation}
    Since $\exp\left(P^\star(\tau \succ \pi^\star)/\beta\right)\ge 1$, and
    $\sum_{\tau' \in \mathcal Y}
    \pi_{\rm ref}(\tau')
    \exp\left(P^\star(\tau' \succ \pi^\star)/\beta\right)
    \le e^{1/\beta}$, 
    we have 
    \[
        \pi^\star(\tau)
        \geq
        \pi_{\rm ref}(\tau)e^{-1/\beta}
        \geq
        \epsilon e^{-1/\beta} 
        = 
        e^{-\frac{1}{4\beta}} e^{-\frac{1}{\beta}} 
        = 
        e^{-\frac{5}{4\beta}} = \epsilon^5.
    \]
    Thus, we have $\min_{\tau \in \A} \pi(\tau) \geq \epsilon^5$ for all $\pi \in \Pi$.
    Clearly, we have $\lvert \log( \frac{\pi(\y)}{\pi'(\y)} ) \rvert \leq \log({1}/{\epsilon^5}) = \frac{5}{4\beta}$ for all $\pi, \pi' \in \Pi$. Hence, \cref{ass:ratio} holds for $\Pi$ with $V = 5/4$.

    We now consider applying~\cref{algo:inpo} to learn a Nash policy.
    For each round $t$, 
    define the event $\mathcal{E}_t := \{\text{actions $a, c$ are not sampled at round $t$}\}$.
    We establish the following lemma and present its proof after this proof.

    \begin{lemma}\label{lem:conditional-probability-dpbound}
        For each round $t$, 
        define the event $\mathcal{E}_t := \{\text{actions $a, c$ are not sampled at round $t$}\}$.
        Then, whenever $\mathcal{E}_0, \mathcal{E}_1, \ldots, \mathcal{E}_{t - 1}$ happen, 
        we can choose the next policy $\pi^{t} = \pi_\tnref$, and therefore
        \begin{itemize}
            \item $\mathbb{P}(\mathcal{E}_t \,\vert\, \mathcal{E}_0, \mathcal{E}_1, \ldots, \mathcal{E}_{t - 1}) \geq 1 - 4 n_t \epsilon$, and 
            \item the duality gap of $\pi^t$ is lower bounded by a constant: 
            $\textnormal{gap}_\Pi(\pi^{t}) \geq c$, where $c$ is the KL divergence between $\pi_\tnref$ and $\pi^\star$.
        \end{itemize}
    \end{lemma}
        
    \rev{ For any $\delta\in(0,1)$, suppose that $4n(T+1)\epsilon\le\delta$. By \cref{lem:conditional-probability-dpbound}, conditional on $E_0,\ldots,E_{t-1}$, $\mathbb P(E_t\mid E_0,\ldots,E_{t-1}) \ge 1-4n\epsilon$. 
    Therefore, by iterated conditioning and Bernoulli's inequality,
    \begin{align}
        \mathbb P(E_0\cap E_1\cap\cdots\cap E_T) &= \mathbb P(E_0)\prod_{t=0}^{T-1}\mathbb P(E_{t+1}\mid E_0,\ldots,E_t) \nonumber \\
        &\ge (1-4n\epsilon)^{T+1} \ge 1-4n(T+1)\epsilon \ge 1-\delta. 
        \label{eq:joint-probability}
    \end{align}
    Since \(\epsilon=e^{-1/(4\beta)}\), the assumed condition is equivalently $4n(T+1)\le\delta e^{1/(4\beta)}$. }
    Combining~\cref{eq:joint-probability} with the second point of~\cref{lem:conditional-probability-dpbound} completes the proof.
\end{proof}

\begin{proof}[Proof of~\cref{lem:conditional-probability-dpbound}]
    If $\mathcal{E}_0, \mathcal{E}_1, \ldots, \mathcal{E}_{t - 1}$ happen, the online dataset $D^{t - 1}_{\textnormal{pref}}$ only contains $b$.
    In this case, the loss function in~\cref{algo:inpo} is a constant that is independent of $\pi$, 
    and thus $\pi_\tnref$ is a valid minimizer of the loss function.
    Then,
    \begin{equation*}
        \mathbb{P}(\mathcal{E}_t \,\vert\, \mathcal{E}_0, \mathcal{E}_1, \ldots, \mathcal{E}_{t - 1}) = (1 - \pi^t(a) - \pi^t(c))^{2n_t} = (1 - 2\epsilon)^{2n_t} \geq 1 - 4n_t\epsilon,
    \end{equation*}
    where the last inequality is because $(1 + x)^a \geq 1 + a x$ for any $x \geq -1$ and $a \geq 1$.

    For the second claim, since $\pi^\star\in\Pi$,~\cref{prop:regret-gap} gives
    \[
    \frac12\mathrm{gap}_\Pi(\pi_{\mathrm{ref}})
    =\max_{\pi_1\in\Pi}J_\beta(\pi_1,\pi_{\mathrm{ref}})-\frac12
    \ge J_\beta(\pi^\star,\pi_{\mathrm{ref}})-\frac12 .
    \]
    Set $f(\pi_2):=J_\beta(\pi^\star,\pi_2)$. By the Nash property $\pi^\star\in\arg\min_{\pi_2}f(\pi_2)$
    and $f(\pi^\star)=\frac12$. Writing the preference term as the linear form
    $P^\star(\pi^\star\succ\pi_2)=\langle g,\pi_2\rangle$ with $g(y'):=P^\star(\pi^\star\succ y'\mid x)$,
    and using the exact three-point identity
    \[
    \mathrm{KL}(\pi_2\|\pi_{\mathrm{ref}})-\mathrm{KL}(\pi^\star\|\pi_{\mathrm{ref}})
    =\mathrm{KL}(\pi_2\|\pi^\star)+\bigl\langle\log(\pi^\star/\pi_{\mathrm{ref}}),\pi_2-\pi^\star\bigr\rangle,
    \]
    we get
    \[
    f(\pi_2)-f(\pi^\star)
    =\bigl\langle g+\beta\log(\pi^\star/\pi_{\mathrm{ref}}),\pi_2-\pi^\star\bigr\rangle
    +\beta\mathrm{KL}(\pi_2\|\pi^\star).
    \]
    First-order optimality of $\pi^\star$ over $\Delta(\mathcal{Y})$ makes the inner product nonnegative
    (the constant component of $\nabla f(\pi^\star)$ annihilates $\pi_2-\pi^\star$, which has zero sum),
    so $f(\pi_2)-f(\pi^\star)\ge\beta\mathrm{KL}(\pi_2\|\pi^\star)$. Taking $\pi_2=\pi_{\mathrm{ref}}$,
    \[
    \frac12\mathrm{gap}_\Pi(\pi_{\mathrm{ref}})
    \ge f(\pi_{\mathrm{ref}})-f(\pi^\star)\ge\beta\mathrm{KL}(\pi_{\mathrm{ref}}\|\pi^\star).
    \]
    Since $\pi^\star$ has full support and $\pi_{\mathrm{ref}}\neq\pi^\star$, the constant
    $c:=2\beta\mathrm{KL}(\pi_{\mathrm{ref}}\|\pi^\star)>0$ is independent of $t$ and $T$, giving
    $\mathrm{gap}_\Pi(\pi^t)\ge c$.
\end{proof}

\subsection{Proofs of the efficient exploration guarantee for ENPO}
\label{app:sec:efficient-exploration-guarantee}

\subsubsection{Algorithm}

We restate Exploratory Nash Preference Optimization (ENPO) in~\cref{alg:adv-exp-NashMD-appendix}.
\begin{algorithm}[t]
    \caption{Exploratory Nash Preference Optimization (ENPO)}
    \label{alg:adv-exp-NashMD-appendix}
    \KwIn{
        horizon $T$, 
        schedules \((m_t, n_t, k_t)_{t=1}^T\), 
        coefficients $\gamma, \lambda>0$, 
        stepsize \( \eta>0 \), 
        reference policy \(\pi_{\rm ref}\), 
        pessimistic parameter $\alpha_t$}
    \textbf{Initialize} $\pi_{-1} \gets \pi_\tnref$, $\pi_0 \gets \pi_\tnref$, $\hat{\pi}_0 \gets \pi_\tnref$\;
    \textbf{Initialize} $\mathcal{D}_{1}^0 \gets \varnothing, \mathcal{D}_{2}^0 \gets \varnothing$\;

    \For{$t = 1, \ldots, T$}{
        Generate $n_t$ response pairs $(\tau_{t - 1,i}, \hat{\tau}_{t - 1,i})$ via $\tau^{t - 1, i} \sim \pi_{t - 2}, \hat{\tau}^{t - 1, i} \sim \hat{\pi}_{t - 1}$\;
        Generate $m_t$ responses $\tilde{\tau}_{t - 1, i}$ via $\tilde{\tau}_{t - 1, i} \sim \pi_{t - 1}$\;
        Update preference datasets $D^t_{1} \gets \{\tilde{\tau}_{t - 1, i} \}_{i=1}^{m_t}$, $D^t_{2} \gets D^{t-1}_{2} \cup \{ (\tau_{t - 1, i}, \hat{\tau}_{t - 1, i}) \}_{i=1}^{n_t}$\; 
        Define $\A_t^{\rm act} := \{\tau,\tau':(\tau,\tau')\in D_2^t\}$. Generate \(c_{t,1},\dots,c_{t,k_t}\sim\pi_{t-1}\) independently. For every \(a\in \A_t^{\rm act}\), query $Y_{t,a,j}\sim \mathrm{Bern}(P^*(a\succ c_{t,j}))$ for $j=1,\dots,k_t$, and set $\widehat P_t(a\succ\pi_{t-1}) = \frac1{k_t}\sum_{j=1}^{k_t}Y_{t,a,j}$\; 
        Compute 
        \begin{equation}
            \label{eq:learner-update-appendix}
            \pi_{t} \gets \underset{\pi \in \Pi}{\textnormal{argmin}} \left\{ - \alpha_t \sum\nolimits_{\tau \in D^{t}_{1}} \log{\pi(\tau)} + L_{t-1}(\pi) \right\}, 
        \end{equation}
        where $L_{t-1}$ is defined in~\cref{eq:learner-loss-app}\;
        Compute 
        \begin{equation}
            \label{eq:explorer-update-appendix}
            \hat{\pi}_{t} \gets \underset{\tilde{\pi} \in \Pi}{\textnormal{argmax}} \max_{\pi \in \Pi} \frac{\big( \frac{1}{m_t}\sum_{i=1}^{m_t} \log{\frac{\pi_t(\tilde{\tau}_{t - 1, i})}{\pi(\tilde{\tau}_{t - 1, i})}} - \EE_{\tau' \sim \tilde{\pi}}\big[ \log{\frac{\pi_t(\tau')}{\pi(\tau')}} \big] \big)^2}{\lambda + 
            \sum_{s=0}^{t-1}\frac{1}{n_{s+1}}\sum_{i=1}^{n_{s+1}}
            \big(
            \log \frac{\pi_t(\tau_{s,i})}{\pi(\tau_{s,i})}
            - \log \frac{\pi_t(\tau'_{s,i})}{\pi(\tau'_{s,i})}
            \big)^2}. 
        \end{equation}
    }
    \KwOut{$\textnormal{Uniform}(\{ \pi^t \}_{t=0}^{T-1})$ or $\pi_{\textnormal{best}} \in \{ \pi^t \}_{t=0}^{T-1}$ with the best duality gap.}
\end{algorithm}

\subsubsection{Additional notation}

For a prompt-response atom $\tau = (\x, \y)$, 
write $\pi(\tau)$ as shorthand for $\pi(\y \mid \x)$ inside likelihood ratios. 
For any $\pi, \pi' \in \Pi$, define the population residual
\begin{equation}
    \label{eq:delta-pop}
    \delta(\tau;\pi,\pi'):=
    \log\frac{\pi(\tau)}{\pi'(\tau)}
    -\eta\beta\log\frac{\pref(\tau)}{\pi'(\tau)}
    -\eta P^\star(\tau\succ\pi'),
\end{equation}
where $P^\star(\tau \succ \pi')$ means $P^\star(\y \succ \pi' \mid \x)$ when $\tau=(\x, \y)$. 
For a general $\pi\in\Pi$, define
\[
    \Delta_t(\tau,\tau';\pi):=\delta(\tau;\pi,\pi_{t-1})-\delta(\tau';\pi,\pi_{t-1}).
\]
At round $t$, write
\[
    \delta_t(\tau):=\delta(\tau;\pi_t,\pi_{t-1}),
    \qquad
    \Delta_t(\tau,\tau') := \Delta_t(\tau, \tau'; \pi_t) = \delta_t(\tau)-\delta_t(\tau'),
\]
where $\tau=(\x,\y)$ and $\tau'=(\x,\y')$ share the same prompt. 
We define the empirical residuals at round $t$ by substituting $P^\star$ with the empirical preferences $\widehat{P}_t$ at round $t$: 
\begin{equation}
    \label{eq:delta-hat}
    \widehat\delta_t(\tau;\pi):=
    \log\frac{\pi(\tau)}{\pi_{t-1}(\tau)}
    -\eta\beta\log\frac{\pref(\tau)}{\pi_{t-1}(\tau)}
    -\eta\widehat P_t(\tau\succ\pi_{t-1}),
\end{equation}
and
\[
    \widehat\Delta_t(\tau,\tau';\pi):=\widehat\delta_t(\tau;\pi)-\widehat\delta_t(\tau';\pi).
\]
The empirical learner loss in \cref{alg:adv-exp-NashMD-appendix} is
\begin{equation}\label{eq:learner-loss-app}
    L_{t-1}(\pi)
    :=\sum_{s=0}^{t-1}\frac{1}{n_{s+1}}
      \sum_{i=1}^{n_{s+1}}
        \widehat\Delta_t(\tau_{s,i},\tau'_{s,i};\pi)^2.
\end{equation}

At round $t$, \cref{alg:adv-exp-NashMD-appendix} only estimates preferences for the active prompt-response atoms appearing in the historical comparison dataset,
\[
    \mathcal Y_t^{\rm act}:=\{(\x_{s,i},\y_{s,i}),(\x_{s,i},\y'_{s,i}):0\le s<t,
    \ 1\le i\le n_{s+1}\}.
\]

For an active atom $a=(\x,\y)\in\mathcal Y_t^{\rm act}$, the proof version of \cref{alg:adv-exp-NashMD-appendix} uses prompt-conditional opponents: draw $c_{t,a,j}\sim\pi_{t-1}(\cdot\mid \x)$ and query $Y_{t,a,j}\sim\mathrm{Bern}(P^\star(\y\succ c_{t,a,j}\mid \x))$, for $j=1,\ldots,k_t$. Reusing a common set of opponents is a practical variant, but the theorem below is proved for the prompt-conditional version stated in the algorithm. All preference-estimation concentration below is imposed only on $\mathcal Y_t^{\rm act}$, not on the entire response space. 

The SEC class is
\[
\Psi:=\left\{\psi_{\rho',\pi}(\tau,\tau')=
\log\frac{\rho'(\tau)}{\pi(\tau)}-
\log\frac{\rho'(\tau')}{\pi(\tau')}:\rho',\pi\in\Pi\right\}.
\]
For a fixed floor $\lambda>0$, define
\begin{equation}
    \label{eq:sec-def}
    \SEC(\Psi,\lambda)
    := \sup_{\psi_1, \ldots, \psi_T \in \Psi} \sup_{\substack{\pi_0, \ldots, \pi_{T-1} \in \Pi \\ \widehat\pi_1, \ldots, \widehat\pi_T \in \Pi}}
    \sum_{t=1}^T \frac{\left( \E_{\tau \sim \pi_{t-1}, \tau' \sim \widehat{\pi}_t} \psi_t(\tau, \tau') \right)^2}{\lambda\vee \sum_{s=1}^{t-1} \E_{\tau \sim \pi_{s-1}, \tau' \sim \widehat{\pi}_s} \psi_t(\tau, \tau')^2}.
\end{equation}

\subsubsection{Proof of~\texorpdfstring{\cref{thm:relaxed-realizability-best}}{Theorem~\ref*{thm:relaxed-realizability-best}}}
\label{app:subsubsec:main-proof}

\paragraph{Proof strategy.}
The proof of the regret bound consists of a deterministic regret bound \cref{thm:versionA} and a high-probability concentration step \cref{lem:explicit-events}. The deterministic part shows that accurate empirical estimates imply small regret; the concentration part verifies that these accuracy events hold with high probability for the stated sample sizes. 

To facilitate the proof, we first restate the main theorem with complete notation.
\begin{theorem}
    \label{thm:concrete}
    Suppose \cref{ass:anti,ass:ratio,ass:finite,ass:real} hold. Let $M_T:=T(T+1)(T+2)/3$ and 
    \[
    L_T:=\max\{1,L_\Pi,L_P,\log(eT)\}, \text{ where } L_\Pi:=\log\left(16T|\Pi|^2\delta^{-1}\right), \; L_P:=\log\left(16M_T\delta^{-1}\right). 
    \]
    Set the exploration coefficient $\alpha_t=2/(\gamma m_t)$, the sample counts $m_t=n_t=k_t=t$, the SEC parameters $\overline H:=1+B^2+\beta^{-2}$, $\lambda:=16\overline H L_T^2$, $S_T\ge 1\vee\SEC(\Psi,\lambda)$, $\gamma:=\sqrt{S_T/(\lambda T)}$, and the constant stepsize $\eta:=2\beta/(1+4\beta^2)$. Then, with probability at least $1-\delta$, 
    \begin{equation}\label{eq:concrete-bound}
        \Reg_T \le \frac{1+2e^{\eta(2V+1)}}{\eta} \left(B+20\sqrt{S_T\lambda T}\right) \leq 3e^{3V}\eta^{-1} \left(B+20\sqrt{S_T\lambda T}\right). 
    \end{equation}
    Thus, \cref{alg:adv-exp-NashMD-appendix} achieves an \(\widetilde O(\sqrt{T})\) regret bound. 
\end{theorem}

We first prove a deterministic version of the regret bound. It assumes the empirical quantities used by the learner and explorer are accurate, and postpones concentration to the end of \cref{subsubsec:main-lemmas}. 
\begin{theorem}\label{thm:versionA}
Suppose \cref{ass:anti,ass:ratio,ass:finite,ass:real} hold, and choose a constant stepsize satisfying $0<\eta\le 4\beta/(1+4\beta^2)$. Let $B:=V/\beta$, $c_\eta:=1-\eta\beta-\eta/(4\beta)$ and $R_\eta:=\eta(2V+1)$. Assume $c_\eta>0$. On any event on which \cref{prop:E-bound} holds for both $\pi=\widetilde\pi_T$ and $\pi=\pistar$, where
\[
    \widetilde\pi_T\in\argmax_{\pi\in\Pi}\sum_{t=1}^T \left( P^\star(\pi\succ\pi_{t-1})-\frac12
    -\beta(\KL(\pi\|\pref)-\KL(\pi_{t-1}\|\pref)) \right),
\]
with $\ell_t$ defined in \cref{eq:sharp-step}, we have
\begin{align}\label{eq:master-regret}
\Reg_T
&\le
\frac{1+e^{R_\eta}/c_\eta}{\eta}
\left[
B+\frac{16}{\gamma}\SEC(\Psi,\lambda)
+\frac{3\gamma\lambda T}{8}
+\sum_{t=1}^T\epsilon_{m,t}
+\frac{18}{\gamma\lambda}\sum_{t=1}^T\epsilon_{m,t}^2
+4\gamma\eta^2\sum_{t=1}^T t\epsilon_{p,t}^2
\right].
\end{align}
\end{theorem}
\begin{proof}[Proof of~\cref{thm:versionA}]
Apply \cref{lem:sharp-step} with comparator $\widetilde\pi_T$ and sum over $t$:
\begin{align*}
\eta\Reg_T
&\le
\sum_{t=1}^T\left(\KL(\widetilde\pi_T\|\pi_{t-1})-\KL(\widetilde\pi_T\|\pitstar)\right)
+\sum_{t=1}^T\KL(\pi_{t-1}\|\pitstar)\\
&=
\KL(\widetilde\pi_T\|\pi_0)-\KL(\widetilde\pi_T\|\pi_T)
+E_T(\widetilde\pi_T)
+\sum_{t=1}^T\KL(\pi_{t-1}\|\pitstar).
\end{align*}
Since $\pi_0=\pref$ and $\widetilde\pi_T\in\Pi$, \cref{ass:ratio} gives $\KL(\widetilde\pi_T\|\pi_0)\le B$. Also $\KL(\widetilde\pi_T\|\pi_T)\ge0$. Therefore
\[
    \eta\Reg_T\le B+E_T(\widetilde\pi_T)+\sum_{t=1}^T\KL(\pi_{t-1}\|\pitstar).
\]
By \cref{lem:movement-control},
\[
    \eta\Reg_T\le
    B+E_T(\widetilde\pi_T)+\frac{e^{R_\eta}}{c_\eta}\left(\KL(\pistar\|\pi_0)+E_T(\pistar)\right).
\]
Again $\KL(\pistar\|\pi_0)\le B$. Apply \cref{prop:E-bound} to $\pi=\widetilde\pi_T$ and $\pi=\pistar$ to obtain \cref{eq:master-regret}.
\end{proof}

It remains to instantiate the deterministic bound with the concentration events proved in \cref{subsubsec:main-lemmas} with specific choices of parameters.

\begin{proof}[Proof of~\cref{thm:concrete}]
\cref{lem:explicit-events} shows that the following three concentration events hold with high probability, i.e., for all $\rho',\pi,\nu\in \Pi$, 
\begin{align*}
    &|N_t^{\rho',\pi}(\nu)-\widehat N_t^{\rho',\pi}(\nu)|\le \epsilon_{m,t},
    \quad
    \epsilon_{m,t}:=B\sqrt{\frac{2L_\Pi}{m_t}}; \\
    & \frac12(\lambda + H_t(\pi)) \leq \lambda + \widehat H_t(\pi) \leq 2(\lambda + H_t(\pi)); \\
    &|\widehat P_t(a\succ\pi_{t-1})-P^\star(a\succ\pi_{t-1})|
    \le \epsilon_{p,t}, \quad \epsilon_{p,t}^2:=\frac{L_P}{2k_t},
\end{align*}
where $N_t^{\rho',\pi}(\nu),\widehat N_t^{\rho',\pi}(\nu),\widehat H_t(\pi),H_t(\pi)$ are defined in \cref{def:N-H,def:N-H-compact}. 

The first event controls the numerator of the explorer objective: the empirical extrapolation shift is close to its population value. The second one controls the denominator: the empirical coverage energy is within a constant factor of the population coverage energy after adding \(\lambda\). The third one controls the preference estimates used in the learner's residual loss on the historical comparison data set. Together, these events allow replacing the population quantities in the deterministic regret bound by the empirical quantities computed by \cref{alg:adv-exp-NashMD-appendix}. 

On the above three events, \cref{prop:E-bound} is valid for every $\pi\in\Pi$, and it allows us to invoke the deterministic regret bound \cref{thm:versionA}, 
\begin{equation*}
    \Reg_T \le
    \frac{1+e^{R_\eta}/c_\eta}{\eta}
    \left[
    B+\frac{16}{\gamma}S_T
    +\frac{3\gamma\lambda T}{8}
    +\sum_{t=1}^T\epsilon_{m,t}
    +\frac{18}{\gamma\lambda}\sum_{t=1}^T\epsilon_{m,t}^2
    +4\gamma\eta^2\sum_{t=1}^T t\epsilon_{p,t}^2
    \right],
\end{equation*}
where we use $S_T=1\vee \SEC(\Psi,\lambda)$. Since \cref{eq:ceta-explicit} gives $c_\eta=1/2$ and $R_\eta=\eta(2V+1)$, the prefactor $(1+e^{R_\eta}/c_\eta)/\eta$ becomes $(1+2e^{\eta(2V+1)})/\eta$. 

It remains to upper bound the bracket in \cref{eq:master-regret}. Let $A_T:=\sqrt{S_T\lambda T}$. We bound each term in the bracket by a constant multiple of $A_T$. Since $\gamma=\sqrt{S_T/(\lambda T)}$, we have 
\begin{equation}\label{eq:sec-explicit-eval}
    \frac{16}{\gamma}S_T+\frac{3\gamma\lambda T}{8}
    =\left(16+\frac38\right)A_T.
\end{equation}
Since we take $m_t=t$, 
\[
    \epsilon_{m,t}=B\sqrt{\frac{2L_\Pi}{t}}
    \le \sqrt{\frac{2\overline H L_T}{t}},
\]
where we used $B^2\le\overline H$ and $L_\Pi\le L_T$. Hence, using $\sum_{t=1}^Tt^{-1/2}\le2\sqrt T$ and $\sum_{t=1}^Tt^{-1}\le\log(eT)\le L_T$, we have 
\begin{equation}\label{eq:epsm-sums}
    \sum_{t=1}^T\epsilon_{m,t}
    \le 2\sqrt{2\overline H L_T T},
    \qquad
    \sum_{t=1}^T\epsilon_{m,t}^2
    \le 2\overline H L_T^2.
\end{equation}
Since $S_T\ge1$, $L_T\ge1$, and $\lambda=16\overline H L_T^2$, the linear $\epsilon_m$ term satisfies
\begin{equation}\label{eq:epsm-linear-eval}
    \sum_{t=1}^T\epsilon_{m,t}
    \le A_T.
\end{equation}
For the quadratic $\epsilon_m$ term,
\begin{equation}\label{eq:epsm-square-eval}
    \frac{18}{\gamma\lambda}\sum_{t=1}^T\epsilon_{m,t}^2
    \le
    \frac{36\overline H L_T^2}{\gamma\lambda}
    =\frac{36\overline H L_T^2}{S_T\lambda}A_T
    \le \frac94 A_T,
\end{equation}
where the last inequality uses $S_T\ge1$ and $\lambda=16\overline H L_T^2$. Next, $k_t=t$ and $L_P\le L_T$, so
\begin{equation}\label{eq:epsp-eval}
    \sum_{t=1}^T t\epsilon_{p,t}^2
    =\sum_{t=1}^T \frac{tL_P}{2k_t}
    \le \frac{TL_T}{2}.
\end{equation}
Since $\eta\le1$, $\overline H\ge1$, and $L_T\ge1$,
\begin{equation}\label{eq:pref-eval}
    4\gamma\eta^2\sum_{t=1}^T t\epsilon_{p,t}^2
    \le 2\gamma\eta^2L_TT
    =\frac{2\eta^2L_T}{\lambda}A_T
    \le \frac18 A_T.
\end{equation}
Combining \cref{eq:sec-explicit-eval}, \cref{eq:epsm-linear-eval}, \cref{eq:epsm-square-eval}, and \cref{eq:pref-eval}, the non-$B$ part of the bracket in \cref{eq:master-regret} is at most
\[
    \left(16+\frac38+1+\frac94+\frac18\right)A_T
    =19.75A_T
    \le 20A_T.
\]
The desired inequality \cref{eq:concrete-bound} then follows by noticing that $V\geq 1$ and $\eta\leq 1$. 
\end{proof}

\cref{thm:concrete} is stated under exact mirror-target realizability. We next consider a relaxed version: if realizability only holds away from a small neighborhood of \(\pi^\star\), the same proof yields a best-candidate guarantee with an additional approximation term. 

Define the best candidate 
\[
    \pi^{\rm best} \in \argmin_{\pi\in\{\pi_0,\pi_1,\ldots,\pi_{T-1},\bar\pi_T\}}\gap(\pi), 
\]
and the best-candidate regret 
\[
    \Reg_T^{\rm best} := T\left(\max_{\nu\in\Pi}J_\beta(\nu,\pi^{\rm best})-\frac12 \right) = \frac{T}{2}\gap(\pi^{\rm best}),
\]
where the equality follows from \cref{prop:regret-gap}.

\begin{corollary}\label{cor:relaxed-realizability-best}
Suppose \cref{ass:anti,ass:finite,ass:ratio,ass:relaxed-realizability} hold. Run \cref{alg:adv-exp-NashMD-appendix} for \(T\) rounds with the same choices of parameters as in \cref{thm:concrete}. Then, with probability at least \(1-\delta\), 
\[
    \Reg_T^{\rm best} \le \max\left\{
        \frac{1+2e^{\eta(2V+1)}}{\eta}
        \left( B+20\sqrt{S_T\lambda T} \right),
        \frac{T}{\sqrt2}(1+4V)\sqrt{\varepsilon_\eta} \right\}.
\]
If $\varepsilon_\eta=O(1/T)$, then $\Reg_T^{\rm best}=\widetilde{O}(\sqrt{T})$. 
\end{corollary}
\begin{proof}
Let $\mathcal C_T:=\{\pi_0,\pi_1,\ldots,\pi_{T-1},\bar\pi_T\}$. Define the first entrance time 
\[
    \tau := \inf\left\{ s\in\{0,\ldots,T-1\}:
    \KL(\pi^\star\|\pi_s)\le\varepsilon_\eta
    \right\}.
\]

First suppose \(\tau=\infty\). Then for every update round \(t=1,\ldots,T\), $\KL(\pi^\star\|\pi_{t-1})>\varepsilon_\eta$. By relaxed realizability, $\pi^+_{\star,\eta}[\pi_{t-1}]\in\Pi$ for every realized round. Inspecting the proof of \cref{thm:concrete}, \cref{ass:real} is used only to ensure that these realized exact mirror targets belong to \(\Pi\). Therefore the proof of \cref{thm:concrete} applies on the full horizon and gives the same trajectory regret \(\Reg_T\). 

Since \(\bar\pi_T\in\mathcal C_T\), by \cref{prop:regret-gap}, we have 
\[
    \Reg_T^{\rm best} = \frac{T}{2}\gap(\pi^{\rm best}) \le \frac{T}{2}\gap(\bar\pi_T) \le \Reg_T. 
\]

Now suppose \(\tau<\infty\). Then $\KL(\pi^\star\|\pi_\tau)\le\varepsilon_\eta$. By \cref{lem:near-nash-small-gap},
\[
    \gap(\pi_\tau) \le \sqrt2(1+4V)\sqrt{\varepsilon_\eta}.
\]
Since \(\pi_\tau\in\mathcal C_T\), the best candidate satisfies
\[
    \Reg_T^{\rm best} = \frac{T}{2}\gap(\pi^{\rm best}) \le
    \frac{T}{2}\gap(\pi_\tau) \le \frac{T}{\sqrt2}(1+4V)\sqrt{\varepsilon_\eta}.
\]

Combining the two cases gives
\[
    \Reg_T^{\rm best} \le \max\left\{\Reg_T, \frac{T}{\sqrt2}(1+4V)\sqrt{\varepsilon_\eta}\right\}.
\]
Substituting the bound on \(\Reg_T\) from \cref{thm:concrete} proves the displayed high-probability bound. 
\end{proof}

Finally, we can invoke a result in \cite{Xie-2025-Exploratory} to upper bound $\SEC(\Psi,\lambda)$ by coverage and some logarithmic factor in $T$. 

\begin{corollary}
For \(\pi,\widehat\pi\in\Pi\), define the same-prompt pair distribution
\[
\nu_{\pi,\widehat\pi}(\x,\y,\y')
:=
\rho(\x)\pi(\y|\x)\widehat\pi(\y'|\x).
\]
Define 
\[
C_{\rm cov}^{\rm pair}(\Pi)
:=
\inf_{\mu}
\sup_{\pi,\widehat\pi\in\Pi}
\left\|
\frac{\nu_{\pi,\widehat\pi}}{\mu}
\right\|_\infty.
\]
Under the same assumptions as in \cref{thm:concrete}
\[
\SEC(\Psi,\lambda)
\le
c\,C_{\rm cov}^{\rm pair}(\Pi)\log(eT),
\]
where \(c>0\) is a universal constant.
\end{corollary}
\begin{proof}
    This is a direct adaptation of \cite{Xie-2025-Exploratory}.
\end{proof}

\subsubsection{Auxiliary Propositions and Lemmas for~\texorpdfstring{\cref{thm:concrete}}{Theorem~\ref*{thm:concrete}}}\label{subsubsec:main-lemmas}

We now collect the auxiliary facts used in \cref{app:subsubsec:main-proof}. The first part concerns the population mirror update, the second part controls extrapolation and local empirical error, and the last part verifies the required concentration events. 

The following result shows that the exact mirror target has a simple residual structure: its population residual is constant across responses for the same prompt. This identity allows us to decompose KL error into a local term and an extrapolation term. 

\begin{proposition}\label{prop:residual-identity}
For every $t\ge 1$ and every prompt $\x$, the function $\y\mapsto \delta((\x,\y);\pitstar,\pi_{t-1})$ is constant over responses $\y$. More precisely,
\[
    \delta((\x,\y);\pitstar,\pi_{t-1})=-\log Z_t(\x),
\]
where $Z_t(\x)$ is the denominator in \cref{eq:exact-update} with $\pi=\pi_{t-1}$. Consequently, for any two prompt-response atoms $\tau=(\x,\y)$ and $\tau'=(\x,\y')$ sharing the same prompt,
\[
    \Delta_t(\tau,\tau')
    =\log\frac{\pi_t(\tau)}{\pitstar(\tau)}
     -\log\frac{\pi_t(\tau')}{\pitstar(\tau')}.
\]
Moreover, for any $\pi\in\Pi$,
\begin{equation}\label{eq:regret-decomp-q}
\KL(\pi\|\pi_t)-\KL(\pi\|\pitstar) =\E_{\tau\sim\pi_{t-1}}\left[\log\pitstar(\tau)-\log\pi_t(\tau)\right] + \E_{\tau\sim\pi_{t-1},\tau'\sim\pi}\left[\Delta_t(\tau,\tau')\right].
\end{equation}
\end{proposition}

\begin{proof}
Fix a prompt $\x$. By \cref{eq:exact-update},
\[
\log\pitstar(\y\mid \x)=\eta\beta\log\pref(\y\mid \x)+(1-\eta\beta)\log\pi_{t-1}(\y\mid \x) + \eta P^\star(\y\succ\pi_{t-1}\mid \x)-\log Z_t(\x).
\]
Substituting this display into \cref{eq:delta-pop} gives
\[
\delta((\x,\y);\pitstar,\pi_{t-1})=-\log Z_t(\x),
\]
which is independent of $\y$ for the fixed prompt $\x$. Therefore 
\[
\delta_t(\x,\y)-\delta((\x,\y);\pitstar,\pi_{t-1})=
\log\frac{\pi_t(\y\mid \x)}{\pitstar(\y\mid \x)}. 
\]
Taking differences between $\y$ and $\y'$ proves the first identity. For \cref{eq:regret-decomp-q}, note that all policies have the same prompt marginal $\rho$, so the conditional constant $-\log Z_t(\x)$ has the same expectation under $\pi$ and under $\pi_{t-1}$. Thus
\[
\KL(\pi\|\pi_t)-\KL(\pi\|\pitstar)=\E_{\x\sim\rho,\y'\sim \pi(\cdot\mid \x)}\!\left[\log\frac{\pitstar(\y'\mid \x)}{\pi_t(\y'\mid \x)}\right]
\]
can be rewritten by adding and subtracting the expectation under $\y\sim\pi_{t-1}(\cdot\mid \x)$, yielding \cref{eq:regret-decomp-q}.
\end{proof}

The next lemma is the one-step mirror-descent inequality. It converts regularized gain against any comparator into a KL telescoping term plus the movement of the exact mirror target. 

\begin{lemma}\label{lem:sharp-step}
For every $\pi\in\Pi$ and every $t\ge 1$, define
\[
    \ell_t(\pi) := P^\star(\pi\succ\pi_{t-1})-\frac12
    -\beta(\KL(\pi\|\pref)-\KL(\pi_{t-1}\|\pref)). 
\]
Then
\begin{equation}\label{eq:sharp-step}
    \eta\ell_t(\pi)
    \le
    \KL(\pi\|\pi_{t-1})-\KL(\pi\|\pitstar)+\KL(\pi_{t-1}\|\pitstar).
\end{equation}
\end{lemma}
\begin{proof}
Let
\[
    g_t(\tau):=P^\star(\tau\succ\pi_{t-1})-\beta\log\frac{\pi_{t-1}(\tau)}{\pref(\tau)}.
\]
The exact update satisfies
\[
    \pitstar=\argmin_{u\in\Delta(\A)}\{\langle -\eta g_t,u\rangle+\KL(u\|\pi_{t-1})\}.
\]
The Bregman three-point inequality for the KL prox step gives, for all $\pi$,
\begin{equation}\label{eq:three-point-sharp}
    \eta\langle g_t,\pi-\pitstar\rangle
    \le
    \KL(\pi\|\pi_{t-1})-\KL(\pi\|\pitstar)-\KL(\pitstar\|\pi_{t-1}).
\end{equation}
Add $\eta\langle g_t,\pitstar-\pi_{t-1}\rangle$ to both sides.  The closed form gives
\[
    \eta g_t(\tau)=\log\frac{\pitstar(\tau)}{\pi_{t-1}(\tau)}+\log Z_t(x), 
\]
so the constant cancels when integrated against $\pitstar-\pi_{t-1}$.  Hence
\begin{align*}
\eta\langle g_t,\pitstar-\pi_{t-1}\rangle
&=\sum_{\tau}(\pitstar(\tau)-\pi_{t-1}(\tau))
\log\frac{\pitstar(\tau)}{\pi_{t-1}(\tau)} = \KL(\pitstar\|\pi_{t-1})+\KL(\pi_{t-1}\|\pitstar).
\end{align*}
Combining with \cref{eq:three-point-sharp},
\[
\eta\langle g_t,\pi-\pi_{t-1}\rangle
\le \KL(\pi\|\pi_{t-1})-\KL(\pi\|\pitstar)+\KL(\pi_{t-1}\|\pitstar).
\]
Finally,
\begin{align*}
\langle g_t,\pi-\pi_{t-1}\rangle &=P^\star(\pi\succ\pi_{t-1})-\frac12 -\beta(\KL(\pi\|\pref)-\KL(\pi_{t-1}\|\pref)) +\beta\KL(\pi\|\pi_{t-1})\\
&\ge \ell_t(\pi),
\end{align*}
which proves the claim.
\end{proof}

We next use the KL regularization to show that the exact population mirror map contracts toward
the Nash policy. This contraction is the key reason a constant stepsize is possible. 
\begin{lemma}[Contraction of mirror descent with constant stepsize]\label{lem:exact-contraction}
    \leavevmode\par
    Let $\pistar$ be the unique symmetric Nash equilibrium of the KL-regularized game.  Let $\pi^+ := \pi^+_{*,\eta}(\pi)$ be defined by \cref{eq:exact-update}.  If
    \begin{equation}\label{eq:eta-condition}
        0< \eta \leq \frac{4\beta}{1+4\beta^2} \leq 1,
    \end{equation}
    then, with $c_\eta:=1-\eta\beta-\eta/(4\beta)\ge0$,
    \begin{equation}\label{eq:exact-contraction}
        \KL(\pistar\|\pi^+) \le (1-\eta\beta)\KL(\pistar\|\pi)-c_\eta\KL(\pi^+\|\pi).
    \end{equation}
    In particular, $\KL(\pistar\|\pi^+)\le(1-\eta\beta)\KL(\pistar\|\pi)$.
\end{lemma}
\begin{proof}
The proof is pointwise in the prompt $\x$; we suppress $\x$ and write distributions over the finite response set $\A$. Integrating the resulting inequality over $\x\sim\rho$ gives the stated prompt-averaged KL inequality.
Let $h(p)=\sum_{\tau}p(\tau)\log p(\tau)$ be negative entropy.  Define the operator
\[
    F(p)=A(p)+B(p),\qquad
    A(p)(\tau):=-P^\star(\tau\succ p),
    \qquad
    B(p):=\beta(\nabla h(p)-\nabla h(\pref)).
\]
The Nash equilibrium $\pistar$ satisfies the variational inequality
\begin{equation}\label{eq:vi-star}
    \langle F(\pistar),u-\pistar\rangle\ge0\qquad\forall u\in\Delta(\A).
\end{equation}
The exact update is equivalently
\begin{equation}\label{eq:prox-vi}
    \pi^+=\argmin_{u\in\Delta(\A)}\left\{\eta\langle F(\pi),u\rangle+\KL(u\|\pi)\right\}.
\end{equation}
Indeed, expanding $F$ gives exactly \cref{eq:exact-update} up to constants independent of $u$.

We first record three structural facts.  Let $P$ be the matrix with entries $P_{\tau,u}=P^\star(\tau\succ u)$. \cref{ass:anti} gives $P+P^\top=\1\1^\top$.  Since both $u-v$ and $\pi-v$ have zero sum for distributions, for all $u,v\in\Delta(\A)$,
\[
    \langle A(u)-A(v),u-v\rangle=-(u-v)^\top P(u-v)
    =-\frac12(u-v)^\top(P+P^\top)(u-v)=0.
\]
Thus $A$ is monotone with equality.  Also,
\[
\|A(u)-A(v)\|_\infty
=\max_\tau\left|\sum_a\left(P_{\tau,a}-\frac12\right)(v(a)-u(a))\right|
\le \frac12\|u-v\|_1.
\]
Finally, the entropy identity gives
\begin{equation}\label{eq:B-strong}
\langle B(u)-B(v),u-v\rangle
=\beta\big(\KL(u\|v)+\KL(v\|u)\big).
\end{equation}

Using
\[
\eta\langle B(\pi),u\rangle
=\eta\beta\KL(u\|\pref)-\eta\beta\KL(u\|\pi)+\text{constant in }u,
\]
we can write the prox objective as
\[
\Psi(u)=\eta\langle A(\pi),u\rangle+
\eta\beta\KL(u\|\pref)+(1-\eta\beta)\KL(u\|\pi),
\]
where $\eta\beta\le1$ follows from \cref{eq:eta-condition}. The two KL weights sum to one, hence $\Psi$ is 1-strongly convex with respect to the entropy Bregman divergence. Since $\pi^+$ minimizes $\Psi$, the Bregman three-point inequality for a $1$-strongly convex function with respect to negative entropy gives
\begin{equation}\label{eq:ineq-i}
\Psi(\pistar)-\Psi(\pi^+)\ge \KL(\pistar\|\pi^+).
\end{equation}
On the other hand, \cref{eq:vi-star} and the strong convexity of
$u\mapsto \langle A(\pistar),u\rangle+\beta\KL(u\|\pref)$ imply
\begin{equation}\label{eq:ineq-ii}
\eta\beta\big(\KL(\pi^+\|\pref)-\KL(\pistar\|\pref)\big)
\ge
\eta\langle A(\pistar),\pistar-\pi^+\rangle+
\eta\beta\KL(\pi^+\|\pistar).
\end{equation}
Adding \cref{eq:ineq-i} and \cref{eq:ineq-ii}, the $\KL(\cdot\|\pref)$ terms cancel and we obtain
\begin{align}\label{eq:key-contraction-pre}
\KL(\pistar\|\pi^+)
&\le
\eta\langle A(\pi)-A(\pistar),\pistar-\pi^+\rangle
+(1-\eta\beta)\KL(\pistar\|\pi)\nonumber\\
&\quad -(1-\eta\beta)\KL(\pi^+\|\pi)-\eta\beta\KL(\pi^+\|\pistar).
\end{align}
Insert and subtract $A(\pi^+)$ in the inner product.  By monotonicity with equality,
\[
\langle A(\pi)-A(\pistar),\pistar-\pi^+\rangle
=\langle A(\pi^+)-A(\pi),\pi^+-\pistar\rangle.
\]
Using the Lipschitz property of $A$, H\"older's inequality, and Pinsker's inequality,
\begin{align*}
\eta\left|\langle A(\pi^+)-A(\pi),\pi^+-\pistar\rangle\right|
&\le \frac{\eta}{2}\|\pi^+-\pi\|_1\|\pi^+-\pistar\|_1\\
&\le \eta\sqrt{\KL(\pi^+\|\pi)\KL(\pi^+\|\pistar)}\\
&\le \frac{\eta}{4\beta}\KL(\pi^+\|\pi)+\eta\beta\KL(\pi^+\|\pistar),
\end{align*}
where the last step applies the AM--GM inequality. Substituting this into \cref{eq:key-contraction-pre} cancels the $\eta\beta\KL(\pi^+\|\pistar)$ term and yields
\[
\KL(\pistar\|\pi^+)
\le (1-\eta\beta)\KL(\pistar\|\pi)
-\left(1-\eta\beta-\frac{\eta}{4\beta}\right)\KL(\pi^+\|\pi).
\]
This proves \cref{eq:exact-contraction}.
\end{proof}

The previous contraction controls one KL direction. The next lemma compares the reverse direction using the fact that one exact mirror step changes log-probabilities by a bounded amount. 
\begin{lemma}\label{lem:reverse-movement}
For every $t\ge1$,
\[
    \KL(\pi_{t-1}\|\pitstar) \le e^{R_\eta}\KL(\pitstar\|\pi_{t-1}),
    \qquad R_\eta:=\eta(2V+1).
\]
\end{lemma}
\begin{proof}
By the closed form \cref{eq:exact-update},
\[
\log\frac{\pi_t^\star(\y\mid \x)}{\pi_{t-1}(\y\mid \x)} = a_t(\x,\y)-\log Z_t(\x),
\qquad
Z_t(\x)=\sum_u\pi_{t-1}(u\mid \x)e^{a_t(\x,u)}.
\]
where 
\[
a_t(\x,\y):=\eta\beta\log\frac{\pi_{\rm ref}(\y\mid \x)}{\pi_{t-1}(\y\mid \x)}
+\eta P^\star(\y\succ\pi_{t-1}\mid \x),
\]
The range of $\eta\beta\log(\pref/\pi_{t-1})$ over $\tau$ is at most $2\eta V$ by \cref{ass:ratio}, and the range of $\eta P^\star(\tau\succ\pi_{t-1})$ is at most $\eta$. Since \(Z_t(\x)\) is a \(\pi_{t-1}(\cdot\mid \x)\)-weighted average of \(e^{a_t(\x,u)}\), \(\log Z_t(\x)\) lies between
\(\min_u a_t(\x,u)\) and \(\max_u a_t(\x,u)\). Therefore
\[
\left|\log\frac{\pi_t^\star(\y\mid \x)}{\pi_{t-1}(\y\mid \x)}\right| \le \max_u a_t(\x,u)-\min_u a_t(\x,u) \le \eta(2V+1)=:R_\eta.
\]
Let $p=\pitstar$, $\bar\pi=\pi_{t-1}$, and $r=p/\bar\pi$.  Then $r\in[e^{-R_\eta},e^{R_\eta}]$ pointwise. The $f$-divergence representation gives
\[
\KL(\bar\pi\|p)=\E_{\bar\pi}[r-1-\log r], \qquad \KL(p\|\bar\pi)=\E_{\bar\pi}[r\log r-r+1].
\]
For $r\in[e^{-R_\eta},e^{R_\eta}]$, the functions $f(r)=r-1-\log r$ and $g(r)=r\log r-r+1$ satisfy $f(1)=g(1)=f'(1)=g'(1)=0$ and
\[
    f''(r)=\frac1{r^2}\le e^{R_\eta}\frac1r=e^{R_\eta}g''(r).
\]
Integrating twice around $r=1$ gives $f(r)\le e^{R_\eta}g(r)$ for all such $r$. Taking expectation under $\bar\pi$ proves the claim.
\end{proof}

For $\pi\in\Pi$, define the pairwise log-ratio direction
\[
    \psi_{t,\pi}(\tau,\tau'):=\log\frac{\pi_t(\tau)}{\pi(\tau)}-
    \log\frac{\pi_t(\tau')}{\pi(\tau')}.
\]
For concentration statements we also need a version with an arbitrary numerator policy $\rho'\in\Pi$:
\[
    \psi_{\rho',\pi}(\tau,\tau'):=\log\frac{\rho'(\tau)}{\pi(\tau)}-
    \log\frac{\rho'(\tau')}{\pi(\tau')}.
\]

We now introduce the quantities optimized by the adversarial explorer. The numerator measures extrapolation to a candidate opponent, while the denominator measures how well past exploratory data cover the corresponding log-ratio direction. 

For $\rho',\pi,\nu\in\Pi$, define
\begin{equation}
    \begin{aligned}
        N_t^{\rho',\pi}(\nu) &:=\E_{\tau\sim\pi_{t-1},\tau'\sim\nu}\psi_{\rho',\pi}(\tau,\tau'), \\
        \widehat N_t^{\rho',\pi}(\nu)&:=\frac1{m_t}\sum_{i=1}^{m_t}\log\frac{\rho'(\widetilde\tau_{t-1,i})}{\pi(\widetilde\tau_{t-1,i})}
        -\E_{\tau'\sim\nu}\log\frac{\rho'(\tau')}{\pi(\tau')}, \\
        H_t^{\rho',\pi} &:=\sum_{s=0}^{t-1}\E_{\tau\sim\pi_{s-1},\tau'\sim\widehat\pi_s}\psi_{\rho',\pi}(\tau,\tau')^2, \\
        \widehat H_t^{\rho',\pi} &:=\sum_{s=0}^{t-1}\frac1{n_{s+1}}\sum_{i=1}^{n_{s+1}}
        \psi_{\rho',\pi}(\tau_{s,i},\tau'_{s,i})^2.
    \end{aligned}
    \label{def:N-H}
\end{equation}
Here $N_t^{\rho',\pi}(\nu)$ denotes the population extrapolation shift, measuring how much the log-ratio direction changes when moving from $\pi_{t-1}$ to candidate opponent $\nu$, and $\widehat N_t^{\rho',\pi}(\nu)$ is its empirical version; $H_t^{\rho',\pi}$ denotes the population coverage energy, accumulating the squared size of this direction under past exploration pairs. 

We can also define their round-$t$ version 
\begin{equation}
    N_t(\nu,\pi):=N_t^{\pi_t,\pi}(\nu), \widehat N_t(\nu,\pi):=\widehat N_t^{\pi_t,\pi}(\nu), H_t(\pi):=H_t^{\pi_t,\pi}, \textnormal{ and } \widehat H_t(\pi):=\widehat H_t^{\pi_t,\pi}.
    \label{def:N-H-compact}
\end{equation}

The following lemma bounds the extrapolation error by the SEC term, up to empirical estimation errors and a quadratic remainder. 

\begin{lemma}
    \label{lem:sec-quarter}
    Assume that for every $t\le T$, every $\pi\in\Pi$, and every $\nu\in\Pi$,
    \begin{equation}\label{eq:N-H-events}
        |N_t(\nu,\pi) - \widehat N_t(\nu,\pi)| \leq \epsilon_{m,t},
        \qquad
        \frac12(\lambda + H_t(\pi)) \leq \lambda + \widehat H_t(\pi) \leq 2(\lambda + H_t(\pi)), 
    \end{equation}
    where $N_t(\nu,\pi)$, $\widehat N_t(\nu,\pi)$, $H_t(\pi)$, $\widehat H_t(\pi)$ are defined in~\cref{def:N-H,def:N-H-compact}.
    Then for every $\pi\in\Pi$,
    \begin{align}\label{eq:sec-quarter-bound}
    \sum_{t=1}^T\E_{\tau\sim\pi_{t-1},\tau'\sim \pi}\Delta_t(\tau,\tau')
    & \leq
    \frac{16}{\gamma}\SEC(\Psi,\lambda)
    +\frac{\gamma\lambda T}{4}
    +\frac{18}{\gamma\lambda}\sum_{t=1}^T\epsilon_{m,t}^2\nonumber\\
    & \hspace{80pt}
    +\frac{\gamma}{4}\sum_{t=1}^T\sum_{s=0}^{t-1}
    \E_{\tau\sim\pi_{s-1},\tau'\sim\widehat\pi_s}\Delta_t(\tau,\tau')^2.
    \end{align}
\end{lemma}
\begin{proof}
By \cref{prop:residual-identity}, $\Delta_t(\tau,\tau')=\psi_{t,\pitstar}(\tau,\tau')$, and this function belongs to $\Psi$ by \cref{ass:real}. Define $D_t:=\lambda+H_t(\pitstar)$. For any real $x$ and $a>0$, $x\le x^2/(\gamma a)+\gamma a/4$.  Therefore,
\begin{equation}\label{eq:young-quarter}
\E_{\pi_{t-1},\pi}\Delta_t\le
\frac{N_t(\pi,\pitstar)^2}{\gamma D_t} +\frac{\gamma\lambda}{4} +\frac{\gamma}{4}H_t(\pitstar).
\end{equation}
Using \cref{eq:N-H-events},
\[
\frac{N_t(\pi,\pitstar)^2}{\gamma D_t}
\le
\frac{2\widehat N_t(\pi,\pitstar)^2}{\gamma D_t}
+\frac{2\epsilon_{m,t}^2}{\gamma\lambda}
\le
\frac{4\widehat N_t(\pi,\pitstar)^2}{\gamma(\lambda+\widehat H_t(\pitstar))}
+\frac{2\epsilon_{m,t}^2}{\gamma\lambda}.
\]
By the definition of $\widehat\pi_t$ in \cref{alg:adv-exp-NashMD-appendix},
\[
\frac{\widehat N_t(\pi,\pitstar)^2}{\lambda+\widehat H_t(\pitstar)}
\le
\max_{r\in\Pi}\frac{\widehat N_t(\widehat\pi_t,r)^2}{\lambda+\widehat H_t(r)}.
\]
Again using \cref{eq:N-H-events}, for every $r\in\Pi$,
\[
\frac{\widehat N_t(\widehat\pi_t,r)^2}{\lambda+\widehat H_t(r)}
\le 4\frac{N_t(\widehat\pi_t,r)^2}{\lambda+H_t(r)}+\frac{4\epsilon_{m,t}^2}{\lambda}.
\]
Combining the last three displays yields
\[
\frac{N_t(\pi,\pitstar)^2}{\gamma D_t}
\le \frac{16}{\gamma} \max_{r\in\Pi}\frac{N_t(\widehat\pi_t,r)^2}{\lambda+H_t(r)} +\frac{18\epsilon_{m,t}^2}{\gamma\lambda}.
\]
Substitute this into \cref{eq:young-quarter} and sum over $t$. For each $r\in\Pi$, $H_t(r)$ contains the SEC denominator sum $\sum_{s=1}^{t-1}\E_{\pi_{s-1},\widehat\pi_s}[\psi_{t,r}^2]$ and possibly the additional nonnegative $s=0$ term. Hence
\[
\lambda+H_t(r)\ge
\lambda\vee \sum_{s=1}^{t-1}\E_{\tau\sim\pi_{s-1},\tau'\sim\widehat\pi_s}\psi_{t,r}(\tau,\tau')^2.
\]
By the SEC definition \cref{eq:sec-def}, the sum of the max-ratio terms is bounded by $\SEC(\Psi,\lambda)$. This proves \cref{eq:sec-quarter-bound}. 
\end{proof}

The preceding lemma controls the extrapolation part but leaves a quadratic remainder. The next lemma controls the local, on-sample part of the update and produces a negative copy of the same quadratic term. This is where the learner's empirical risk minimization step is used: because \(\pi_t\) minimizes the empirical residual loss, it cannot be much worse than the exact mirror target \(\pi_t^\star\) on the sampled data. The remaining error comes only from finite-sample estimation of preferences.

\begin{lemma}\label{lem:local-sharp}
Fix $t\ge1$.  Suppose the following events hold:
\begin{align}
\E_{\tau\sim\pi_{t-1}}[\log\pitstar(\tau)-\log\pi_t(\tau)]
&\le
\frac1{m_t}\sum_{i=1}^{m_t}[\log\pitstar(\widetilde\tau_{t-1,i})-\log\pi_t(\widetilde\tau_{t-1,i})]+\epsilon_{m,t},\label{eq:local-event-m} \\
\sum_{s=0}^{t-1}\E_{\pi_{s-1},\widehat\pi_s}\Delta_t(\tau,\tau')^2
&\le
\sum_{s=0}^{t-1}\frac1{n_{s+1}}\sum_{i=1}^{n_{s+1}}\Delta_t(\tau_{s,i},\tau'_{s,i})^2+\frac{\lambda}{2},\label{eq:local-event-H}
\end{align}
and for every empirical historical pair, with
\[
    e_t(a):=\widehat P_t(a\succ\pi_{t-1})-P^\star(a\succ\pi_{t-1}),
\]
we have
\begin{equation}\label{eq:local-event-p}
    |e_t(a)|\le \epsilon_{p,t}
    \qquad
    \text{for all }a\in\{\tau_{s,i},\tau'_{s,i}:0\le s<t,\ 1\le i\le n_{s+1}\}.
\end{equation}
Then, under \cref{eq:learner-update-appendix},
\begin{equation}\label{eq:local-sharp-bound}
\E_{\tau\sim\pi_{t-1}}[\log\pitstar(\tau)-\log\pi_t(\tau)] + \frac{\gamma}{4}\sum_{s=0}^{t-1}\E_{\pi_{s-1},\widehat\pi_s}\Delta_t(\tau,\tau')^2 \le
\epsilon_{m,t}+\frac{\gamma\lambda}{8}+4\gamma\eta^2 t\epsilon_{p,t}^2.
\end{equation}
\end{lemma}
\begin{proof}
For brevity, define the empirical averages
\[
    \widehat S_t(\pi):=\sum_{s=0}^{t-1}\frac1{n_{s+1}}\sum_{i=1}^{n_{s+1}}
    \widehat\Delta_t(\tau_{s,i},\tau'_{s,i};\pi)^2,
\]
\[
    S_t(\pi):=\sum_{s=0}^{t-1}\frac1{n_{s+1}}\sum_{i=1}^{n_{s+1}}
    \Delta_t(\tau_{s,i},\tau'_{s,i};\pi)^2.
\]
From \cref{eq:local-event-m} and \cref{eq:local-event-H},
\begin{align}\label{eq:start-local}
&\E_{\pi_{t-1}}[\log\pitstar-\log\pi_t]
+\frac{\gamma}{4}\sum_{s=0}^{t-1}\E_{\pi_{s-1},\widehat\pi_s}\Delta_t^2
\nonumber\\
&\qquad\le
\frac1{m_t}\sum_{i=1}^{m_t}[\log\pitstar(\widetilde\tau_{t-1,i})-\log\pi_t(\widetilde\tau_{t-1,i})]
+\frac{\gamma}{4}S_t(\pi_t)+\epsilon_{m,t}+\frac{\gamma\lambda}{8}.
\end{align}
The ERM optimality of $\pi_t$ against the feasible comparator $\pitstar$ gives
\begin{align}\label{eq:erm-opt}
&\frac1{m_t}\sum_{i=1}^{m_t}[\log\pitstar(\widetilde\tau_{t-1,i})-\log\pi_t(\widetilde\tau_{t-1,i})]
+\frac{\gamma}{2}\widehat S_t(\pi_t)
\le
\frac{\gamma}{2}\widehat S_t(\pitstar).
\end{align}
This is exactly where $\alpha_t=2/(\gamma m_t)$ is used.
For an empirical pair $(a,b)$, write
\[
    u:=\Delta_t(a,b;\pi_t),
    \qquad
    z:=\eta(e_t(a)-e_t(b)).
\]
Then $\widehat\Delta_t(a,b;\pi_t)=u-z$.  Also, by \cref{prop:residual-identity},
$\Delta_t(a,b;\pitstar)=0$, and hence $\widehat\Delta_t(a,b;\pitstar)=-z$.
Combining \cref{eq:erm-opt} with the term $\frac{\gamma}{4}S_t(\pi_t)$ in \cref{eq:start-local}, pair by pair we get the upper bound
\[
    \frac{\gamma}{2}z^2+\frac{\gamma}{4}u^2-\frac{\gamma}{2}(u-z)^2
    =\gamma uz-\frac{\gamma}{4}u^2
    \le \gamma z^2,
\]
where the last inequality follows from $uz-u^2/4\le z^2$. Therefore
\begin{align*}
&\frac1{m_t}\sum_{i=1}^{m_t}[\log\pitstar(\widetilde\tau_{t-1,i})-\log\pi_t(\widetilde\tau_{t-1,i})]
+\frac{\gamma}{4}S_t(\pi_t)\le
\gamma\eta^2\sum_{s=0}^{t-1}\frac1{n_{s+1}}\sum_{i=1}^{n_{s+1}}
\big(e_t(\tau_{s,i})-e_t(\tau'_{s,i})\big)^2.
\end{align*}
By \cref{eq:local-event-p}, each squared difference is at most $4\epsilon_{p,t}^2$. Since
\[
\sum_{s=0}^{t-1}\frac1{n_{s+1}}\sum_{i=1}^{n_{s+1}}1=t,
\]
the last display is bounded by $4\gamma\eta^2t\epsilon_{p,t}^2$. Substituting into \cref{eq:start-local} proves \cref{eq:local-sharp-bound}.
\end{proof}

Combining the extrapolation bound with the local bound cancels the shared quadratic remainder. This gives a uniform bound on the accumulated inexact-update error \(E_T(\pi)\).

\begin{proposition}\label{prop:E-bound}
Assume the events \cref{eq:N-H-events,eq:local-event-m,eq:local-event-H,eq:local-event-p} hold for every $t\le T$.  Then for every fixed comparator $\pi\in\Pi$,
\begin{align}\label{eq:E-bound}
E_T(\pi)&:=\sum_{t=1}^T\big(\KL(\pi\|\pi_t)-\KL(\pi\|\pitstar)\big)\nonumber\\
&\le
\frac{16}{\gamma}\SEC(\Psi,\lambda)
+\frac{3\gamma\lambda T}{8}
+\sum_{t=1}^T\epsilon_{m,t}
+\frac{18}{\gamma\lambda}\sum_{t=1}^T\epsilon_{m,t}^2
+4\gamma\eta^2\sum_{t=1}^T t\epsilon_{p,t}^2.
\end{align}
\end{proposition}
\begin{proof}
Fix a comparator \(\pi\in\Pi\). Throughout this proof, for each \(t\), write
\[
    \Delta_t(\tau,\tau')
    :=
    \delta_t(\tau)-\delta_t(\tau').
\]
Also define the accumulated population quadratic term
\[
    Q_t
    :=
    \sum_{s=0}^{t-1}
    \E_{\tau\sim\pi_{s-1},\,\tau'\sim\widehat\pi_s}
    \left[
        \Delta_t(\tau,\tau')^2
    \right],
\]
and
\[
    Q_T^{\rm tot}
    :=
    \sum_{t=1}^T Q_t
    =
    \sum_{t=1}^T
    \sum_{s=0}^{t-1}
    \E_{\tau\sim\pi_{s-1},\,\tau'\sim\widehat\pi_s}
    \left[
        \Delta_t(\tau,\tau')^2
    \right].
\]

We start from the definition of \(E_T(\pi)\). For every \(t\),
\[
    \KL(\pi\|\pi_t)-\KL(\pi\|\pitstar)
    =
    \E_{\tau\sim\pi}
    \left[
        \log\frac{\pitstar(\tau)}{\pi_t(\tau)}
    \right].
\]
Therefore,
\[
    E_T(\pi)
    =
    \sum_{t=1}^T
    \E_{\tau\sim\pi}
    \left[
        \log\pitstar(\tau)-\log\pi_t(\tau)
    \right].
\]

Next we use the residual decomposition from \cref{eq:regret-decomp-q}. For completeness, we spell out the identity. By the residual identity,
\[
    \log\pitstar(\tau)-\log\pi_t(\tau)
    =
    \delta_t^\star(\tau)-\delta_t(\tau),
\]
and by \cref{prop:residual-identity}, \(\delta_t^\star(\cdot)\) is constant in \(\tau\). Hence
\begin{align*}
    \E_{\tau\sim\pi}
    \left[
        \log\pitstar(\tau)-\log\pi_t(\tau)
    \right]
    &=
    \E_{\tau\sim\pi}
    \left[
        \delta_t^\star(\tau)-\delta_t(\tau)
    \right]
    \\
    &=
    \E_{\tau\sim\pi_{t-1}}
    \left[
        \delta_t^\star(\tau)-\delta_t(\tau)
    \right]
    +
    \E_{\tau\sim\pi_{t-1}}\left[\delta_t(\tau)\right]
    -
    \E_{\tau'\sim\pi}\left[\delta_t(\tau')\right]
    \\
    &=
    \E_{\tau\sim\pi_{t-1}}
    \left[
        \log\pitstar(\tau)-\log\pi_t(\tau)
    \right]
    +
    \E_{\tau\sim\pi_{t-1},\,\tau'\sim\pi}
    \left[
        \delta_t(\tau)-\delta_t(\tau')
    \right]
    \\
    &=
    \E_{\tau\sim\pi_{t-1}}
    \left[
        \log\pitstar(\tau)-\log\pi_t(\tau)
    \right]
    +
    \E_{\tau\sim\pi_{t-1},\,\tau'\sim\pi}
    \left[
        \Delta_t(\tau,\tau')
    \right].
\end{align*}
Summing this equality over \(t=1,\dots,T\) gives
\begin{align}\label{eq:E-decomp-expanded}
    E_T(\pi)
    =
    \underbrace{
    \sum_{t=1}^T
    \E_{\tau\sim\pi_{t-1}}
    \left[
        \log\pitstar(\tau)-\log\pi_t(\tau)
    \right]
    }_{=:L_T}
    +
    \underbrace{
    \sum_{t=1}^T
    \E_{\tau\sim\pi_{t-1},\,\tau'\sim\pi}
    \left[
        \Delta_t(\tau,\tau')
    \right]
    }_{=:X_T}.
\end{align}

We now bound the extrapolation part \(X_T\). Since the event \cref{eq:N-H-events} holds for every \(t\le T\), the hypotheses of \cref{lem:sec-quarter} are satisfied. Applying \cref{lem:sec-quarter} with the fixed comparator \(\pi\) gives
\begin{align}\label{eq:extrap-bound-expanded}
    X_T
    &=
    \sum_{t=1}^T
    \E_{\tau\sim\pi_{t-1},\,\tau'\sim\pi}
    \left[
        \Delta_t(\tau,\tau')
    \right]
    \nonumber\\
    &\le
    \frac{16}{\gamma}\SEC(\Psi,\lambda)
    +
    \frac{\gamma\lambda T}{4}
    +
    \frac{18}{\gamma\lambda}
    \sum_{t=1}^T\epsilon_{m,t}^2
    +
    \frac{\gamma}{4}
    Q_T^{\rm tot}.
\end{align}

We next bound the local part \(L_T\). Since the events \cref{eq:local-event-m,eq:local-event-H,eq:local-event-p} hold for every \(t\le T\), the hypotheses of \cref{lem:local-sharp} are satisfied at every round \(t\). Thus, for each \(t\),
\begin{align}\label{eq:local-sharp-expanded}
    \E_{\tau\sim\pi_{t-1}}
    \left[
        \log\pitstar(\tau)-\log\pi_t(\tau)
    \right]
    +
    \frac{\gamma}{4} Q_t
    \le
    \epsilon_{m,t}
    +
    \frac{\gamma\lambda}{8}
    +
    4\gamma\eta^2 t\epsilon_{p,t}^2.
\end{align}
Rearranging \cref{eq:local-sharp-expanded} gives
\[
    \E_{\tau\sim\pi_{t-1}}
    \left[
        \log\pitstar(\tau)-\log\pi_t(\tau)
    \right]
    \le
    \epsilon_{m,t}
    +
    \frac{\gamma\lambda}{8}
    +
    4\gamma\eta^2 t\epsilon_{p,t}^2
    -
    \frac{\gamma}{4}Q_t.
\]
Summing this inequality over \(t=1,\dots,T\), we obtain
\begin{align}\label{eq:local-bound-expanded}
    L_T
    &=
    \sum_{t=1}^T
    \E_{\tau\sim\pi_{t-1}}
    \left[
        \log\pitstar(\tau)-\log\pi_t(\tau)
    \right]
    \nonumber\\
    &\le
    \sum_{t=1}^T\epsilon_{m,t}
    +
    \frac{\gamma\lambda T}{8}
    +
    4\gamma\eta^2\sum_{t=1}^T t\epsilon_{p,t}^2
    -
    \frac{\gamma}{4}
    \sum_{t=1}^T Q_t
    \nonumber\\
    &=
    \sum_{t=1}^T\epsilon_{m,t}
    +
    \frac{\gamma\lambda T}{8}
    +
    4\gamma\eta^2\sum_{t=1}^T t\epsilon_{p,t}^2
    -
    \frac{\gamma}{4}Q_T^{\rm tot}.
\end{align}

Finally, substitute \cref{eq:extrap-bound-expanded} and \cref{eq:local-bound-expanded}
into the decomposition \cref{eq:E-decomp-expanded}. We get 
\begin{align*}
    E_T(\pi)
    &=
    L_T+X_T
    \\
    &\le
    \left[
    \sum_{t=1}^T\epsilon_{m,t}
    +
    \frac{\gamma\lambda T}{8}
    +
    4\gamma\eta^2\sum_{t=1}^T t\epsilon_{p,t}^2
    -
    \frac{\gamma}{4}Q_T^{\rm tot}
    \right]
    \\
    &\quad
    +
    \left[
    \frac{16}{\gamma}\SEC(\Psi,\lambda)
    +
    \frac{\gamma\lambda T}{4}
    +
    \frac{18}{\gamma\lambda}
    \sum_{t=1}^T\epsilon_{m,t}^2
    +
    \frac{\gamma}{4}Q_T^{\rm tot}
    \right].
\end{align*}
Therefore,
\[
E_T(\pi)
\le
\frac{16}{\gamma}\SEC(\Psi,\lambda)
+\frac{3\gamma\lambda T}{8}
+\sum_{t=1}^T\epsilon_{m,t}
+\frac{18}{\gamma\lambda}\sum_{t=1}^T\epsilon_{m,t}^2
+4\gamma\eta^2\sum_{t=1}^T t\epsilon_{p,t}^2,
\]
which is exactly \cref{eq:E-bound}.
\end{proof}

The deterministic regret bound still contains the total movement of the exact mirror targets. The next lemma controls this movement using the contraction of the population mirror map.

\begin{lemma}\label{lem:movement-control}
Let
\[
    E_T(\pistar):=\sum_{t=1}^T\left(\KL(\pistar\|\pi_t)-\KL(\pistar\|\pitstar)\right).
\]
Under \cref{eq:eta-condition},
\begin{equation}\label{eq:movement-control}
    \sum_{t=1}^T\KL(\pi_{t-1}\|\pitstar)
    \le
    \frac{e^{R_\eta}}{c_\eta}
    \left(\KL(\pistar\|\pi_0)+E_T(\pistar)\right),
\end{equation}
where $c_\eta=1-\eta\beta-\eta/(4\beta)$ and $R_\eta=\eta(2V+1)$.
\end{lemma}
\begin{proof}
By \cref{lem:exact-contraction} applied with $\pi=\pi_{t-1}$ and $\pi^+=\pitstar$,
\[
    c_\eta\KL(\pitstar\|\pi_{t-1})
    \le
    (1-\eta\beta)\KL(\pistar\|\pi_{t-1})-\KL(\pistar\|\pitstar).
\]
Dropping the nonpositive term $-\eta\beta\KL(\pistar\|\pi_{t-1})$ and adding then subtracting $\KL(\pistar\|\pi_t)$ gives
\[
    c_\eta\KL(\pitstar\|\pi_{t-1})
    \le
    \KL(\pistar\|\pi_{t-1})-\KL(\pistar\|\pi_t)
    +\KL(\pistar\|\pi_t)-\KL(\pistar\|\pitstar).
\]
Sum over $t=1,\ldots,T$ and telescope:
\[
    c_\eta\sum_{t=1}^T\KL(\pitstar\|\pi_{t-1})
    \le \KL(\pistar\|\pi_0)+E_T(\pistar).
\]
Finally apply \cref{lem:reverse-movement} to each term to obtain the desired result. 
\end{proof}

We now turn to concentration. The following martingale Hoeffding bound handles adaptively chosen but conditionally bounded samples. 

\begin{lemma}\label{lem:master-azuma}
Let $(\mathcal F_i)_{i=0}^N$ be a filtration and let $(X_i)_{i=1}^N$ be adapted. Suppose deterministic constants $a_i\le b_i$ satisfy $a_i\le X_i\le b_i$ almost surely. Then, for every $\delta\in(0,1)$, with probability at least $1-\delta$, 
\[
\left|\sum_{i=1}^N X_i-\sum_{i=1}^N\E[X_i\mid\mathcal F_{i-1}]\right|
\le
\sqrt{\frac12\left(\sum_{i=1}^N(b_i-a_i)^2\right)\log\left(\frac2\delta\right)}.
\]
\end{lemma}
\begin{proof}
Let $D_i:=X_i-\E[X_i | \mathcal F_{i-1}]$ and $S_n := \sum_{i=1}^nD_i$. Then $(S_n)$ is a martingale. Conditional on $\mathcal F_{i-1}$, $D_i$ has mean zero and lies in an interval of length at most $c_i:=b_i-a_i$. Hoeffding's lemma applied conditionally gives
\[
    \E[e^{\theta D_i}\mid\mathcal F_{i-1}]
    \le \exp\left(\frac{\theta^2c_i^2}{8}\right),
    \qquad \theta\in\mathbb R.
\]
Thus $M_n:=\exp(\theta S_n-\theta^2\sum_{i=1}^n c_i^2/8)$ is a nonnegative supermartingale. For $\theta>0$ and $\varepsilon>0$, Markov's inequality gives
\[
    P(S_N\ge\varepsilon)
    \le \exp\left(-\theta\varepsilon+\frac{\theta^2}{8}\sum_{i=1}^N c_i^2\right).
\]
Optimizing over $\theta$ gives
\[
    P(S_N\ge\varepsilon)
    \le \exp\left(-\frac{2\varepsilon^2}{\sum_i c_i^2}\right).
\]
Applying the same bound to $-S_N$ and union bounding gives
\[
    P(|S_N|\ge\varepsilon)
    \le 2\exp\left(-\frac{2\varepsilon^2}{\sum_i c_i^2}\right).
\]
The claim follows by taking $\varepsilon=\sqrt{(\sum_ic_i^2)\log(2/\delta)/2}$.
\end{proof}

This corollary records the conditional i.i.d.\ form used for the per-round batches in \cref{alg:adv-exp-NashMD-appendix}.

\begin{corollary}\label{cor:conditional-hoeffding}
Let $\mathcal G$ be a sigma-field. Conditional on $\mathcal G$, let $X_1,\ldots,X_n$ be independent random variables with $X_i\in[a,b]$ almost surely and $\E[X_i\mid\mathcal G]=\mu$ for every $i$. Then, for every $\epsilon>0$,
\[
    P\left(\left.\left|\frac1n\sum_{i=1}^nX_i-\mu\right| > \epsilon \right| \mathcal G\right)
    \le 2\exp\left(-\frac{2n\epsilon^2}{(b-a)^2}\right).
\]
Equivalently, for every $\alpha\in(0,1)$, with conditional probability at least $1-\alpha$,
\[
    \left|\frac1n\sum_{i=1}^nX_i-\mu\right|
    \le (b-a)\sqrt{\frac{\log(2/\alpha)}{2n}}.
\]
\end{corollary}
\begin{proof}
Apply \cref{lem:master-azuma} to the filtration $\mathcal F_i:=\mathcal G\vee\sigma(X_1,\ldots,X_i)$, $i=0,\ldots,n$. Since the variables are conditionally independent given $\mathcal G$ and have conditional mean $\mu$, we have $\E[X_i\mid\mathcal F_{i-1}]=\E[X_i\mid\mathcal G]=\mu$. Moreover, the range length is $b-a$ for every $i$. \cref{lem:master-azuma} gives the displayed high-probability inequality conditionally on $\mathcal G$; rewriting it in tail form gives the exponential bound. 
\end{proof}

{\bf Here, we restate the batch sample sizes defined in~\cref{thm:concrete}.}

For $T\ge 1$ and $\delta\in(0,1)$, define the batch sample sizes as follows:
\begin{equation}\label{eq:explicit-samples}
    m_t:=t,
    \qquad
    n_t:=t,
    \qquad
    k_t:=t,
    \qquad t=1,\ldots,T.
\end{equation}
The logarithmic factors are moved into the ridge parameter and the final regret bound. Set
\begin{equation}\label{eq:explicit-Lpi}
    L_\Pi:=\log\left(\frac{16T|\Pi|^2}{\delta}\right).
\end{equation}
Because $n_r=r$, the total number of active atoms across all rounds is bounded by
\begin{equation}\label{eq:explicit-MT}
    M_T:=2\sum_{r=1}^T (T-r+1)r
    =\frac{T(T+1)(T+2)}{3}.
\end{equation}
Define
\begin{equation}\label{eq:explicit-LP-LT}
    L_P:=\log\left(\frac{16M_T}{\delta}\right),
    \qquad
    L_T:=\max\{1,L_\Pi,L_P,\log(eT)\}.
\end{equation}
Then $\sum_{t=1}^T |\mathcal Y_t^{\rm act}|\le M_T$. Define
\begin{equation}\label{eq:explicit-params}
    \overline H:=1+B^2+\beta^{-2},\qquad
    \lambda:=16\overline H L_T^2,
    \qquad
    S_T\ge 1\vee\SEC(\Psi,\lambda),
    \qquad
    \gamma:=\sqrt{\frac{S_T}{\lambda T}}.
\end{equation}
Use the learner coefficient
\begin{equation}\label{eq:explicit-alpha}
    \alpha_t=\frac{2}{\gamma m_t}=\frac{2}{\gamma t}.
\end{equation}
Finally, fix the constant stepsize
\begin{equation}\label{eq:explicit-eta}
    \eta:=\frac{2\beta}{1+4\beta^2}.
\end{equation}
Then $\eta\le 4\beta/(1+4\beta^2)$, $\eta\beta<1$, $\eta\le 1$, and
\begin{equation}\label{eq:ceta-explicit}
    c_\eta=1-\eta\beta-\frac{\eta}{4\beta}=\frac12.
\end{equation}

The final lemma verifies that all deterministic events required above hold with high probability. It controls the explorer numerator, the explorer denominator, and the active-set preference estimates simultaneously. 

\begin{lemma}\label{lem:explicit-events}
Under the sampling rules of \cref{alg:adv-exp-NashMD-appendix} and the choices \cref{eq:explicit-Lpi,eq:explicit-samples}, with probability at least $1-\delta$, all of the following statements hold simultaneously for all $t\le T$.
\begin{enumerate}[label=(\roman*), leftmargin=*]
\item For all $\rho',\pi\in\Pi$ and all $\nu\in\Pi$,
\begin{equation}\label{eq:explicit-N}
    |N_t^{\rho',\pi}(\nu)-\widehat N_t^{\rho',\pi}(\nu)|\le \epsilon_{m,t},
    \qquad
    \epsilon_{m,t}:=B\sqrt{\frac{2L_\Pi}{m_t}}.
\end{equation}
Although the event is stated for all $\nu\in\Pi$, the random deviation depends only on the on-policy empirical average, so no union bound over $\nu$ is needed. In particular, \cref{eq:local-event-m} holds with the same $\epsilon_{m,t}$.
\item For all $\rho',\pi\in\Pi$,
\begin{equation}\label{eq:explicit-H}
    |\widehat H_t^{\rho',\pi}-H_t^{\rho',\pi}|\le \frac{\lambda}{2}.
\end{equation}
Consequently, for all $\pi\in\Pi$,
\begin{equation}\label{eq:explicit-H-comparison}
    \frac12\bigl(\lambda+H_t(\pi)\bigr)
    \le \lambda+\widehat H_t(\pi)
    \le 2\bigl(\lambda+H_t(\pi)\bigr),
\end{equation}
and \cref{eq:local-event-H} holds.
\item For every active prompt-response atom in the empirical historical loss at round $t$,
\begin{equation}\label{eq:explicit-p}
    |\widehat P_t(a\succ\pi_{t-1})-P^\star(a\succ\pi_{t-1})|
    \le \epsilon_{p,t},
    \qquad
    \epsilon_{p,t}^2:=\frac{L_P}{2k_t}.
\end{equation}
\end{enumerate}
\end{lemma}
\begin{proof}
We prove the three claims and then union bound.
\begin{enumerate}[label=(\roman*), leftmargin=*]
    \item Fix $t$, $\rho',\pi\in\Pi$, and $\nu\in\Pi$. Let $\mathcal G_t^m$ be the sigma-field containing all randomness before drawing the on-policy batch $\widetilde\tau_{t-1,1:m_t}$. Conditional on $\mathcal G_t^m$, the variables
    \[
        X_i:=\log\frac{\rho'(\widetilde\tau_{t-1,i})}{\pi(\widetilde\tau_{t-1,i})},\qquad i=1,\ldots,m_t,
    \]
    are independent because $\widetilde\tau_{t-1,i}$ are sampled independently from $\pi_{t-1}$. \cref{ass:ratio} gives $X_i\in[-B,B]$ almost surely. Moreover,
    \[
        \E[X_i\mid\mathcal G_t^m]
        =\E_{\tau\sim\pi_{t-1}}\left[\log\frac{\rho'(\tau)}{\pi(\tau)}\right].
    \]
    The $\nu$-dependent term $\E_{\tau'\sim\nu}[\log(\rho'(\tau')/\pi(\tau'))]$ is deterministic conditional on $\mathcal G_t^m$ and appears in both $N_t^{\rho',\pi}(\nu)$ and $\widehat N_t^{\rho',\pi}(\nu)$, so it cancels. Hence
    \[
        N_t^{\rho',\pi}(\nu)-\widehat N_t^{\rho',\pi}(\nu)
        =\E[X_1\mid\mathcal G_t^m]-\frac1{m_t}\sum_{i=1}^{m_t}X_i.
    \]
    Apply \cref{cor:conditional-hoeffding} with $a=-B$, $b=B$, $n=m_t$, and $\epsilon=B\sqrt{2L_\Pi/m_t}$. Since $b-a=2B$, this gives
    \[
        P\left(\left.|N_t^{\rho',\pi}(\nu)-\widehat N_t^{\rho',\pi}(\nu)|>\epsilon\ \right|\ \mathcal G_t^m\right)
        \le 2e^{-L_\Pi}.
    \]
    This bound is uniform in $\nu$ because the random deviation does not depend on $\nu$. Therefore no union bound over $\nu$ is needed. Union bounding over $t\le T$ and $(\rho',\pi)\in\Pi^2$ gives total failure probability at most $2T|\Pi|^2e^{-L_\Pi}=\delta/8$. The local event \cref{eq:local-event-m} follows by substituting $(\rho',\pi)=(\pitstar,\pi_t)$; this is legitimate because \cref{ass:real} implies $\pitstar\in\Pi$ and \cref{alg:adv-exp-NashMD-appendix} outputs $\pi_t\in\Pi$. 
    \item Fix $t$ and $(\rho',\pi)\in\Pi^2$. For $s=0,\ldots,t-1$ and $i=1,\ldots,n_{s+1}$ define
    \[
        Z_{s,i}:=\frac1{n_{s+1}}\psi_{\rho',\pi}(\tau_{s,i},\tau'_{s,i})^2.
    \]
    Order the pairs $(s,i)$ lexicographically and let $\mathcal F_{s,i}$ be the filtration generated by all samples revealed up to pair $(s,i)$. Conditional on the history before drawing the $s$-th exploration batch, the pairs $(\tau_{s,i},\tau'_{s,i})$ are independent with law $\pi_{s-1}\otimes\widehat\pi_s$. Hence $Z_{s,i}$ is adapted and
    \[
        \E[Z_{s,i}\mid\mathcal F_{s,i-1}]
        =\frac1{n_{s+1}}\E_{\tau\sim\pi_{s-1},\tau'\sim\widehat\pi_s}\left[\psi_{\rho',\pi}(\tau,\tau')^2\right].
    \]
    By \cref{ass:ratio}, $|\log(\rho'(\tau)/\pi(\tau))|\le B$, so $|\psi_{\rho',\pi}(\tau,\tau')|\le 2B$. Therefore $0\le Z_{s,i}\le 4B^2/n_{s+1}$, and the range length in \cref{lem:master-azuma} is $4B^2/n_{s+1}$. Applying \cref{lem:master-azuma} to the ordered collection $\{Z_{s,i}\}_{s<t,i\le n_{s+1}}$ with confidence parameter $\alpha=2e^{-L_\Pi}$ gives, for this fixed $t,\rho',\pi$,
    \[
        |\widehat H_t^{\rho',\pi}-H_t^{\rho',\pi}|
        \le
        \sqrt{\frac12\sum_{s=0}^{t-1}n_{s+1}\left(\frac{4B^2}{n_{s+1}}\right)^2L_\Pi}
        =4B^2\sqrt{\frac12\left(\sum_{s=0}^{t-1}\frac1{n_{s+1}}\right)L_\Pi}
    \]
    with failure probability at most $2e^{-L_\Pi}$. Union bounding over $t\le T$ and $(\rho',\pi)\in\Pi^2$ costs at most another $2T|\Pi|^2e^{-L_\Pi}=\delta/8$. Under the simple schedule $n_{s+1}=s+1$,
    \[
        \sum_{s=0}^{t-1}\frac1{n_{s+1}}\le\sum_{r=1}^T\frac1r\le\log(eT)\le L_T,
        \qquad L_\Pi\le L_T.
    \]
    Thus
    \[
        |\widehat H_t^{\rho',\pi}-H_t^{\rho',\pi}|
        \le 4B^2\frac{L_T}{\sqrt2}
        \le 4\overline H L_T
        \le 8\overline H L_T^2
        =\frac{\lambda}{2},
    \]
    where $B^2\le\overline H$, $L_T\ge1$, and $\lambda=16\overline H L_T^2$. The comparison \cref{eq:explicit-H-comparison} follows by applying this bound to $(\rho',\pi)=(\pi_t,\pi)$ and using $|\widehat H-H|\le\lambda/2$.
    \item Fix $t$ and condition on the history $\mathcal G_t^p$ just before drawing the comparison opponents for preference estimation. For a fixed active prompt-response atom $a=(\x,\y)\in\mathcal Y_t^{\rm act}$, \cref{alg:adv-exp-NashMD-appendix} draws $c_{t,a,j}\sim\pi_{t-1}(\cdot\mid \x)$ independently and then observes
\[
    Y_{t,a,j}\sim\mathrm{Bern}\big(P^\star(\y\succ c_{t,a,j}\mid \x)\big),\qquad j=1,\ldots,k_t.
\]
The variables $Y_{t,a,j}$ are conditionally independent and lie in $[0,1]$. By the tower property,
\[
    \E[Y_{t,a,j}\mid\mathcal G_t^p]
    =\E_{c\sim\pi_{t-1}(\cdot\mid \x)}[P^\star(\y\succ c\mid \x)]
    =P^\star(a\succ\pi_{t-1}).
\]
Apply \cref{cor:conditional-hoeffding} with $a=0$, $b=1$, $n=k_t$, and $\epsilon=\sqrt{L_P/(2k_t)}$. For this fixed active atom, the conditional failure probability is at most $2e^{-L_P}$. Since the active sets are history-measurable before the labels at round $t$ are drawn and $\sum_{t=1}^T|\mathcal Y_t^{\rm act}|\le M_T$, union bounding over all realized active atoms and all times gives failure probability at most $2M_Te^{-L_P}=\delta/8$. 
\end{enumerate}
Combining the three bounds gives total failure probability at most $3\delta/8<\delta$.
\end{proof}

\subsection{Proofs of improved sample efficiency of \texttt{BENPO}}
\label{app:subsec:BENPO}

\subsubsection{Algorithm}

\begin{algorithm}[t]
    \caption{Bonus-Explorer Exploratory Nash Preference Optimization (BENPO)}
    \KwIn{
        horizon $T$,
        reference policy $\pi_{\mathrm{ref}}$,
        confidence-radius schedule $\{\gamma_{t}(\delta)\}_{t \geq 1}$,
        coefficient $\lambda > 0$}
    \textbf{Initialize} $D^{0} \leftarrow \varnothing$\;
    \For{$t = 1, 2, \ldots, T$}{
      Compute MLE $\widehat{P}_{t}$ via~\cref{eq:MLE}\;
      Compute a Nash policy under $\widehat{P}_{t}$ by solving
      \begin{equation*}
          \pi_{t} \in \argmax_{\pi_{1}} \min_{\pi_{2}}  \, J_{\beta} \bigl(\pi_{1},\pi_{2}; \widehat{P}_{t}\bigr);
      \end{equation*}
      
      Generate a prompt-response pair $(\x_t, \y_t)$ via $\x_t \sim \rho$ and $\y_t \sim \pi_{t}(\cdot \mid \x_t)$\;
      Construct UCB bonus at $(\x_t, \y_t)$ for any $\y' \in \mathcal{Y}_{\x_t}$: $\mathfrak{b}_{t}(\y'; \x_t, \y_t)$ by 
      \begin{equation*}
          \min \left\{ 1, 
          \sqrt{\gamma_{t}(\delta)} \sup_{P, P' \in \mathcal{P}} 
          \frac{ \bigl| P(\y_t \succ \y' | \x_t) - P'(\y_t \succ \y' | \x_t) \bigr|}
           {\sqrt{ \lambda + \sum_{s = 1}^{t - 1} \bigl( P(\y_s \succ \hat{\y}_s | \x_s) - P'(\y_s \succ \hat{\y}_s | \x_s) \bigr)^{2}}} \right\};
      \end{equation*}
      
      Select an exploratory response:
      \begin{equation*}
          \hat{\y}_t \in \argmax_{\y' \in \mathcal{Y}_{\x_t}} \mathfrak{b}_{t}(\y'; \x_t, \y_t);
      \end{equation*}
      
      Query preference label $I_t \sim \textnormal{Bern}(P^{\star}(\y_t \succ \hat{\y}_t | \x_t ))$\;
      Update dataset $D^{t} \leftarrow D^{t-1} \cup \{(\x_t, \y_t, \hat{\y}_t, I_t)\}$\;
    }
    \KwOut{$\{\pi_{t}\}_{t=1}^{T}$}
    \label{alg:BENPO}
\end{algorithm}

We first provide the high-level idea of Bonus-Explorer ENPO\@.
At round $t$, 
\texttt{BENPO} computes the MLE $\widehat{P}_{t}$ of the preference function from the data collected so far:  
\begin{equation}
    \label{eq:MLE}
    \widehat{P}_t \in \argmax_{P \in \mathcal{P}} \sum_{s=1}^{t-1} \Bigl[ I_s \log P(\mathbf{y}_s \succ \hat{\mathbf{y}}_s | \x_s) + (1 - I_s) \log\bigl(1 - P(\mathbf{y}_s \succ \hat{\mathbf{y}}_s | \x_s)\bigr) \Bigr],
\end{equation}
where $\x_s, \y_s, \hat{\y}_s$ are sampled prompts and responses at round $s$, and $I_s \sim \textnormal{Bern}(P^\star( \y_s \succ \hat{\y}_s | \x_s ))$ is the preference label collected at round $s$.
Then, it finds a regularized Nash policy $\pi_{t}$ by querying a minimax oracle:
\begin{equation*}
    \pi_{t} \in \argmax_{\pi_{1}} \min_{\pi_{2}}  \, \widehat{P}_t(\pi_1\succ\pi_2) - \beta\textnormal{KL}(\pi_1\|\pi_\textnormal{ref})+\beta\textnormal{KL}(\pi_2\|\pi_\textnormal{ref}).
\end{equation*}
This main-policy step plays the same role as \texttt{ENPO}'s policy update in~\cref{eq:learner-update-appendix}---it tracks the regret-relevant policy, 
the one whose suboptimality is the object of analysis.
We then draw one sample $\y_t \sim \pi_{t}$, and based on this sample, we solve the following maximization problem: 
\begin{equation}
    \max_{\y' \in \mathcal{Y}_{\x_t}} \left\{
    1 \wedge 
    \sqrt{\gamma_{t}(\delta)} \sup_{P, P' \in \mathcal{P}} 
    \frac{ \bigl| P(\y_t \succ \y' | \x_t) - P'(\y_t \succ \y' | \x_t) \bigr|}
    {\sqrt{ \lambda + \sum_{s = 1}^{t - 1} \bigl( P(\y_s \succ \hat{\y}_s | \x_s) - P'(\y_s \succ \hat{\y}_s | \x_s) \bigr)^{2}}} \right\},
    \label{eq:UCB-bonus}
\end{equation}
and denote this adversarial sample by $\hat{\y}_t$.
We then query a preference label $I_t$ between $\y_t$ and $\hat{\y}_t$.

See~\cref{alg:BENPO} for the full statement of the algorithm.

As a counterpart of ENPO with stronger oracles, BENPO aligns with ENPO in a clear manner: \emph{maintain a main policy $\pi_{t}$ whose regret is the object of analysis,
and pair that main policy with a second mechanism whose only role is to guide exploration}.

\subsubsection{Assumptions and additional notation}
\label{app:sec:BENPO-additional-assumptions}

Before we analyze the regret of \texttt{BENPO}, we first introduce some assumptions and notation.

Let \( \mathcal{P} \) be a finite realizable class of preference models: \( P^\star \in \mathcal{P} \). Furthermore, every \(P\in\mathcal P\) satisfies
\[
    P(\y \succ \y' \mid \x) + P(\y' \succ \y \mid \x) = 1,
    \textnormal{ thus }
    P(\y \succ \y\mid \x) = \tfrac{1}{2}.
\]

We define the KL-regularized game with a $t$-round empirical preference as follows: 
\begin{equation}
    \max_{\pi_1} \min_{\pi_2} J_\beta\bigl(\pi_1,\pi_2; \widehat{P}_t\bigr) := \widehat{P}_t(\pi_1\succ\pi_2) - \beta\textnormal{KL}(\pi_1\|\pi_\textnormal{ref})+\beta\textnormal{KL}(\pi_2\|\pi_\textnormal{ref})
    \label{eq:t-step-empirical-preference-game}
\end{equation}

Note that the maximum and the minimum are taken over all policies with the same support set as $\pi_{\mathrm{ref}}$, 
which is the class over which the Nash equilibrium $\pi^\star$ is defined in~\cref{sec:prelim}. 

We denote $J_\beta(\pi_1,\pi_2):=J_\beta(\pi_1,\pi_2;P^\star)$.
For each $P\in\mathcal{P}$, let $\pi^\star_P$ denote the unique symmetric Nash policy of the KL-regularized game under $P$, i.e., 
\begin{equation}
    J_\beta(\pi_1,\pi^\star_P;P)
    \leq J_\beta(\pi^\star_P,\pi^\star_P;P)
    \leq J_\beta(\pi^\star_P,\pi_2;P),
    \qquad \forall\, \pi_1, \pi_2.
    \label{eq:t-step-empirical-preference-game-Nash-policy-condition}
\end{equation}
By anti-symmetry of the preference model $P$, $J_\beta(\pi^\star_P,\pi^\star_P;P)=1/2$. 

We write $\mathrm{gap}_\beta(\pi):=\max_{\pi_1}J_\beta(\pi_1,\pi)-\min_{\pi_2}J_\beta(\pi,\pi_2)$ for the duality gap over this class. By~\cref{prop:regret-gap}, whose proof applies verbatim, $\mathrm{gap}_\beta(\pi)=2\big(\max_{\pi_1}J_\beta(\pi_1,\pi)-\tfrac12\big)$.

Throughout~\cref{app:subsec:BENPO}, 
all maxima and minima over policies, 
including those defining $\mathrm{gap}_\beta$, 
are taken over this class rather than the finite class of~\cref{ass:finite}. 
This is the natural choice for BENPO: 
its goal is to learn this Nash equilibrium, 
and at each round it solves the same KL-regularized minimax problem with $P^\star$ replaced by the estimate $\hat P_t$, 
so that $\pi_t := \pi^\star_{\hat P_t}$ is a plug-in estimate of $\pi^\star$.

We specify that the statistical domain is
\[
    \mathcal{Z}:=\{(\x,\y,\y'):\x\in\X,\ \y,\y'\in\A_x\}.
\]

\rev{
Following~\citet{zhao2025logarithmic} (see also~\citet{Ye-2024-Online}), we measure the complexity of $\mathcal P$ by the eluder coefficient defined as follows: 
\begin{equation*}
    d_{\mathcal P} := \sup_{z_1, \dots, z_T\in\mathcal Z}\sum_{t=1}^T\min \left\{ 1, \ \sup_{P,P'\in\mathcal P}\frac{ \left(P( \y_t \succ \y'_t | \x_t) - P'(\y_t \succ \y'_t | \x_t ) \right)^2}{\lambda + \sum_{s<t}\left(P( \y_s \succ \y'_s | \x_s) - P'(\y_s \succ \y'_s | \x_s ) \right)^2} \right\},
\end{equation*}
where $z_t := (\x_t, \y_t, \y'_t)$ for all $1 \leq t \leq T$. 
}

\subsubsection{Main proof of sample efficiency}

The following squared $\ell^2$ bound on regret, or equivalently the duality gap, is the key to obtaining the logarithmic regret bound.
The proof strategy follows that in the proof of~\cite[Theorem 3.1.]{Lee-2026-Regularized}.
\begin{lemma}[Squared $\ell^2$ bound]
    \label{lem:sq-l2}
    Let $\widehat{P}$ be any antisymmetric preference function, and let $\pi_{\widehat{P}}^\star$ 
    denote the symmetric regularized Nash policy under $\widehat{P}$:
    \[
        \pi_{\widehat{P}}^\star \in \arg\max_{\pi_{1}} \min_{\pi_{2}} J_{\beta}(\pi_{1}, \pi_{2}; \widehat{P}).
    \]
    Then, we have 
    \[
        \textnormal{gap}_{\beta}(\pi_{\widehat{P}}^\star) \leq \frac{4}{\beta} \EE_{\x \sim \rho, \y \sim \pi_{\widehat{P}}^\star(\cdot \mid \x)}\!\left[\,\sup_{\y' \in \mathcal{Y}}\bigl(P^{\star}(\y \succ \y' \mid \x) - \widehat{P}(\y \succ \y' \mid \x)\bigr)^{2}\,\right].
    \]
\end{lemma}

\begin{proof}
    Denote 
    \begin{equation*}
        \widetilde{\pi} := \arg\max_{\pi_{1}} J_{\beta}(\pi_{1}, \pi_{\widehat{P}}^\star; P^{\star}).
    \end{equation*}
    First, note that 
    \begin{align*}
        \frac{1}{2} \textnormal{gap}_{\beta}(\pi_{\widehat{P}}^\star)
        &= \max_{\pi_1} J_{\beta}(\pi_{1}, \pi_{\widehat{P}}^\star; P^{\star}) - \frac{1}{2} \\
        &= J_{\beta}(\widetilde{\pi}, \pi_{\widehat{P}}^\star; P^{\star}) - J_{\beta}(\widetilde{\pi}, \pi_{\widehat{P}}^\star; \widehat{P}) + J_{\beta}(\widetilde{\pi}, \pi_{\widehat{P}}^\star; \widehat{P}) - \frac{1}{2} \\
        &\leq J_{\beta}(\widetilde{\pi}, \pi_{\widehat{P}}^\star; P^{\star}) - J_{\beta}(\widetilde{\pi}, \pi_{\widehat{P}}^\star; \widehat{P}) \\ 
        &= \mathbb E_{\mathbf x\sim\rho}\,\mathbb E_{\mathbf y\sim\widetilde\pi(\cdot|\mathbf x),\,\mathbf y'\sim\pi^\star_{\widehat P}(\cdot|\mathbf x)}\big[P^\star(\mathbf y\succ\mathbf y'|\mathbf x)-\widehat P(\mathbf y\succ\mathbf y'|\mathbf x)\big],
    \end{align*}
    where the inequality follows from the fact that $\pi_{\widehat{P}}^\star$ is a Nash policy with preference $\widehat{P}$.

    \rev{
    Since, by the antisymmetry of $P^\star$ and $\widehat P$, for any $\pi$ and every $\x$,
    $\mathbb E_{\y,\y'\sim\pi(\cdot|\x)}\big[P^\star(\y\succ\y'|\x)-\widehat P(\y\succ\y'|\x)\big]=0$, we have
    \begin{align}
        & \frac{1}{2}\mathrm{gap}_\beta(\pi^\star_{\widehat P}) \notag \\ 
        \leq &\mathbb E_{\x\sim\rho}\Big[\mathbb E_{\y\sim\widetilde\pi(\cdot|\x),\y'\sim\pi^\star_{\widehat P}(\cdot|\x)}\big[P^\star(\y\succ\y'|\x)-\widehat P(\y\succ\y'|\x)\big] \notag\\
        & \hspace{80pt} -\mathbb E_{\y,\y'\sim\pi^\star_{\widehat P}(\cdot|\x)}\big[P^\star(\y\succ\y'|\x)-\widehat P(\y\succ\y'|\x)\big]\Big]\notag\\
        ={}&\mathbb E_{\x\sim\rho}\,\mathbb E_{\y'\sim\pi^\star_{\widehat P}(\cdot|\x)}\Big[\big\langle P^\star(\cdot\succ\y'|\x)-\widehat P(\cdot\succ\y'|\x),\ \widetilde\pi(\cdot|\x)-\pi^\star_{\widehat P}(\cdot|\x)\big\rangle\Big]\notag\\
        \le{}&\mathbb E_{\x\sim\rho}\Big[\mathbb E_{\y'\sim\pi^\star_{\widehat P}(\cdot|\x)}\Big[\sup_{\y\in\mathcal Y}\big|P^\star(\y\succ\y'|\x)-\widehat P(\y\succ\y'|\x)\big|\Big]\cdot\big\|\widetilde\pi(\cdot|\x)-\pi^\star_{\widehat P}(\cdot|\x)\big\|_1\Big]\notag\\
        \le{}&\Big(\mathbb E_{\x\sim\rho}\Big[\Big(\mathbb E_{\y'\sim\pi^\star_{\widehat P}(\cdot|\x)}\sup_{\y\in\mathcal Y}\big|P^\star(\y\succ\y'|\x)-\widehat P(\y\succ\y'|\x)\big|\Big)^2\Big]\Big)^{1/2} \Big(\mathbb E_{\x\sim\rho}\big\|\widetilde\pi(\cdot|\x)-\pi^\star_{\widehat P}(\cdot|\x)\big\|_1^2\Big)^{1/2},
        \label{eq:Holder-bound}
    \end{align}
    where the last two inequalities follow from H\"older's inequality for each $\x$ and the Cauchy--Schwarz inequality over $\x\sim\rho$.

    On the other hand, $-\beta\,\mathrm{KL}(\pi_1\|\pi_{\mathrm{ref}}) = -\beta\,\EE_{\x \sim \rho} \mathrm{KL}(\pi_1(\cdot|\x)\|\pi_{\mathrm{ref}}(\cdot|\x))$ is $\beta$-strongly concave with respect to $\|\pi\| := (\EE_{\x \sim \rho} \|\pi(\cdot|\x)\|_1^2)^{1/2}$ 
    (by Pinsker's inequality, prompt by prompt), 
    and $\tilde\pi$ maximizes the concave map $\pi_1\mapsto J_\beta(\pi_1,\pi^\star_{\hat P};P^\star)$ over this (convex) policy class. Thus, $J_\beta(\pi_1,\pi^\star_{\hat P};P^\star)$ is $\beta$-strongly concave in $\pi_1$ and the first-order optimality condition holds at $\tilde\pi$.
    Also, we have $J_\beta(\pi^\star_{\widehat{P}}, \pi^\star_{\widehat{P}}; P^\star) = \tfrac{1}{2}$; therefore, 
    \begin{equation}
        \frac{1}{2}\textnormal{gap}_\beta(\pi^\star_{\widehat{P}})
        = J_\beta(\widetilde{\pi}, \pi^\star_{\widehat{P}}; P^\star) - J_\beta(\pi^\star_{\widehat{P}}, \pi^\star_{\widehat{P}}; P^\star)
        \geq \frac{\beta}{2} \big\|\widetilde{\pi} - \pi^\star_{\widehat{P}}\big\|^2.
        \label{eq:gap-quadratic-lower-bound}
    \end{equation}
    }
    
    By rearranging~\cref{eq:gap-quadratic-lower-bound}, we have 
    \begin{equation}
        \left( \EE_{\x \sim \rho} \left[ \big\|\widetilde{\pi}(\cdot|\x) - \pi^\star_{\widehat{P}}(\cdot|\x)\big\|_1^2 \right] \right)^{1/2} \leq \sqrt{\textnormal{gap}_\beta(\pi^\star_{\widehat{P}})\big/\beta}.
        \label{eq:l1-norm-upper-bound}
    \end{equation}
    
    Substituting~\cref{eq:l1-norm-upper-bound} into~\cref{eq:Holder-bound} yields the quadratic inequality
    \[
        \frac12\mathrm{gap}_\beta(\pi^\star_{\widehat P})\le\Big(\mathbb E_{\mathbf x\sim\rho}\Big[\Big(\mathbb E_{\mathbf y'\sim\pi^\star_{\widehat P}(\cdot|\mathbf x)}\sup_{\mathbf y\in\mathcal Y}\big|P^\star(\mathbf y\succ\mathbf y'|\mathbf x)-\widehat P(\mathbf y\succ\mathbf y'|\mathbf x)\big|\Big)^2\Big]\Big)^{1/2}\sqrt{\mathrm{gap}_\beta(\pi^\star_{\widehat P})/\beta},
    \]
    which solves to
    \[
        \mathrm{gap}_\beta(\pi^\star_{\widehat P})\le\frac4\beta\,\mathbb E_{\mathbf x\sim\rho}\Big[\Big(\mathbb E_{\mathbf y'\sim\pi^\star_{\widehat P}(\cdot|\mathbf x)}\sup_{\mathbf y\in\mathcal Y}\big|P^\star(\mathbf y\succ\mathbf y'|\mathbf x)-\widehat P(\mathbf y\succ\mathbf y'|\mathbf x)\big|\Big)^2\Big].
    \]
    By Jensen's inequality, we obtain
    \begin{equation}
        \mathrm{gap}_\beta(\pi^\star_{\widehat P})\le\frac4\beta\,\mathbb E_{\mathbf x\sim\rho}\,\mathbb E_{\mathbf y'\sim\pi^\star_{\widehat P}(\cdot|\mathbf x)}\Big[\sup_{\mathbf y\in\mathcal Y}\big(P^\star(\mathbf y\succ\mathbf y'|\mathbf x)-\widehat P(\mathbf y\succ\mathbf y'|\mathbf x)\big)^2\Big].
        \label{eq:second-order-1}
    \end{equation}
    Note that, by the antisymmetry of the preferences, we further have 
    \begin{equation*}
        \big(P^\star(\mathbf{y} \succ \mathbf{y}'|\x) - \widehat{P}(\mathbf{y} \succ \mathbf{y}'|\x)\big)^2 = \big(P^\star(\mathbf{y}' \succ \mathbf{y}|\x) - \widehat{P}(\mathbf{y}' \succ \mathbf{y}|\x)\big)^2, 
    \end{equation*}
    
    By swapping the names of $\y$ and $\y'$ in~\cref{eq:second-order-1}, we complete the proof. 
\end{proof}

\begin{theorem}[Logarithmic regret bound of~\cref{alg:BENPO}]
    \label{thm:benpo-regret}
    Suppose $\mathcal{P}$ is a finite class of antisymmetric preference functions with $P^\star \in \mathcal{P}$. 
    In~\cref{alg:BENPO}, set $\lambda\in(0,1]$ and 
    \begin{equation*}
        \gamma_t(\delta) = 4\log(2T|\mathcal{P}|/\delta) + \lambda, \quad \forall\, 1 \leq t \leq T.
    \end{equation*}
    Then, the iterates $\{\pi_t\}_{t=1}^T$ produced by \cref{alg:BENPO} satisfy
    \begin{equation*}
        \sum_{t=1}^T \textnormal{gap}_\beta(\pi_t) \leq \widetilde{O}\!\left(\frac{d_\mathcal{P}\,\log(T|\mathcal{P}|/\delta)\,\log T}{\beta}\right)
    \end{equation*}
    with probability at least $1 - \delta$.
\end{theorem}
\begin{proof}[Proof of~\cref{thm:benpo-regret}]
    We begin by reducing the per-round suboptimality to a \emph{squared} preference error via~\cref{lem:sq-l2}.
    Since $\pi_t = \pi^\star_{\widehat{P}_t}$ by construction, by~\cref{lem:sq-l2} we have 
    \begin{equation}
        \textnormal{gap}_\beta(\pi_t) \leq \frac{4}{\beta} \EE_{\x \sim \rho,\, \mathbf{y} \sim \pi_t(\cdot | \x)} \left[ \sup_{\mathbf{y}' \in \mathcal{Y}} \bigl(P^\star(\mathbf{y} \succ \mathbf{y}' | \x) - \widehat{P}_t(\mathbf{y} \succ \mathbf{y}' | \x)\bigr)^2 \right].
        \label{eq:benpo-per-round-bound}
    \end{equation}

    Throughout this proof, we abbreviate $\Delta_t(\mathbf{y}, \mathbf{y}' | \x) := P^\star(\mathbf{y} \succ \mathbf{y}'|\x) - \widehat{P}_t(\mathbf{y} \succ \mathbf{y}'|\x)$, and denote $\mathcal{F}_{t-1} := \sigma\bigl(\{\x_s, \mathbf{y}_s, \hat{\y}_s, I_s\}_{s < t}\bigr)$ and $\EE_t[\,\cdot\,] := \EE[\,\cdot \mid \mathcal{F}_{t-1}]$.

    By~\cref{lem:mle-conc}, there exists an event $\mathcal{E}_1$ with $\mathbb{P}(\mathcal{E}_1) \geq 1 - \delta/2$ on which
    \begin{equation*}
        \sum_{s=1}^{t-1} \left( P^\star(\mathbf{y}_s \succ \hat{\y}_s | \x_s) - \widehat{P}_t(\mathbf{y}_s \succ \hat{\y}_s | \x_s) \right)^2 \leq 4\log(2T|\mathcal{P}|/\delta) \qquad \forall\, t \in [T].
    \end{equation*}

    Because $\gamma_t(\delta) := 4\log(2T|\mathcal{P}|/\delta) + \lambda$, taking $P = P^\star$, $P' = \widehat{P}_t$ inside the supremum that defines $\mathfrak{b}_t$, on the event $\mathcal{E}_1$, we have $\lambda + \sum_{s<t}\Delta_t(\y_s, \hat{\y}_s | \x_s)^2 \leq \gamma_t(\delta)$. 
    Note that $|\Delta_t(\y_s, \y'_s | \x_s)| \leq 1$ holds trivially. 
    Then, 
    \begin{equation*}
        \Delta_t(\mathbf{y}, \mathbf{y}' | \x)^2 = \min\left\{ \Delta_t(\mathbf{y}, \mathbf{y}' | \x)^2, 1 \right\} \leq \min\left\{ \frac{\gamma_t(\delta)\Delta_t(\mathbf{y}, \mathbf{y}' | \x)^2}{\lambda + \sum_{s=1}^{t-1} \Delta_t(\mathbf{y}_s, \hat{\y}_s | \x_s)^2}, 1 \right\} \leq \mathfrak{b}_t(\mathbf{y}'; \x, \mathbf{y})^2,
    \end{equation*}
    therefore $|\Delta_t(\mathbf{y}, \mathbf{y}' | \x)| \leq \mathfrak{b}_t(\mathbf{y}'; \x, \mathbf{y})$ for all $(\x, \mathbf{y}, \mathbf{y}') \in \mathcal{X} \times \mathcal{Y}^2$. 
    
    Squaring, taking $\sup_{\mathbf{y}' \in \mathcal{Y}}$, and noting that the explorer $\hat{\y}_t = \arg\max_{\mathbf{y}' \in \mathcal{Y}_{\x_t}} \mathfrak{b}_t(\mathbf{y}'; \x_t, \mathbf{y}_t)$ realizes the supremum, we obtain 
    \begin{equation*}
        \sup_{\mathbf{y}' \in \mathcal{Y}_{\x_t}} \Delta_t(\mathbf{y}_t, \mathbf{y}' | \x_t)^2 \leq \mathfrak{b}_t(\hat{\y}_t; \x_t, \mathbf{y}_t)^2.
    \end{equation*}
    Substituting into~\cref{eq:benpo-per-round-bound} yields 
    \begin{equation}
        \textnormal{gap}_\beta(\pi_t) \leq \frac{4}{\beta} \EE_t\!\bigl[\mathfrak{b}_t(\hat{\y}_t; \x_t, \mathbf{y}_t)^2\bigr].
        \label{eq:benpo-per-round-bonus}
    \end{equation}

    We next concentrate the squared bonus around its conditional expectation. Since $\mathfrak{b}_t \in [0, 1]$ by construction, the sequence
    \begin{equation*}
        X_t \;:=\; \EE_t\!\bigl[\mathfrak{b}_t(\hat{\y}_t; \x_t, \mathbf{y}_t)^2\bigr] - \mathfrak{b}_t(\hat{\y}_t; \x_t, \mathbf{y}_t)^2
    \end{equation*}
    is a martingale difference sequence with $|X_t| \leq 1$ and $\EE_t[X_t^2] \leq \EE_t\bigl[\mathfrak{b}_t(\hat{\y}_t; \x_t, \mathbf{y}_t)^2\bigr]$. Applying~\cref{lem:freedman} with $Y_t = \mathfrak{b}_t(\hat{\y}_t; \x_t, \mathbf{y}_t)^2$ and confidence $\delta/2$ yields an event $\mathcal{E}_2$ with $\mathbb{P}(\mathcal{E}_2) \geq 1 - \delta/2$ on which
    \begin{equation}
        \sum_{t=1}^T \EE_t\!\bigl[\mathfrak{b}_t(\hat{\y}_t; \x_t, \mathbf{y}_t)^2\bigr] \leq 2 \sum_{t=1}^T \mathfrak{b}_t(\hat{\y}_t; \x_t, \mathbf{y}_t)^2 + 4\log(2/\delta).
        \label{eq:benpo-freedman}
    \end{equation}

    To control the cumulative realized squared bonus, observe that by the definition of $\mathfrak{b}_t$ and $\gamma_t(\delta) \equiv \gamma_T(\delta) = 4\log(2T|\mathcal{P}|/\delta) + \lambda \geq 1$,
    \begin{equation*}
        \mathfrak{b}_t(\hat{\y}_t; \x_t, \mathbf{y}_t)^2 \leq \gamma_T(\delta) \cdot \min\!\Bigg\{1,\; \sup_{P, P' \in \mathcal{P}} \frac{(P - P')^2 (\mathbf{y}_t \succ \hat{\y}_t | \x_t)}{\lambda + \sum_{s=1}^{t-1} (P - P')^2 (\mathbf{y}_s \succ \hat{\y}_s | \x_s)} \Bigg\}.
    \end{equation*}
    By the definition of $d_{\mathcal P}$, we have 
    \begin{equation}
        \sum_{t=1}^T\mathfrak b_t(\hat{\y}_t; \x_t, \y_t)^2 \leq \gamma_T(\delta)\,d_{\mathcal P}.
        \label{eq:benpo-eluder}
    \end{equation}
    
    Combining the pieces, on $\mathcal{E}_1 \cap \mathcal{E}_2$ (which has probability at least $1 - \delta$ by a union bound), summing~\cref{eq:benpo-per-round-bonus} over $t \in [T]$ and chaining~\cref{eq:benpo-freedman,eq:benpo-eluder},
    \begin{align*}
        \sum_{t=1}^T \textnormal{gap}_\beta(\pi_t)
        &\leq \frac{4}{\beta} \sum_{t=1}^T \EE_t\!\bigl[\mathfrak{b}_t(\widehat{\mathbf{y}}_t; \x_t, \mathbf{y}_t)^2\bigr] \\
        &\leq \frac{8}{\beta} \sum_{t=1}^T \mathfrak{b}_t(\widehat{\mathbf{y}}_t; \x_t, \mathbf{y}_t)^2 + \frac{16 \log(2/\delta)}{\beta} \\
        &\leq \frac{8}{\beta} \cdot \gamma_T(\delta) \cdot d_\mathcal{P} + \frac{16 \log(2/\delta)}{\beta} \\
        &= \widetilde{O}\!\left(\frac{d_\mathcal{P} \cdot \log(|\mathcal{P}|/\delta) \cdot \log T}{\beta}\right).
    \end{align*}
\end{proof}

\subsubsection{Technical lemmas}

\begin{lemma}[MLE concentration]
    \label{lem:mle-conc}
    Let $\mathcal{P}$ be a finite class of preference functions $P: \VCal^\star \times \VCal^\star \times \VCal^\star \to [0, 1]$ 
    with $P^\star \in \mathcal{P}$. 
    Suppose that at each round $s \in [T]$ the preference label is generated as $I_s \sim \textnormal{Bern}\bigl(P^\star(\mathbf{y}_s \succ \widehat{\mathbf{y}}_s \mid \x_s)\bigr)$. Then, for any $\delta \in (0, 1)$, with probability at least $1 - \delta$,
    \begin{equation*}
        \sum_{s=1}^{t-1} \bigl(\widehat{P}_t(\mathbf{y}_s \succ \widehat{\mathbf{y}}_s \mid \x_s) - P^\star(\mathbf{y}_s \succ \widehat{\mathbf{y}}_s \mid \x_s)\bigr)^2 \leq 4 \log\bigl(T |\mathcal{P}| / \delta\bigr) \qquad \forall\, t \in [T].
    \end{equation*}
\end{lemma}

\begin{proof}
    Fix $t \in [T]$ and a candidate $P \in \mathcal{P}$. Write $p_s(P) := P(\mathbf{y}_s \succ \widehat{\mathbf{y}}_s \mid \x_s)$ for short, and let $\mathcal{F}_{s} := \sigma\bigl(\{\x_r, \mathbf{y}_r, \widehat{\mathbf{y}}_r, I_r\}_{r \leq s}\bigr)$ be the natural filtration and $\mathcal G_s:=\sigma\big(\mathcal F_{s-1},\mathbf x_s,\mathbf y_s,\widehat{\mathbf y}_s\big)$, which contains the round-$s$ query but not its label $I_s$. Define the half-log-likelihood-ratio process
    \begin{equation*}
        M_t(P) := \prod_{s=1}^{t-1} \sqrt{\frac{p_s(P)^{I_s} (1 - p_s(P))^{1 - I_s}}{p_s(P^\star)^{I_s} (1 - p_s(P^\star))^{1 - I_s}}}.
    \end{equation*}
    A direct computation gives
    \begin{align*}
        \EE\bigl[M_t(P) \,\big|\, \mathcal{G}_{t-1}\bigr] =& M_{t-1}(P) \cdot \Bigl(\sqrt{p_{t-1}(P) p_{t-1}(P^\star)} + \sqrt{(1 - p_{t-1}(P))(1 - p_{t-1}(P^\star))}\Bigr) \\ 
        =& M_{t-1}(P) \cdot \bigl(1 - \tfrac{1}{2} D_H^2(p_{t-1}(P), p_{t-1}(P^\star))\bigr),
    \end{align*}
    where $D_H^2(p, q) := (\sqrt{p} - \sqrt{q})^2 + (\sqrt{1-p} - \sqrt{1-q})^2$ is the squared Hellinger distance between Bernoulli$(p)$ and Bernoulli$(q)$. 
    Since $1-u\le e^{-u}$, the process
    \[
      W_t(P):=M_t(P)\exp\Big(\tfrac12\sum_{s=1}^{t-1}D_H^2\big(p_s(P),p_s(P^\star)\big)\Big)
    \]
    satisfies $\mathbb E\big[W_t(P)\mid\mathcal G_{t-1}\big]\le W_{t-1}(P)$. Hence $(W_t(P))_t$ is a non-negative supermartingale with $\mathbb E[W_t(P)]\le W_1(P)=1$, and by Markov's inequality, for any $P\in\mathcal P$ and $\delta'\in(0,1)$,
    \begin{equation}
        \Pr\!\left[\,\tfrac{1}{2}\sum_{s=1}^{t-1} D_H^2(p_s(P), p_s(P^\star)) \;>\; -\log M_t(P) + \log(1/\delta')\right] \leq \delta'.
        \label{eq:hellinger-tail}
    \end{equation}
    Setting $\delta' = \delta/(T|\mathcal{P}|)$ and taking a union bound over $P \in \mathcal{P}$ and $t \in [T]$, the event in~\cref{eq:hellinger-tail} fails for all $(P, t)$ with probability at least $1 - \delta$. On this event, plugging in $P = \widehat{P}_t$ and using $-\log M_t(\widehat{P}_t) \leq 0$ by the MLE optimality $\sum_{s < t} \log p_s(\widehat{P}_t)^{I_s}(1 - p_s(\widehat{P}_t))^{1-I_s} \geq \sum_{s < t} \log p_s(P^\star)^{I_s}(1 - p_s(P^\star))^{1-I_s}$,
    \begin{equation*}
        \sum_{s=1}^{t-1} D_H^2\bigl(p_s(\widehat{P}_t), p_s(P^\star)\bigr) \leq 2 \log\bigl(T|\mathcal{P}|/\delta\bigr) \qquad \forall\, t \in [T].
    \end{equation*}
    Finally, the standard Hellinger-vs.-TV comparison for Bernoulli distributions gives $D_H^2(p, q) \geq \tfrac{1}{2}(p - q)^2$ for all $p, q \in [0, 1]$, so
    \begin{equation*}
        \sum_{s=1}^{t-1} \bigl(\widehat{P}_t - P^\star\bigr)^2(\mathbf{y}_s \succ \widehat{\mathbf{y}}_s \mid \x_s) \leq 2 \sum_{s=1}^{t-1} D_H^2\bigl(p_s(\widehat{P}_t), p_s(P^\star)\bigr) \leq 4 \log\bigl(T|\mathcal{P}|/\delta\bigr).
    \end{equation*}
\end{proof}
\begin{lemma}[Freedman's inequality~\cite{freedman1975tail}]
    \label{lem:freedman}
    Let $\{X_t\}_{t=1}^T$ be a sequence of real-valued random variables adapted to a filtration $\{\mathcal{F}_t\}_{t=0}^T$ satisfying $\EE[X_t \mid \mathcal{F}_{t-1}] = 0$ and $|X_t| \leq R$ almost surely. Let $V_T := \sum_{t=1}^T \EE[X_t^2 \mid \mathcal{F}_{t-1}]$ denote the predictable quadratic variation. Then for any $\delta \in (0, 1)$ and any $\eta \in (0, 1/R)$, with probability at least $1 - \delta$,
    \begin{equation*}
        \sum_{t=1}^T X_t \leq \eta V_T + \frac{\log(1/\delta)}{\eta}.
    \end{equation*}
    In particular, applying this with $X_t = \EE_t[Y_t] - Y_t$ for $Y_t \in [0, 1]$ adapted, and noting that $\EE_t[X_t^2] \leq \EE_t[Y_t]$, with probability at least $1 - \delta$,
    \begin{equation*}
        \sum_{t=1}^T \EE_t[Y_t] \leq 2 \sum_{t=1}^T Y_t + 4 \log(1/\delta).
    \end{equation*}
\end{lemma}

\section{Additional details on DENPO and numerical experiments}

\subsection{Details on Direct ENPO (DENPO)}
\label{app:sec:DENPO}

In this section, we discuss a simplified version of ENPO: Direct ENPO (DENPO).

For ease of presentation, we summarize the notation for policies and datasets. At iteration $t$, let $\pi_t$ be the policy output by the learner and $\hat{\pi}_t$ be the policy output by the adversarial explorer. We use $D_1^t$ for the SFT dataset sampled from $\pi_{t-1}$, and $D_2^t$ for the contrastive dataset of response pairs from $\pi_{t-2}$ and $\hat{\pi}_{t-1}$ with preference labels. Finally, $\tilde{D}_2^t$ denotes the re-labeled dataset used to train the adversarial explorer via the DPO-style objective.

\paragraph{SFT-type regularizer.} The exploitation update drives the current policy in the direction of a mirror descent update. First, we sample $n$ triples $(\x,\y,\y')$ by drawing $\x$ from a fixed prompt distribution, $\y$ from the policy learner $\pi_{t - 2}$, and $\y'$ from the adversarial explorer $\hat{\pi}_{t - 1}$, and use them to construct the dataset $D_2^t$ for the contrastive objective.
We then sample an i.i.d. comparator response $\hat{\y}$ from $\pi_{t - 1}(\cdot | \x)$ and query preference labels
$I_t(\y) \sim \textnormal{Bern}(P^\star(\y \succ \hat{\y} | \x))$ and
$I_t(\y') \sim \textnormal{Bern}(P^\star(\y' \succ \hat{\y} | \x))$
to estimate the mirror descent gradient.
In addition, we generate or reuse $m$ samples $(\x,\tilde{\y})$ from $\pi_{t - 1}$ to construct a dataset $D_1^t$ for the SFT-type regularizer.
We then minimize the loss function $\mathcal{L}_{\textnormal{md}}(\pi,D_1^t,D_2^t)
:=$
\begin{equation}
    - \alpha_t \sum_{(\x,\tilde{\y})\in D_1^t} \log \pi(\tilde{\y} | \x)
    + \sum_{(\x,\y,\y',I_t(\y),I_t(\y'))\in D_2^t}
    \frac{1}{n}
    \Bigl(
    h_t(\pi,\x,\y,\y')
    - \eta\bigl(I_t(\y)-I_t(\y')\bigr)
    \Bigr)^2,
    \label{def:main-l2-loss}
\end{equation}
where
\begin{equation}
    h_t(\pi, \x, \y, \y')
    :=
    \log\frac{\pi(\y | \x)}{\pi(\y' | \x)}
    - \eta\beta \log\frac{\pi_\tnref(\y | \x)}{\pi_\tnref(\y' | \x)}
    - (1 - \eta\beta)\log\frac{\pi_{t-1}(\y | \x)}{\pi_{t-1}(\y' | \x)}.
    \label{eq:def-ht}
\end{equation}
The validity of this loss function can be verified by an argument similar to those in~\cite{azar2024general,Zhang-2025-Iterative}.

The first term in~\cref{def:main-l2-loss} is an SFT-type bonus that strengthens the anchor to the current policy learner.
This bonus differs from the more common exploration bonuses used in online RL.
The additional regularization can improve the learning process by enhancing the proximal effect and mitigating the instability that arises when learning an equilibrium.
This modification is also in the same spirit as Magnetic Mirror Descent~\citep{sokota2023unified}, which augments regularization for learning in games.

To guide the reader through the sequence of objectives below, we briefly summarize their roles. \cref{def:main-l2-loss} defines the exploitation loss used to update the policy learner. \cref{eq:adversarial-sampler-original} introduces the theoretical adversarial explorer objective. The subsequent derivations reformulate and simplify this objective, leading to the practical DPO-style loss in~\cref{eq:explorer-dpo-loss-sigmoid}.

\begin{algorithm}[t]
    \caption{Direct Exploratory Nash Preference Optimization (DENPO)}
    \label{alg:practical-adv-exp-NashMD}
    \SetKwInOut{Input}{Input}
    \SetKwInOut{Output}{Output}
    \Input{
    \(T\), coefficients $\gamma,\lambda>0$, stepsize \(\eta>0\), reference policy \(\pi_{\rm ref}\), $\alpha_t$. 
    }
    Initialize policies $\pi_{-1} \gets \pi_\tnref$, $\pi_0 \gets \pi_\tnref$, $\hat{\pi}_0 \gets \pi_\tnref$\;
    Initialize datasets $D_{1}^0 \gets \varnothing, D_{2}^0 \gets \varnothing$\;

    \For{$t = 1, \ldots, T$}{
        Generate $n$ response pairs $\{ (\x, \y, \y') \}$ via $\x \sim \rho$, $\y \sim \pi_{t - 2}(\cdot | \x)$ and $\y' \sim \hat{\pi}_{t - 1}(\cdot | \x)$\;
        Generate or reuse $m$ responses $\{ (\x, \tilde{\y}) \}$ via $\x \sim \rho$ and $\tilde{\y} \sim \pi_{t - 1}(\cdot | \x)$\; 
        Sample response pairs from $D^t_2$, for each $(\x, \y, \y')$, generate or reuse $(\x, \hat{\y})$ via $\hat{\y} \sim \pi_{t - 1}(\cdot \ | \ \x)$\; 
        Query preference labels $I(\y) \sim \textnormal{Bern}(P^\star(\y \succ \hat{\y} \mid \x))$ and $I(\y') \sim \textnormal{Bern}(P^\star(\y' \succ \hat{\y} \mid \x))$\;
        Update datasets $D^t_{1} \gets \{ (\x, \tilde{\y}) \}$, $D^t_{2} \gets D^{t-1}_{2} \cup \{ (\x, \y, \y', I(\y), I(\y')) \}$\;
        Compute $\pi_{t} \gets \textnormal{argmin}_{\pi \in \Pi} \ \mathcal{L}_{\textnormal{md}}(\pi, D^t_1, D^t_2)$, 
        where $\mathcal{L}_{\textnormal{md}}(\pi, D^t_1, D^t_2)$ is defined in~\cref{def:main-l2-loss}\;
        Compute $\hat{\pi}_t \gets \textnormal{argmin}_{\tilde{\pi} \in \Pi} \mathcal{L}_{\textnormal{exp}}(\tilde{\pi}, D^t_2)$, 
        where $\mathcal{L}_{\textnormal{exp}}(\tilde{\pi}, D^t_2)$ is defined in~\cref{eq:explorer-dpo-loss-sigmoid}\;
    }
    \Output{$\textnormal{best}(\{ \pi^t \}_{t=1}^T)$.}
\end{algorithm}

\paragraph{DPO-type explorer surrogate.} After the exploitation update, we invoke an exploration oracle to identify an adversarial sampler by solving the following policy maximization problem without collecting additional samples: 
\begin{equation}
    \max_{\tilde{\pi} \in \Pi} \max_{\pi \in \Pi}
    \frac{\big( \frac{1}{m}\sum_{(\x, \tilde{\y}) \in D^t_1} \log{\frac{\pi_t(\tilde{\y} | \x)}{\pi(\tilde{\y} | \x)}}
    - \mathbb{E}_{\x \sim \rho, \y' \sim \tilde{\pi}(\cdot|\x)}\big[ \log{\frac{\pi_t(\y' | \x)}{\pi(\y' | \x)}} \big] \big)^2}
    {\lambda + \frac{1}{n}
    \sum_{(\x, \y, \y') \in D^t_2}
    \big(
    \log \frac{\pi_t(\y | \x)}{\pi(\y | \x)}
    - \log \frac{\pi_t(\y' | \x)}{\pi(\y' | \x)}
    \big)^2}.
    \label{eq:adversarial-sampler-original}
\end{equation}
One can use a surrogate of this oracle in practice.
Applying the identity $a^2/b = \max_{\omega} \{2a\omega - b\omega^2\}$ for $b > 0$, we recast~\cref{eq:adversarial-sampler-original} as
\begin{align}
    &\max_{\tilde{\pi} \in \Pi} \max_{\pi \in \Pi} \max_{\omega} \
    2 \omega \Big(
    \mathbb{E}_{\x \sim \rho, \y' \sim \tilde{\pi}(\cdot|\x)}\big[ \log{\pi_t(\y' | \x)} - \log{\pi(\y' | \x)} \big]
    - \tfrac{1}{m}\sum\nolimits_{(\x, \tilde{\y}) \in D^t_1}
    \log{\tfrac{\pi_t(\tilde{\y} | \x)}{\pi(\tilde{\y} | \x)}}
    \Big) \nonumber \\
    & \hspace{80pt}
    - \omega^2 \Big(\lambda + \tfrac{1}{n}
    \sum\nolimits_{(\x, \y, \y') \in D^t_2}
    \left(
    \log \tfrac{\pi_t(\y | \x)}{\pi_t(\y' | \x)}
    - \log \tfrac{\pi(\y | \x)}{\pi(\y' | \x)}
    \right)^2 \Big).
\end{align}
The two empirical terms depending on $\pi$ act as implicit anchors: each penalizes the log-density gap between $\pi$ and $\pi_t$ on observed data, 
mimicking the geometry of a KL regularizer. 
We then derive a tractable surrogate loss as follows.
First, we approximate the implicit regularization terms using a simpler term:
\begin{align*}
    \max_{\tilde{\pi} \in \Pi} \max_{\pi \in \Pi} \
    \mathbb{E}_{\x \sim \rho, \y \sim \tilde{\pi}(\cdot|\x)}\big[ \log{\pi_t(\y | \x)} - \log{\pi(\y | \x)} \big]
    - c\, \E_{\x \sim \rho, \y \sim \pi_{t - 1}(\cdot|\x)}
    \left[ \log{\frac{\pi_t(\y | \x)}{\pi(\y | \x)}} \right],
\end{align*}
where $c$ is a constant hyperparameter. 
This subproblem, up to a constant independent of $\pi$, is equivalent to
\begin{align}
    \max_{\tilde{\pi} \in \Pi} \max_{\pi \in \Pi} \
    \mathbb{E}_{\x \sim \rho, \y \sim \tilde{\pi}(\cdot|\x)}\big[ \log{\pi_t(\y | \x)} - \log{\pi(\y | \x)} \big]
    - c\, \E_{\x \sim \rho} \left[
    \KL({\pi_{t-1}(\cdot | \x)}\Vert{\pi(\cdot | \x)}) \right].
    \label{eq:adversarial-sampler-smooth}
\end{align}
Then, setting $\pi = \pi_{t-1}$ collapses the inner maximization and zeros out the KL anchor, leaving an objective in $\tilde{\pi}$ alone. The remaining linear term depends on $\tilde{\pi}$ only through its mass on $\mathrm{supp}(\pi_t) \cap \mathrm{supp}(\pi_{t-1})$---we pin this effect down with a KL anchor of strength $b > 0$ against $\pi_{t-1}$, which yields 
\begin{equation}
    \max_{\tilde{\pi} \in \Pi} \
    \mathbb{E}_{\x \sim \rho, \y' \sim \tilde{\pi}}
    \big[ \log{\pi_t(\y' | \x)} - \log{\pi_{t - 1}(\y' | \x)} \big]
    -
    b \E_{\x \sim \rho}\left[ \KL(\tilde{\pi}(\cdot | \x) \Vert \pi_{t - 1}(\cdot | \x)) \right].
    \label{eq:stepB-objective}
\end{equation}
A DPO-style reparameterization of~\cref{eq:stepB-objective} yields the following surrogate loss:
\begin{equation}
    - \sum\nolimits_{(\x,\y_+,\y_-)\in \tilde{D}_2^t}
    \log\sigma \left(
    b\log\tfrac{\tilde\pi(\y_+ | \x)}{\pi_{t - 1}(\y_+ | \x)}
    -
    b\log\tfrac{\tilde\pi(\y_- | \x)}{\pi_{t - 1}(\y_- | \x)}
    \right),
    \label{eq:explorer-dpo-loss-sigmoid}
\end{equation}
where each triple in $\tilde{D}_2^t$ is constructed from a triple $(\x, \y, \y')$ in $D_2^t$ by assigning $(\y_+, \y_-)$ according to the ordering of
$\log{\pi_t(\y | \x)} - \log{\pi_{t - 1}(\y | \x)}$ and $\log{\pi_t(\y' | \x)} - \log{\pi_{t - 1}(\y' | \x)}$. We summarize this simplified version of \texttt{ENPO}, called \texttt{DENPO}, in~\cref{alg:practical-adv-exp-NashMD}.

\subsection{Computational cost and budget-matched comparisons}
\label{sec:denpo-compute}

The comparison in the main text fixes the number of prompts and iterations, but the methods have different costs per iteration. 
We therefore report sampling, preference feedback, generated tokens, elapsed time, and optimization work separately. 
All runs use 20{,}000 prompts per iteration on a node with eight H100 GPUs. 
We list the per-prompt cost for each method in~\cref{tab:computational-cost}.

\begin{table}[!htbp]
    \centering
    \caption{Computational cost per prompt for different methods.}
    \label{tab:computational-cost}
    \small
    \begin{tabular}{lcccc}
        \toprule
        \textbf{Method}
        & Generated completions
        & Preference queries
        & Judge evaluations
        & Policy-training stages \\
        \midrule
        XPO   & 16 & N/A & 16 & 1 \\
        INPO  & 8  & 10  & 20 & 1 \\
        DENPO & 8  & 8   & 16 & 2 \\
        \bottomrule
    \end{tabular}
\end{table}

XPO generates $8$ current-policy and $8$ reference-policy completions. INPO generates $8$ completions and applies its tournament procedure. DENPO generates $1$ learner candidate, $1$ explorer candidate, and $4$ comparator responses. Comparing the two candidates with the four comparators gives $8$ pairwise queries. 
DENPO generates two fresh responses in the exploration phase.
DENPO has two policy-training stages per iteration, one for the learner and one for the explorer (these are not counts of individual gradient updates).

Other computational-cost-per-iteration measures are listed in~\cref{tab:denpo-time-cost}.
We can see that DENPO is \emph{not} substantially more costly than baselines in a variety of measures.
\begin{table}[!htbp]
    \centering
    \caption{Measured per-iteration cost on eight H100s in the seed-1 runs.}
    \label{tab:denpo-time-cost}
    \small
    \begin{tabular}{lrrrrr}
        \toprule
        Method & Tokens & Wall-clock (h) & GPU-hours & Optimizer steps & Training stages \\
        \midrule
        XPO   & 100M & 1.567 & 12.54 & 217 & 1 \\
        INPO  & 54M  & 0.897 & 7.17  & 109 & 1 \\
        DENPO & 68M  & 1.990 & 15.92 & 311 & 2 \\
        \bottomrule
    \end{tabular}
\end{table}

\paragraph{Sampling and feedback budgets.}
\cref{tab:denpo-budget-quality} compares checkpoints using the
same evaluation protocol as~\cref{tab:leaderboard}. Its scores are
means across the three training runs, while the completion and query
counts are fixed by the sampling protocols. DENPO$_{1/8}$ after two
iterations exceeds the strongest iteration-3 INPO variant and XPO on
all four reported metrics, despite using 320K rather than 480K and
960K completions consumed by those baselines, respectively. It also
requires 320K pairwise preference queries versus INPO's
600K. Thus, the quality gain is already visible before DENPO's third iteration.

\begin{table}[!htbp]
\centering
\caption{Quality versus cumulative sampling cost.}
\label{tab:denpo-budget-quality}
\small
\begin{tabular}{lrrrrrr}
\toprule
Checkpoint & Used comp. & Ordinal & AE2-mini & AE2-turbo & MT-Bench & Arena-Hard \\
\midrule
XPO, iter. 3              & 960K & ---  & 47.06 & 40.75 & 6.91 & 45.33 \\
INPO, iter. 3             & 480K & 600K & 53.69 & 45.81 & 6.99 & 50.83 \\
\textbf{DENPO$_{1/8}$, iter. 2} & \textbf{320K} & \textbf{320K} & \textbf{56.22} & \textbf{46.92} & \textbf{7.03} & \textbf{53.03} \\
DENPO$_{1/8}$, iter. 3    & 480K & 480K & 58.33 & 47.93 & 7.04 & 55.10 \\
\bottomrule
\end{tabular}
\end{table}

\paragraph{Matched GPU-hour budgets.}
\cref{tab:denpo-gpu-quality} uses the fully cost-tracked seed-1
run to compare checkpoints close in elapsed GPU-hours. 
DENPO after two iterations uses fewer GPU-hours and optimizer steps than XPO after three, while scoring higher on all four benchmarks. 
For INPO, we continued its strongest seed-1 variant from iteration 3 through iteration 6. 
This consumed 45.90 GPU-hours, within 4.2\% of DENPO's 47.89 GPU-hours at iteration 3. 
The final DENPO checkpoint scores 57.88 versus 45.56 on AE2-mini and leads on the other three metrics.
Even choosing INPO's \emph{best earlier checkpoint} within those six iterations yields 54.04 AE2-mini, 45.22 AE2-turbo, 7.08 MT-Bench, and 51.8 Arena-Hard: DENPO remains ahead except for a tie on MT-Bench.

\begin{table}[!htbp]
\centering
\caption{Single-run, near-matched measured budgets.}
\label{tab:denpo-gpu-quality}
\small
\begin{tabular}{lrrrrrrr}
\toprule
Checkpoint & Wall (h) & GPU-h & Steps & AE2-mini & AE2-turbo & MT-Bench & Arena-Hard \\
\midrule
XPO, iter. 3           & 4.710 & 37.680 & 651 & 45.65 & 40.35 & 6.97 & 42.6 \\
\textbf{DENPO$_{1/8}$, iter. 2} & \textbf{3.912} & \textbf{31.296} & \textbf{621} & \textbf{57.40} & \textbf{45.84} & \textbf{7.08} & \textbf{52.2} \\
\midrule
INPO, iter. 6          & 5.737 & 45.898 & 658 & 45.56 & 37.84 & 6.66 & 33.8 \\
\textbf{DENPO$_{1/8}$, iter. 3} & \textbf{5.986} & \textbf{47.888} & \textbf{934} & \textbf{57.88} & \textbf{46.46} & \textbf{7.08} & \textbf{54.0} \\
\bottomrule
\end{tabular}
\end{table}

The matched-budget experiments demonstrate the strong performance of DENPO.

\subsection{Performance across training iterations}
\label{sec:denpo-trajectories}

\begin{figure}[t]
\centering
\includegraphics[width=0.92\linewidth]{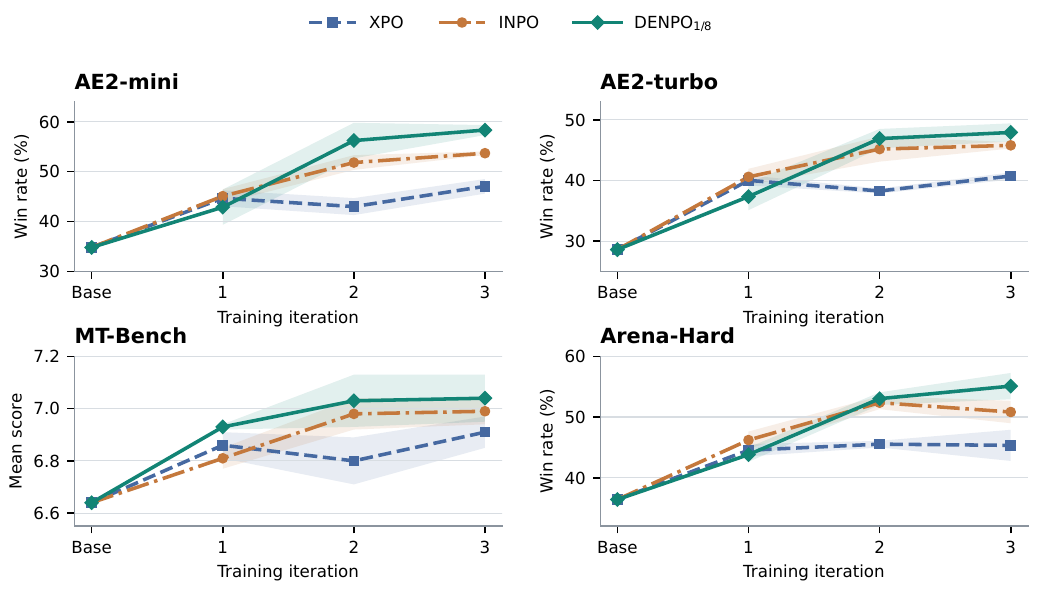}
\caption{Evaluation performance over training iterations. AE2-mini,
AE2-turbo, and Arena-Hard report win rates (\%); MT-Bench reports a
mean score on a 0--10 scale. Curves show three-run checkpoint means and
shading shows one sample standard deviation at iterations 1--3. The
common base-model point has no training-run variation. Higher is better.}
\label{fig:denpo-trajectories}
\end{figure}

Figure~\ref{fig:denpo-trajectories} traces the four evaluation metrics
from the base model through iterations 1-3. 
We display XPO, 
the default-parameter INPO variant, 
and DENPO$_{1/8}$, 
the strongest iteration-3 variants of INPO and DENPO in~\cref{tab:leaderboard}.
The checkpoint values are means over three training runs.
The shaded bands show their sample standard deviations. 
The base model is a single evaluation shared by the three trajectories.

The methods do not improve monotonically on every metric. XPO's two
AlpacaEval-2 scores decline at iteration 2 before recovering at
iteration 3; INPO's Arena-Hard mean falls from 52.37 at iteration 2 to
50.83 at iteration 3. DENPO$_{1/8}$ trails INPO on both AlpacaEval-2
metrics after the first iteration, then exceeds it on all four metrics
at iterations 2 and 3. At iteration 2 its mean advantage over INPO is
4.40 points on AE2-mini, 1.73 points on AE2-turbo, 0.05 on MT-Bench,
and 0.66 points on Arena-Hard. At iteration 3 the respective margins
are 4.64, 2.12, 0.05, and 4.27.

\end{document}